\documentclass{article} 
\usepackage{style,times}
\usepackage{etoolbox}
\renewcommand{\ruler}[1]{}
\makeatletter

\patchcmd{\@maketitle}
  {Anonymous authors\\Paper under double-blind review}
  {\@author}
  {}{}

\patchcmd{\@maketitle}
  {{\LARGE\sc \@title\par}}
  {\vskip -0.4in{\LARGE\sc \@title\par}\vskip 0.15in}
  {}{}

\patchcmd{\@maketitle}
  {\vskip 0.7in minus 0.05in}
  {\vskip 0.75in minus 0.05in}
  {}{}

\makeatother

\usepackage{amsmath,amsfonts,bm}

\def\eqref#1{equation~\ref{#1}}

\def\1{\bm{1}}

\DeclareMathAlphabet{\mathsfit}{\encodingdefault}{\sfdefault}{m}{sl}
\SetMathAlphabet{\mathsfit}{bold}{\encodingdefault}{\sfdefault}{bx}{n}

\newcommand{\E}{\mathbb{E}}

\usepackage{hyperref}
\usepackage{url}
\newtheorem{assumption}{Assumption}
\newtheorem{theorem}{Theorem}
\newtheorem{proposition}{Proposition}
\newtheorem{lemma}{Lemma}
\newtheorem{corollary}{Corollary}
\newtheorem{definition}{Definition}
\usepackage{algorithm}
\usepackage{algpseudocode}
\usepackage{float}
\usepackage{graphicx}
\usepackage{booktabs}
\usepackage{mathtools}
\usepackage{tabularx}
\usepackage{mathrsfs}
\newenvironment{proof}
{\par\noindent\textit{Proof.}\ }
{\hfill$\square$\par}
\title{Taylor Representations for Model-Free RL in Networked MDPs}

\author{%
\normalfont
{\renewcommand{\arraystretch}{1.25}%
\begin{tabular*}{0.95\textwidth}{@{\extracolsep{\fill}}ll@{}}
\textbf{Salah R. Chikhi}
&
\textbf{Abdelhaq Chaoui}
\\
Massachusetts Institute of Technology
&
Massachusetts Institute of Technology
\\
\addlinespace[1.6em]
\textbf{Asuman Ozdaglar}
&
\textbf{Saurabh Amin}
\\
Massachusetts Institute of Technology
&
Massachusetts Institute of Technology
\end{tabular*}%
}%
}
\begin{document}

\maketitle

\begin{abstract}
In Networked Markov Decision Processes, transition dynamics are often unknown and the state--action space grows rapidly with the number of agents. In this setting, Taylor representations naturally approximate $Q$-functions, but a naive order-$n$ expansion over $N$ agents requires $\Theta(N^n)$ coefficients. We justify these expansions under smooth expected future local rewards with controlled derivatives. Under this condition, finite-speed information propagation and discounting imply that local-critic Taylor coefficients decay exponentially with the graph distance to the farthest agent involved. Discarding distant-agent coefficients and marginalizing then yield scalable local Taylor representations with a bound controlled by graph locality. Building on these representations, we propose a scalable model-free actor--critic algorithm, establishing finite-sample critic and near-stationarity guarantees for a linear LSTD critic. We then introduce a more expressive neural TD parameterization. Unlike prior constructive spectral methods, our approach covers settings without access to a known local dynamics map, such as hidden switched linear--quadratic regulation. Across three control benchmarks, our method matches or outperforms spectral baselines while scaling efficiently to large graphs.
\end{abstract}

\section{Introduction}
\label{sec:introduction}

Networked Markov Decision Processes, in the context of multi-agent reinforcement learning (MARL), provide a natural framework for cooperative control in large interconnected systems, including communication networks, traffic systems, queueing networks, and other cyber-physical systems \citep{qu2020localized,qu2022networked,chu2020networked}. They nevertheless face two central challenges: the transition kernel is often unknown and the joint state--action space grows rapidly with the number of agents, making large-scale control challenging. To address the latter, in finite state--action spaces, \citet{qu2020localized,qu2022networked} exploit the exponential decay property of local \(Q\)-functions, whereby the influence of distant agents decreases exponentially with graph distance. This enables localized \(Q\)-function and policy-gradient approximations. \citet{lin2021stochastic} extend this principle to stochastic non-local dependencies through a more general decay property, while related guarantees also hold for average-reward objectives \citep{qu2020average}. Related locality phenomena have also been studied in mean-field MARL and networked policy optimization~\citep{gu2022meanfield,zhang2023lpi,liang2025causality} and exponential sensitivity decay arises in graph-structured optimization and nonlinear optimal control ~\citep{shin2022exponential,na2020exponential,na2022schwarz}. In continuous state--action spaces, however, localization alone is insufficient: even a local critic remains an unknown continuous function that must be represented and learned.

For continuous networked control, recent work addresses this representation problem through spatial structure~\citep{olsson2024lqr} and spectral representations~\citep{ren2024spectral,wang2026continuous}. A main setting targeted by the spectral approaches is one in which each local transition follows a known dynamics map with additive Gaussian noise. Spectral representations, together with finite-dimensional features such as random Fourier features and the exponential decay property mentioned above, then yield scalable critic approximations with explicit guarantees \citep{ren2022spectral,ren2023stochastic,rahimi2007random}. However, this approach can be restrictive: the transition kernel is typically unknown, and canonical network-control problems need not belong to this class of dynamics. Switched linear--quadratic regulation (LQR) with an unobserved global mode~\citep{costa2005markov,ogura2018hidden}, for instance, is not captured by such dynamics as it induces state-- and action-dependent noise, shared
across agents and temporally correlated through the Markov mode. \citet{ren2024spectral} also propose a model-free variant named Spectral-SAC, but without the corresponding guarantees of their random Fourier features approach. In contrast, we focus here on the fully model-free setting, where the transition dynamics are unknown and only rollout data are available.

Beyond the setting of networked MDPs, a complementary route to tractable function approximation is smoothness. Taylor representations have been developed for continuous-state MDPs \citep{braverman2020taylor,xu2020kernel,maran2024local}, with \citet{maran2024local,maran2024smooth} showing how regularizing properties of the Bellman operator can justify polynomial representations. However, the poor scalability of a naive Taylor representation stems from its isotropy: the analysis places all Taylor coefficients on equal footing, so an order-\(n\) expansion over \(N\) agents generally requires \(\Theta(N^n)\) coefficients. Moreover, much like the dynamics assumptions of prior constructive spectral methods, the smoothness assumption of~\citep{maran2024local} need not hold in canonical network-control settings such as switched LQR with unobserved operating modes, as we show in Appendix~\ref{app:regularity}. In our analysis, we instead require smoothness of expected future local rewards in the initial state. 

\paragraph{Our contribution.} Our key observation is that, in networked MDPs, Taylor coefficients are not equally important: they themselves inherit graph locality. Finite propagation and discounting imply that derivatives involving an agent at graph distance \(d\) decay exponentially with \(d\). This allows us to truncate the global Taylor expansion to monomials supported on \(\mathcal N_i^\kappa\), the \(\kappa\)-hop neighborhood of agent \(i\). After marginalizing distant variables, the resulting polynomial is equivalently a Taylor representation of a localized critic with only \(\Theta(|\mathcal N_i^\kappa|^n)\) state monomials. Using rollout-based centroids as expansion points, we first obtain a linear-in-parameters Taylor Centroid critic and a scalable actor--critic with finite-sample critic and near-stationarity guarantees. We then generalize the parameterization with neural networks. Unlike prior constructive spectral approaches, our method requires neither knowledge nor a structural specification of the transition kernel. We benchmark it across three control problems of increasing difficulty. Table~\ref{tab:comparison} summarizes how our contributions relate to prior work.

\vspace{-1.65em}
\begin{table}[h!]
\centering
\caption{Comparison with representative prior representation-based approaches.}
\label{tab:comparison}
\small
\setlength{\tabcolsep}{2.7pt}
\renewcommand{\arraystretch}{1.08}
\begin{tabularx}{\linewidth}{@{}lcccccc@{}}
\toprule
& \textbf{Qu et al.}
& \textbf{Olsson et al.}
& \textbf{Ren et al.}
& \textbf{Wang et al.}
& \textbf{Maran et al.}
& \textbf{Ours} \\
\midrule

Networked MDP
& $\checkmark$
& $\checkmark^\ddagger$
& $\checkmark$
& $\checkmark$
& --
& $\checkmark$ \\

Continuous state--action
& --
& $\checkmark$
& $\checkmark$
& $\checkmark$
& $\checkmark$
& $\checkmark$ \\

Critic representation
& Tabular
& Local quadratic
& Spectral
& Spectral
& Taylor
& Local Taylor \\

Scalable representation
& $\checkmark$
& $\checkmark$
& $\checkmark$
& $\checkmark$
& --
& $\checkmark$ \\

Model-free construction
& $\checkmark$
& $\checkmark$
& --$^\dagger$
& --
& $\checkmark$
& $\checkmark$ \\

End-to-end guarantee
& Stationarity
& --
& Stationarity
& Stationarity
& No-regret
& Stationarity \\

\bottomrule
\end{tabularx}

\vspace{0.25em}
\parbox{\linewidth}{
$^\dagger$\citet{ren2024spectral} also propose a model-free alternative to their random Fourier features but without guarantees.
$^\ddagger$\citet{olsson2024lqr} specialize to LQR.}
\end{table}
\vspace{-1.8em}

\section{Problem Setting}\label{sec:setting}

We consider $N$ agents on an undirected graph $G=(V,E)$,
$V=[N]$. Let $d_G(i,j)$ denote graph distance and, for $\kappa \geq 0$, define
$\mathcal N_i^\kappa:=\{j:d_G(i,j)\le\kappa\}$,
$\mathcal N_i:=\mathcal N_i^1$, and
$\mathcal N_{-i}^{\kappa}:=V\setminus\mathcal N_i^\kappa$.
Agent $i$ has state
$s_i\in\mathcal S_i\subset\mathbb R^{d_s}$ and takes action
$a_i\in\mathcal A_i\subset\mathbb R^{d_a}$, where the state spaces are
compact and convex and the action spaces are compact. We write
$s_I=(s_j)_{j\in I}$ and $a_I=(a_j)_{j\in I}$. Our framework also allows a global context $z\in\mathcal Z$ which may be discrete or continuous. We consider the class of factorized transition kernels and policies: \vspace{-0.6em}\[
P(z',s'\mid z,s,a)
=
P_g(z'\mid z)
\prod_{i=1}^N
P_i(s_i'\mid s_{\mathcal N_i},a_{\mathcal N_i},z),
\quad
\pi_\theta(a\mid z,s)
=
\prod_{i=1}^N
\pi_{\theta_i}\!\left(
a_i\mid s_{\mathcal N_i^{\kappa_\pi}},h_i(z)
\right).
\] Here, the policy observation radius is $\kappa_\pi$, and
$h_i(z)\in\mathbb R^{d_g^i}$ denotes the global context observable to agent $i$. Thus $z$ may be unobserved, partially observed, or fully observed by agent $i$. Furthermore, $z$ affects every local transition but cannot transmit physical
information between distant agents. 
Let $\Delta$ be the maximum degree of the graph. Assume that the rewards satisfy $r(s,a) = \frac{1}{N}\sum_{i=1}^{N} r_i(s_i,a_i)$ with $|r_i(s_i,a_i)|\le R_{\max}$.
 Our analysis extends directly to local rewards depending on any fixed-radius neighborhood of agent $i$. The discount factor satisfies $\gamma \in (0,1)$. We assume that the initial distribution $\nu_0$ of $(Z_0,S_0)$ is fixed and does not
depend on the policy parameters $\theta$. The cooperative objective to be maximized is \[J(\theta)
=
\frac1N
\sum_{i=1}^N
\E_{(Z_0,S_0)\sim \nu_0, \pi_\theta}
\left[
\sum_{t\ge0}\gamma^t r_i(S_{i,t},A_{i,t})
\right].
\]To maximize this objective, we adopt the actor-critic framework \citep{konda2000actor,sutton1999policy,sutton2018reinforcement}. To this end, define the local and global critics 
\begin{equation}
Q_i^\pi(z,s,a)
:=
\E_\pi\!\left[
\sum_{t\ge0}\gamma^t r_i(S_{i,t},A_{i,t})
\,\middle|\,
Z_0=z,S_0=s,A_0=a
\right],
\quad
Q^\pi:=\frac1N\sum_{i=1}^N Q_i^\pi .
\label{eq:local-q}
\end{equation}
To make our targeted Taylor representations well-defined, we use the following assumption:
\begin{assumption}[Expected reward regularity]
\label{asm:orbit}
There exist an integer $q\ge2$ and a constant $\bar L<\infty$,
both independent of the network size $N$, such that, uniformly over agents $i$ and all policies $\pi$ of the above factorized form, for every $t\ge0$ the map $s\mapsto
\E_\pi\!\left[
r_i(S_{i,t},A_{i,t})
\,\middle|\,
Z_0=z,S_0=s,A_0=a
\right]
$ is of class $C^q$ and
\[
\max_{|\beta|\le q}
\sup_{z,s,a}
\left|
\partial_s^\beta
\E_\pi\!\left[
r_i(S_{i,t},A_{i,t})
\,\middle|\,
Z_0=z,\;S_0=s,\;A_0=a
\right]
\right|
\le \bar L .
\]
\end{assumption}

Assumption~\ref{asm:orbit} is an expected reward regularity condition: a small
perturbation of the current physical state may lead to very different
trajectories, but should not create arbitrarily sharp changes in the expected
future local reward. Appendix~\ref{app:regularity} gives \emph{concrete sufficient
conditions} for Assumption~\ref{asm:orbit}, introduces a \emph{rollout-based
stress test} for assessing it empirically, and shows that the uniform-in-$t$
derivative bound can be weakened to a \emph{discounted-summability
condition}. We retain the uniform formulation in the main text for simplicity. \\
Assumption~\ref{asm:orbit} plays two roles: it makes each \(Q_i^\pi\) of class \(C^q\) (Proposition~\ref{prop:smoothness}), enabling Taylor expansions of any order \(n\le q-1\), and, together with finite graph propagation, yields the exponential derivative decay of Lemma~\ref{lem:graded-derivatives}. In contrast, the standard exponential decay of the value-localization error with the localization radius \citep{ren2024spectral,wang2026continuous} requires only bounded rewards and graph-local factorization.

\begin{proposition}
\label{prop:smoothness}
    The individual action-value function
    $s\mapsto Q_i^\pi(z,s,a)$ is of class $C^q$ for all $i,z,a$.
\end{proposition}

\citet{maran2024local} also construct local polynomial value-function
approximations through Taylor expansions, but rely on Mild smoothness, a
regularity assumption on the Bellman operator, whereas
Assumption~\ref{asm:orbit} only requires smoothness of the future reward
sequence defining $Q_i^\pi$. Hence, under smooth rewards, sufficiently high-order Mild smoothness implies Assumption~\ref{asm:orbit}, but not conversely: Appendix~\ref{app:regularity} shows this separation on switched LQR \citep{costa2005markov,ogura2018hidden}, while also showing that the additive-Gaussian setting of \citet{ren2024spectral,wang2026continuous} satisfies Assumption~\ref{asm:orbit} under uniformly \(C^q\) local dynamics and rewards when \(\Delta\) is independent of \(N\). \\
The smoothness property of Proposition~\ref{prop:smoothness} naturally leads to our Taylor representations for the actor--critic structure developed in the remainder of the paper: we first construct a scalable Taylor representation of the critic in Section~\ref{sec:tc}, and then use it to derive an actor-critic algorithm in Section~\ref{sec:learning}.

\vspace{-0.7em}
\section{Taylor Centroid Critics}\label{sec:tc}

\subsection{Taylor expansions and scalability}

Having established the required smoothness of the critic, we now turn to its representation. As in \citet{maran2024local}, we begin with an order-$n$ Taylor expansion around an expansion point $\mu$, for any chosen integer $1\le n\le q-1$.\footnote{
Write $\beta=(\beta_{j,r})_{j\in[N],\,r\in[d_s]}\in\mathbb N^{Nd_s}$.
We use
$|\beta|=\sum_{j=1}^N\sum_{r=1}^{d_s}\beta_{j,r}$,
$\beta!=\beta_{1,1}!\cdots\beta_{N,d_s}!$,
$x^\beta=x_{1,1}^{\beta_{1,1}}\cdots x_{N,d_s}^{\beta_{N,d_s}}$, 
$\partial_s^\beta
=
\frac{\partial^{|\beta|}}
{\partial s_{1,1}^{\beta_{1,1}}\cdots
 \partial s_{N,d_s}^{\beta_{N,d_s}}}$
and $\mathrm{supp}_G(\beta) = \{i : \ \exists l \mid \beta_{i,l} \neq 0\}$.}
 \begin{equation}
 \label{eq:taylor}
Q_i^\pi(z,s,a) \simeq 
\sum_{\substack{\beta\in\mathbb N^{Nd_s}\\|\beta|\le n}}
\frac{
\partial_s^\beta Q_i^\pi
\bigl(z,\mu,a\bigr)
}{\beta!}
\bigl(s-\mu\bigr)^\beta
\end{equation}
The drawback of this representation is dimensionality: an order-$n$ expansion
contains $\binom{Nd_s+n}{n}=\Theta(N^n)$ coefficients, while $\mu$ and $a$
are themselves ambient-dimensional. Importantly, existing locality results show that changing states and actions outside the $\kappa$-hop neighborhood of any agent $i$ changes its local value $Q_i^\pi$ by an amount that decays exponentially with $\kappa$
\citep{qu2020localized,qu2022networked,ren2024spectral,wang2026continuous},
but do not characterize its Taylor coefficients. Our key observation is that
expected reward regularity further allows exponential
decay of the derivatives with graph distance. Lemma~\ref{lem:graded-derivatives}
formalizes this decay.

\begin{lemma}[Exponential decay of critic derivatives]
\label{lem:graded-derivatives}
Define $\rho:=\gamma^{1/(1+\kappa_\pi)} \in (0,1)$, and the graph reach
$d_i(\beta):=\max_{j\in\operatorname{supp}_G(\beta)}d_G(i,j)$ with $d_i(0):=0$. Then, for every $|\beta|\le n+1,$ $$\displaystyle 
\left|\partial_s^\beta Q_i^\pi(z,s,a)\right|
\le
L_Q\rho^{d_i(\beta)}, 
\qquad L_Q:= \frac{\bar L}{1-\gamma}$$
\end{lemma}
\vspace{-0.8em}
Hence any derivative involving an agent at distance $d$ is
$O(\rho^d)$: indeed, finite propagation delays its influence on agent $i$ by at least
$d/(1+\kappa_\pi)$ steps, after which discounting contributes the factor
$\rho^d$. The decay can be substantially faster in practice
(see Appendix~\ref{app:exp-graded-decay}); henceforth, $\rho$ may denote the best
valid geometric rate. This motivates discarding derivatives with respect to agents
outside $\mathcal N_i^\kappa$. We denote their multi-indices by:
\[
\mathcal M_{\mathcal N_i^\kappa}^n
:=
\{
\beta\in\mathbb N^{Nd_s}:
|\beta|\le n, \ 
\operatorname{supp}_G(\beta)\subseteq\mathcal N_i^\kappa
\}.
\]

This removes coordinates of agents at distance more than $\kappa$ from the polynomial terms. However, the remaining Taylor coefficients still depend on their values. To obtain
coefficients that require only local states and actions, fix an arbitrary probability measure $\nu_{i,\mathrm{far}}$ on $\mathcal S_{\mathcal N_{-i}^{\kappa}}
\times
\mathcal A_{\mathcal N_{-i}^{\kappa}}$ and define, for
$\beta\in\mathcal M_{\mathcal N_i^\kappa}^n$, the marginalized Taylor
coefficients:
\begin{equation}
\mathscr D_{i,\beta}
(z,s_{\mathcal N_i^\kappa},a_{\mathcal N_i^\kappa})
:=
{
\int_{\mathcal S_{\mathcal N_{-i}^{\kappa}}
\times \mathcal A_{\mathcal N_{-i}^{\kappa}}}
}
\partial_s^\beta Q_i^\pi(z,s,a)\,
d\nu_{i,\mathrm{far}}(s_{\mathcal N_{-i}^{\kappa}},
a_{\mathcal N_{-i}^{\kappa}}).
\label{eq:local-coeff}
\end{equation}

We use expansion points $\mu_1,\ldots,\mu_K\in\mathcal S$, with
$k(s)\in[K]$ selecting the point $\mu_{k(s)}$ closest to $s$ in $L^1$ norm,
in the spirit of classical local polynomial and locally weighted approximation
methods \citep{fan1996local,atkeson1997locally}. We define the clusters in the global state space, so that the cluster index $k(s)$ can inject global-state information into otherwise local critics, in particular helping recover latent global regimes when $z$ is unobserved. Figure~\ref{fig:switched-clusters} illustrates this effect.
Discarding distant derivatives and marginalizing the retained ones then yields
the local Taylor representation:\vspace{-0.1em}\begin{equation}
\widehat Q_i^\pi(z,s,a)
=
\sum_{\beta\in\mathcal M_{\mathcal{N}_i^\kappa}^n}
\frac{
\mathscr{D}_{i,\beta}
(z,\mu_{k(s),\mathcal N_i^\kappa},a_{\mathcal N_i^\kappa})
}{\beta!}
\left(
s_{\mathcal N_i^\kappa}
-\mu_{k(s),\mathcal N_i^\kappa}
\right)^\beta .
\label{eq:ideal-tc}
\end{equation}
At each fixed $k(s)$, this Taylor representation requires $\Theta(|\mathcal{N}_i^\kappa|^n )$ learned coefficients. Notice that \(\mathscr D_{i,0}\) is precisely the critic obtained by marginalizing \(Q_i^\pi\) over states and actions outside \(\mathcal N_i^\kappa\); denoting it by \(\bar Q_i^\pi\), this recovers the localized critic of \citet{ren2024spectral}. Moreover, \(\mathscr D_{i,\beta}=\partial_s^\beta \bar Q_i^\pi\), so Equation~\ref{eq:ideal-tc} is exactly the order-\(n\) Taylor expansion of \(\bar Q_i^\pi\) around the selected centroid.
 Combining this viewpoint with Lemma~\ref{lem:graded-derivatives} yields the following error estimate:

\begin{theorem}[Scalable local Taylor representation]
\label{thm:trunc}
Under Assumption~\ref{asm:orbit}, define, for $0\le d\le\kappa$ and every integer $p\ge0$,\vspace{-1em}\[
W_d
:=
\sup_s
\left\|
s_{\mathcal N_i^d}
-\mu_{k(s),\mathcal N_i^d}
\right\|_1,
\qquad
\mathcal R_p
:=
\rho^\kappa W_\kappa^p
+
(1-\rho)
\sum_{d=0}^{\kappa-1}\rho^d W_d^p .
\]
Then the truncated Taylor critic $\widehat Q_i^\pi$ satisfies \[\left|\left|
Q_i^\pi-\widehat Q_i^\pi\right|\right|_{\infty}
\le
\frac{2R_{\max}}{1-\gamma}\rho^{\kappa+1}
+
\frac{L_Q}{(n+1)!}\mathcal R_{n+1} =: \Delta_{\mathrm{TC}}
\]
Moreover, $\widehat Q_i^\pi$ contains$
\binom{|\mathcal N_i^\kappa|d_s+n}{n}
=
\Theta\!\left(|\mathcal N_i^\kappa|^n\right) = O(\Delta^{n\kappa})
$
state monomials. 
\end{theorem}
Theorem~\ref{thm:trunc} decomposes the approximation error into the standard
\(O(\rho^{\kappa+1})\) localization term
\citep{qu2020localized,ren2024spectral,wang2026continuous} and a Taylor
remainder. Lemma~\ref{lem:graded-derivatives} sharpens the usual isotropic
remainder \(W_\kappa^{n+1}\) to \(\mathcal R_{n+1}\), which
attenuates contributions from the \(d\)-hop neighborhood by \(\rho^d\) through graph locality.
Due to the $1/(n+1)!$ term, the Taylor order \(n\) typically improves the local
approximation. The effect of \(\kappa\) is subtler: enlarging the neighborhood
reduces localization error but increases the Taylor radius \(W_\kappa\) and
introduces additional distant coordinates into \(\mathcal R_{n+1}\). 
Corollary~\ref{cor:distributional-tc} replaces the worst-case radii \(W_d\) by
average distances to the centroids, motivating the \(K\)-medians~\citep{cardot2012kmedians} procedure used
below to reduce the average \(L^1\) distance to expansion points. The key takeaway is a controlled local Taylor approximation with an \(N\)-independent number of monomials for fixed \((\kappa,n,d_s)\).

\subsection{Linear and Neural Taylor Centroid critic estimation}

The local Taylor representation of Equation~\ref{eq:ideal-tc} is scalable, but its
marginalized coefficients $\mathscr D_{i,\beta}$ must still be learned from rollouts.
Here, we first consider a linear parameterization which we name Linear Taylor Centroid (LTC), and a neural one, Neural Taylor Centroid (NTC). Letting $y_i= h_i(z)$, define the coefficient features\footnote{The coefficient
feature map $\psi$ may be chosen arbitrarily and may include nonlinearities.} and local Taylor basis as 
 \vspace{-0.6em}\[
\psi_i(k,y_i,a_{\mathcal N_i^\kappa})
:=
[1;\mu_{k,\mathcal N_i^\kappa};a_{\mathcal N_i^\kappa};y_i],
\qquad
T_i(s)
:=
\left[
\frac{
(s_{\mathcal N_i^\kappa}-\mu_{k(s),\mathcal N_i^\kappa})^\beta
}{\beta!}
\right]_{\beta\in\mathcal M_{\mathcal N_i^\kappa}^n}.
\]
\vspace{-1.9em}
\paragraph{Linear Taylor Centroid.}
LTC approximates each marginalized Taylor coefficient with a function linear
in $\psi_i$:\vspace{-0.8em}\[
\widetilde Q_i^{\mathrm{LTC}}
(y_i,s,a_{\mathcal N_i^\kappa};w_i)
=
\sum_{\beta\in\mathcal M_{\mathcal N_i^\kappa}^n}
\Bigl(
\psi_i(k(s),y_i,a_{\mathcal N_i^\kappa})^\top
w_{i,\beta}
\Bigr)
\frac{
\bigl(
s_{\mathcal N_i^\kappa}
-\mu_{k(s),\mathcal N_i^\kappa}
\bigr)^\beta
}{\beta!}
\] Alternatively, setting $\Phi_i = \psi_i\otimes T_i$, $\widetilde Q_i^{\mathrm{LTC}} = \Phi_i^\top w_i$ is linear in $w_i$, enabling us to use the Least-Squares Temporal Differences (LSTD) algorithm~\citep{bradtke-barto-1996,lazaric2012lstd,ren2024spectral,qu2020localized,wang2026continuous} to fit the critic. Moreover, the number of learnable parameters in this representation is $p_i
:=
\binom{|\mathcal N_i^\kappa|d_s+n}{n}
\left[
1+|\mathcal N_i^\kappa|(d_s+d_a)+d_g^i
\right] = \Theta\left( |\mathcal N_i^\kappa|^{n+1}\right)=O\left(\Delta^{\kappa(n+1)} \right)$ which is independent of $N$. Importantly, we share the coefficient map \(w_i\) across clusters rather than fitting cluster-specific weights, since rarely visited clusters would otherwise have too little data to estimate their coefficients reliably. Cluster-specific behavior is still captured through the centroid $\mu_k$ in the feature map and behavior relative to the global state $s$ is captured by $k(s)$. Since the ideal Taylor Centroid coefficient maps $\mathscr{D}_{i,\beta}$ need not be affine, we quantify the resulting approximation error by defining
\[
\begin{aligned}
\bar\varepsilon_{\mathrm{lin}}
&:=\max_{\beta\in\mathcal M_{\mathcal N_i^\kappa}^n}
\rho^{-d_i(\beta)}\inf_{w\in\mathbb R^{d_\psi^i}}\sup_{z,k,a}\,
\left|\mathscr{D}_{i,\beta}(z,\mu_{k,\mathcal N_i^\kappa},a_{\mathcal N_i^\kappa})
-w^\top\psi_i(k,h_i(z),a_{\mathcal N_i^\kappa})\right|
\le L_Q.
\end{aligned}
\]
The quantity $\bar\varepsilon_{\mathrm{lin}}$ measures misspecification of the
coefficient feature map, normalized by the natural scale
$\rho^{d_i(\beta)}$ from Lemma~\ref{lem:graded-derivatives}. Importantly, localization reduces the
effective dimension of coefficient estimation from the global state dimension
to that of a local neighborhood, making this estimation problem substantially
more tractable. Furthermore, adding adequate features to the feature map $\psi_i$ may lower this misspecification. Notice also that increasing $n$ introduces additional
coefficients to fit, so $\bar\varepsilon_{\mathrm{lin}}$ is nondecreasing in $n$. 
Increasing $\kappa$ similarly enlarges the representation and its parameter
count but may allow for better estimation through a better feature map. Thus, although larger $n$ and $\kappa$ can capture richer critic structure, they can also make estimation harder, yielding an
approximation--estimation tradeoff in practice. This leads to the following finite-sample bound.\vspace{-0.4em}

\begin{theorem}[Finite-sample error of LTC]
\label{thm:ltc-finite}
Fix the centroids and collect $M_c$ independent rollouts of length $T$. Define
\[
\mathcal E_i^{(\kappa)}
:=\bar\varepsilon_{\rm lin}\sum_{p=0}^{n}\frac{\mathcal R_p}{p!},
\qquad
\mathcal G_i^{(\kappa)}
:=R_\psi^2\sum_{p=0}^{n}\frac{W_\kappa^{2p}}{p!},
\qquad
\mathcal V_i^{(\kappa)}
:=\frac{R_{\max}}{1-\gamma}
+L_Q\sum_{p=1}^{n}\frac{\mathcal R_p}{p!},
\]
where $R_\psi:=\sup_{k,y,a}\|\psi_i(k,y,a)\|$.
Conditionally on the centroids and on
$\bigcap_{i=1}^N\{\sigma_{\min}(\widehat A_i)\ge c_{\rm LSTD}>0\}$,
with probability at least $1-\delta$, simultaneously for all $i\in[N]$,
\[
\begin{aligned}
\left\|
Q_i^\pi-
\widetilde Q_i^{\rm LTC}
\bigl(h_i(\cdot),\cdot,\cdot_{\mathcal N_i^\kappa};\widehat w_i\bigr)
\right\|_\infty
&\le
\left(
1+\frac{(1+\gamma)\mathcal G_i^{(\kappa)}}{c_{\rm LSTD}}
\right)
\bigl(\Delta_{\rm TC}+\mathcal E_i^{(\kappa)}\bigr)
\\
&\quad+
\frac{
2\gamma\mathcal G_i^{(\kappa)}
\bigl(\mathcal V_i^{(\kappa)}+\mathcal E_i^{(\kappa)}\bigr)}
{c_{\rm LSTD}}
\sqrt{\frac{2\log(2N/\delta)}{M_c(T-1)}} .
\end{aligned}
\]
\end{theorem}
Theorem~\ref{thm:ltc-finite} separates the LTC error into the intrinsic
Taylor Centroid approximation $\Delta_{\mathrm{TC}}$, the estimation
error $\mathcal E_i^{(\kappa)}$, and a standard statistical
$O((M_cT)^{-1/2})$ term modulated by the local feature factor
$\mathcal G_i^{(\kappa)}$, critic scale $\mathcal V_i^{(\kappa)}$, and LSTD
conditioning. The term $\mathcal E_i^{(\kappa)}$ captures the following centroid tradeoff:
increasing $K$ reduces the local radii in $\mathcal R_p$, but can increase
$\bar\varepsilon_{\mathrm{lin}}$ as a single shared coefficient map must fit
more centroid-dependent coefficients. Appendix~\ref{app:centroid-ablation}
examines these two effects empirically. As in \citet{ren2024spectral}, our guarantee is conditional on the LSTD matrices being well conditioned. We choose the number of rollouts empirically to ensure this condition (see Appendix~\ref{app:common-protocol}); ridge regularization can further improve conditioning in practice. \\
This parallels prior finite-sample critic bounds of \citep{qu2020localized,ren2024spectral}, with \citet{ren2024spectral} also incurring spectral approximation error. Ours instead obtains Taylor and estimation errors, motivating the following neural extension when the latter dominates.\vspace{-0.8em}
\paragraph{Neural Taylor Centroid.}
NTC replaces the affine coefficients by a neural network $f_{\xi_i}:
\mathbb R^{\mathrm{dim}(\psi_i)}
\longrightarrow
\mathbb R^{|\mathcal{M}_{\mathcal{N}_i^\kappa}^n|}.$ \[
\widetilde Q_i^{\mathrm{NTC}}
=
\langle f_{\xi_i}(\psi_i),T_i(s)\rangle = \sum_{\beta\in\mathcal M_{\mathcal N_i^\kappa}^{n}}
\left[f_{\xi_i}(\psi_i)\right]_\beta
\frac{
\left(
s_{\mathcal N_i^\kappa}
-
\mu_{k(s),\mathcal N_i^\kappa}
\right)^\beta
}{
\beta!
}.
\] 
 We train $f_{\xi_i}$ by standard semi-gradient TD regression~\citep{sutton2018reinforcement} (see Appendix~\ref{app:ntc}). NTC preserves the structural guarantees of Theorem~\ref{thm:trunc} and changes only the
coefficient parameterization. This can reduce the analogous error to
$\bar\varepsilon_{\rm lin}$, at the cost of the guarantees of Theorem~\ref{thm:ltc-finite}. 
\vspace{-0.9em}
\section{Learning and Policy Optimization}\label{sec:learning}
To construct a scalable policy-gradient estimator, let agent $i$'s observation be
$o_i(z,s):=(s_{\mathcal N_i^{\kappa_\pi}},h_i(z))$ and define its policy score $\mathsf{s}_i(o_i,a_i)
:=
\nabla_{\theta_i}
\log \pi_{i,\theta_i}(a_i\mid o_i)$,
which we assume to be uniformly bounded as standard in localized policy-gradient
analyses~\citep{qu2020average,ren2024spectral}: $
\sup_{o_i,a_i}
\|\mathsf{s}_i(o_i,a_i)\|
\le B_{\mathsf{s},i}.
$ By the global reward decomposition and the policy gradient theorem
\citep{sutton1999policy}:
\[
\nabla_{\theta_i} J(\theta)
=
\frac{1}{N}
\mathbb E_{\pi_\theta}
\left[
\sum_{t\ge0}\gamma^t
\mathsf{s}_i(O_{i,t},A_{i,t})
\sum_{\ell=1}^N
A_\ell^{\pi_\theta}(Z_t,S_t,A_t)
\right].
\]\paragraph{Estimating the advantage through the critic.}
At actor iteration $m$, let $\widetilde Q_{\ell,m}$ denote the LTC/NTC critic
of agent $\ell$ and suppose that, for every $\ell, z,s,a,
|
\widetilde Q_{\ell,m}
(h_\ell(z),s,a_{\mathcal N_\ell^\kappa})
-
Q_\ell^{\pi_{\theta_m}}(z,s,a)
|
\le
\Delta_Q^{(m)}$. For LTC, Theorem~\ref{thm:ltc-finite} provides such a finite-sample
certificate; for NTC, $\Delta_Q^{(m)}$ denotes the approximation
error achieved by the neural critic.
Let $\widetilde V_{\ell,m}$ be the value function induced by
$\widetilde Q_{\ell,m}$. We use the one-step TD~\citep{sutton1988learning,sutton2018reinforcement,konda2000actor} residual as the local
advantage estimate
$\widehat A_{\ell,t}:=\widehat\delta_{\ell,t}
=r_{\ell,t}
+\gamma\widetilde V_{\ell,m}(Z_{t+1},S_{t+1})
-\widetilde V_{\ell,m}(Z_t,S_t)$.
In practice, we use $M_s$ additional actor rollouts to form agent $i$'s empirical
policy-gradient estimator:
\begin{equation}
\widehat g_{i,m}
=
\frac{1}{NM_s}
\sum_{j=1}^{M_s}
\sum_{t=0}^{T-1}
\gamma^t
\mathsf{s}_i
\bigl(O_{i,t}^{(j)},A_{i,t}^{(j)}\bigr)
\sum_{\ell\in\mathcal N_i^{\kappa+\kappa_\pi}}
\widehat A_{\ell,t}^{(j)} .
\label{eq:actor-gradient-estimator}
\end{equation}
Equation~\ref{eq:actor-gradient-estimator} estimates agent $i$'s policy-gradient
contribution by retaining only advantages of agents $\ell\in\mathcal N_i^{\kappa+\kappa_\pi}$ which removes the dependence in $N$.  This follows the truncated policy gradient
principle of \citet{qu2020localized,ren2024spectral}, here using localized
one-step TD advantages rather than localized $Q$-values. We then obtain the following \emph{near-stationarity} guarantee.

\begin{theorem}[End-to-end policy optimization]
\label{thm:end-to-end}
Let $\Gamma_T:=(1-\gamma^T)/(1-\gamma)$ and define
\[
\tau_i:=\frac1N\!\sum_{\ell:\,d_G(i,\ell)>\kappa+\kappa_\pi}\rho^{d_G(i,\ell)},\quad
\mathcal T:=\Bigl(\sum_{i=1}^N B_{\mathsf s,i}^2\tau_i^2\Bigr)^{1/2},\quad
\mathcal B:=\frac1N\Bigl(\sum_{i=1}^N
(B_{\mathsf s,i}|\mathcal N_i^{\kappa+\kappa_\pi}|)^2\Bigr)^{1/2}.
\]
Assume that $J$ is $L_J$-smooth. Let $\hat{g}_m = (\hat{g}_{1,m}, \dots, \hat{g}_{N,m})$; with probability at least $1-\delta$, for every $m<M$,
\begin{equation}
\begin{aligned}
\|\widehat g_m-\nabla J(\theta_m)\|
\le \epsilon_m
:={}&
\frac{2R_{\max}}{(1-\gamma)^2}\mathcal T
+\mathcal B\!\left(
\gamma\Gamma_T\Delta_Q^{(m)}
+\frac{2R_{\max}\gamma^T}{(1-\gamma)^2}
\right)\\
&\quad+
G_m\sqrt{\frac{2\log(4M/\delta)}{M_s}},
\end{aligned}
\label{eq:gradient-error}
\end{equation}
where $G_m$ bounds the localized policy-gradient samples (details are given in
Appendix~\ref{app:end-to-end-proof}). If $\epsilon_m\le\bar\epsilon$ and
$G_m\le\bar G$ for all $m<M$, then
$\theta_{m+1}=\theta_m+\eta(m+1)^{-1/2}\widehat g_m$ satisfies
\begin{equation}
\min_{0\le m<M}\|\nabla J(\theta_m)\|^2
\le
O(M^{-1/2})+\bar\epsilon^2
+\bar G^2 O\!\left(\frac{\log M}{\sqrt M}\right).
\label{eq:stationarity-main}
\end{equation}
\end{theorem}

Theorem~\ref{thm:end-to-end} separates localization, critic, rollout-truncation,
and actor-sampling errors, which decrease with localization radius, critic
accuracy, rollout length, and actor sample size, respectively, yielding a
near-stationarity guarantee with an explicit error floor. The 
term $\mathcal T$ sharpens the worst-case truncation bound of
\citet{ren2024spectral} by accounting for each discarded agent $\ell$ with
its own decay factor $\rho^{d_G(i,\ell)}$. This can substantially improve the
bound when $\rho$ is small: if
$\sum_\ell \rho^{d_G(i,\ell)}=O(1)$, then $\mathcal T=O(N^{-1/2})$. In contrast,  discarded agents were previously assigned the same worst-case decay factor $\rho^{\kappa+\kappa_\pi}$ which can produce an $O(1)$ error in each agent's policy-gradient estimate and a corresponding term $\mathcal T$ of order $O(\sqrt N)$.
\vspace{-0.8em}
\section{Algorithm and Complexity}\label{sec:algorithm}

We present the Taylor Centroid actor--critic for LTC and NTC, which uses only rollout data. \vspace{-0.8em}

\begin{algorithm}[H]
\caption{Taylor Centroid Actor--Critic}
\label{alg:tc-actor-critic}
\small
\begin{algorithmic}[1]
\Require Policy $\theta_0$, stepsizes $\eta_m=\eta(m+1)^{-1/2}$,
radii $\kappa,\kappa_\pi$, Taylor order $n$, centroids $K$, number of epochs $M$
\For{$m=0,\ldots,M-1$}
    \State Refit the $K$ centroids from previously observed states.
    \State Collect $M_c$ critic rollouts; in parallel over agents $i$, fit $\widetilde Q_{i,m}$ using LTC or NTC.
    \State Collect $M_s$ additional actor rollouts; in parallel over $i$, compute $\widehat A_{\ell,t}$ for $\ell\in\mathcal N_i^{\kappa+\kappa_\pi}$ and form $\widehat g_{i,m}$.
    \State Set $\widehat g_m=(\widehat g_{1,m},\ldots,\widehat g_{N,m})$ and update $\theta_{m+1}\gets\theta_m+\eta_m\widehat g_m$.
\EndFor
\end{algorithmic}
\end{algorithm} \vspace{-1.4em}
Motivated by Corollary~\ref{cor:distributional-tc}, we fit $K$-medians centroids and refit them at each iteration to track policy-induced shifts in the state distribution (see Appendix~\ref{app:switched-clusters}). Importantly, the time complexity of each LTC iteration scales linearly with \(N\) and is largely parallelizable across agents and samples; a detailed complexity analysis is provided in Appendix~\ref{app:complexity-proof}.

\vspace{-0.7em}
\section{Experiments}
\label{sec:experiments}

We evaluate LTC and NTC on networked LQR
\citep{kalman1960control,fazel2018global,ren2024spectral}, switched LQR
\citep{costa2005markov,ogura2018hidden}, and nonlinear Kuramoto
synchronization \citep{kuramoto1975,dorfler2014synchronization,ren2024spectral}.
Baselines include random Fourier features (RF on LQR and RF-SAC for Kuramoto) and Spectral-SAC \citep{ren2024spectral}, CDCPG\footnote{To satisfy the assumption $\kappa> \kappa_\pi$ of \citet{wang2026continuous}, we choose $\kappa' = \max(\kappa,\kappa_\pi+1)$ in each panel of the plots below for CDCPG only.}
\citep{wang2026continuous}, and generic multilayer perceptron (MLP) and graph neural network (GNN) critics. Recall that RF and CDCPG do not directly apply to hidden switched LQR, which lacks their
required known local dynamics model. We describe the environments in Appendix~\ref{app:experimental-details}; the evaluation and training protocols are given in Appendix~\ref{app:experimental-protocol}.

\paragraph{Critic estimation quality.}
Figure~\ref{fig:critic-calibration} describes the effect of Taylor order on the critic estimation.  On LQR, where $V^\pi$ is quadratic for the affine-mean Gaussian policies considered here, the sharp improvement from $n=1$ to $n=2$ and saturation at $n=3$ match the expected second-order structure.  Switched LQR shows the same pattern
despite the additional difficulty of predicting the mode-conditioned value
$V_\sigma^\pi$ without observing $\sigma$; here NTC significantly improves over LTC,
consistent with its greater expressivity.
Regarding the Kuramoto benchmark, we observe a near-systematic improvement with n, consistent with the fact that the critic is non-quadratic. 
Increasing \(\kappa\) or \(n\) can improve expressivity but makes estimation harder by increasing the number of parameters.  Nevertheless, for every $(\kappa,\kappa_\pi)$, we observe that Taylor Centroid is competitive with the other baselines. Another parameter affecting the critic estimation quality is the number of clusters. We study this effect in Appendix~\ref{app:centroid-ablation}.
\begin{figure}[h!]
    \centering
    \includegraphics[width=1\linewidth]{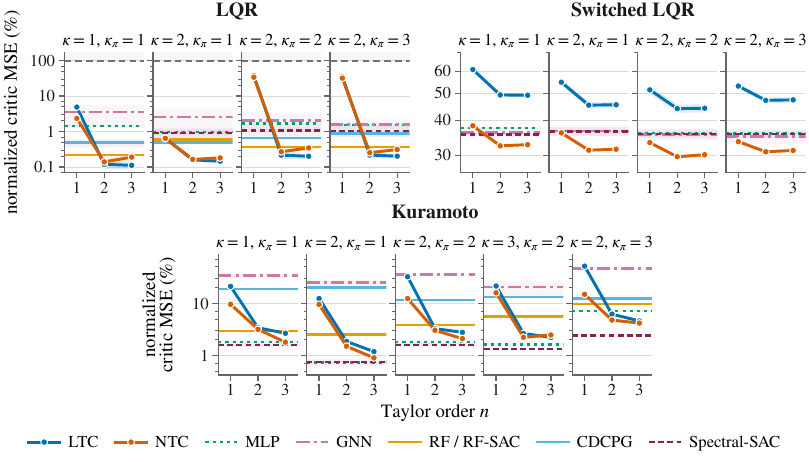}%
    \caption{\textbf{Critic estimation quality.}
Normalized critic MSE across Taylor orders and locality settings on LQR, switched LQR and the Kuramoto oscillator.}
    \label{fig:critic-calibration}
\end{figure}
\paragraph{Policy performance and convergence.}
Figure~\ref{fig:policy-convergence}a shows the evolution of the gap $|J-J^\star|/|J^\star|$ at each epoch on both LQR settings. Grey lines give the best value of $J(\theta)$ found by multistart optimization over each class of affine-mean Gaussian policies with radius $\kappa_\pi$ (a local optimum in general, see Paragraph~\ref{par:policy-perf}). We observe that LTC/NTC matches or outperforms  other baselines across various configurations $(\kappa,\kappa_\pi)$. In particular, on both LQR settings and all configurations $(\kappa,\kappa_\pi)$ but one, LTC/NTC converge to the maximum value mentioned above. Similarly to Figure~\ref{fig:critic-calibration}, we observe that for $\kappa_\pi\ge2$ the performance improves significantly from $n=1$ to $n=2$ (Figure~\ref{fig:taylor-ablation} in Appendix~\ref{app:further-exp}). Figure~\ref{fig:policy-convergence}b shows the evolution of the undiscounted return at each epoch on Kuramoto.  Importantly, NTC matches or outperforms LTC, highlighting the benefit of more expressivity under nonlinear dynamics. Furthermore, in this setting, our methods converge near the greedy controller's value. It is also noteworthy that for both LQR benchmarks, LTC ($n=2$) matches or beats the baselines with few learnable parameters. For instance, LTC $(n=2)$ matches the MLP baseline on LQR and systematically outperforms it on switched LQR, while learning $231$ coefficients per agent at $\kappa=2$ against $17{,}409$ for the MLP (Table~\ref{tab:param-counts}). Finally, our model-free approach systematically matches or outperforms prior spectral baselines across all three environments despite not having access to the transition kernel. Appendix~\ref{app:trajectories}
provides empirical evidence that these policy improvements translate into the intended behavior in the resulting trajectories.
\begin{figure}[h!]
    \centering
    \includegraphics[width=1\linewidth]{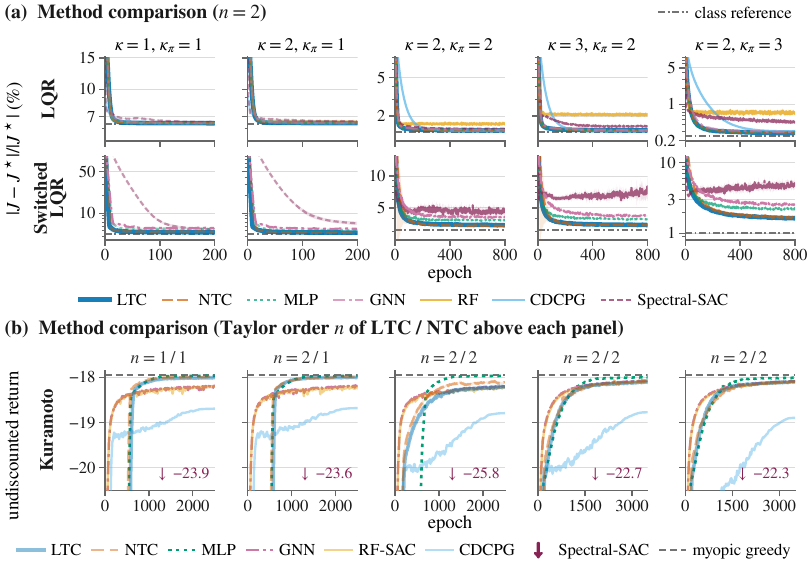}%
    \caption{\textbf{Policy convergence}. Method comparison LQR and switched LQR (relative optimality gap), and policy convergence on Kuramoto (undiscounted return).}
    \label{fig:policy-convergence}
\end{figure}
\paragraph{Scalability.} Figure~\ref{fig:scalability} shows the predicted \(O(N)\) runtime scaling and near-reference performance up to \(N=1000\); in this experiment LTC outperforms RF in runtime.

\begin{figure}[h!]
    \centering
    \includegraphics[width=\linewidth]{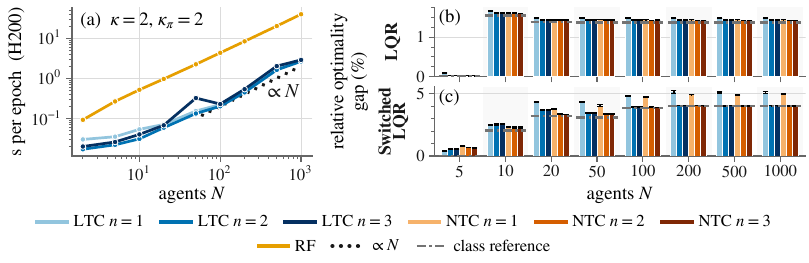}%
    \caption{\textbf{Scalability.}
\textbf{Left:} runtime per epoch on an H200 GPU as a function of the number of agents $N$.
\textbf{Right:} relative optimality gap as a function of $N$ on LQR and switched LQR.}
    \label{fig:scalability}
\end{figure}

\section{Conclusion and limitations}
We showed that graph locality turns otherwise ambient-dimensional Taylor representations into scalable local critics without the transition-model access required by prior constructive spectral approaches. This yields a model-free actor--critic for continuous networked MDPs applicable beyond prior frameworks. While experiments show competitive, scalable performance, some limitations also suggest directions for future work. Our LTC guarantee assumes well-conditioned LSTD systems, whose characterization requires complementary analysis. Moreover, the Taylor-coefficient function class is unknown a priori, precluding sharp control of $\bar\varepsilon_{\mathrm{lin}}$ and requiring empirical feature selection. The policy-optimization guarantee reaches a stationary point of $J$, which may not be a global maximum. $K$-medians can also be challenging in high-dimensional state spaces. Finally, extending finite-sample critic guarantees to NTC is an important but more challenging direction.


\bibliography{biblio}

@article{bradtke-barto-1996,
  title   = {Linear Least-Squares Algorithms for Temporal Difference Learning},
  author  = {Bradtke, Steven J. and Barto, Andrew G.},
  journal = {Machine Learning},
  volume  = {22},
  number  = {1--3},
  pages   = {33--57},
  year    = {1996},
  doi  = {10.1007/BF00114723}
}

@article{lazaric2012lstd,
  title   = {Finite-Sample Analysis of Least-Squares Policy Iteration},
  author  = {Lazaric, Alessandro and Ghavamzadeh, Mohammad and Munos, R{\'e}mi},
  journal = {Journal of Machine Learning Research},
  volume  = {13},
  pages   = {3041--3074},
  year    = {2012}
}

@inproceedings{qu2020localized,
  title     = {Scalable Reinforcement Learning of Localized Policies for Multi-Agent Networked Systems},
  author    = {Qu, Guannan and Wierman, Adam and Li, Na},
  booktitle = {Proceedings of the 2nd Conference on Learning for Dynamics and Control},
  series    = {Proceedings of Machine Learning Research},
  volume    = {120},
  pages     = {256--266},
  publisher = {PMLR},
  year      = {2020}
}

@article{qu2022networked,
  title   = {Scalable Reinforcement Learning for Multiagent Networked Systems},
  author  = {Qu, Guannan and Wierman, Adam and Li, Na},
  journal = {Operations Research},
  volume  = {70},
  number  = {6},
  pages   = {3601--3628},
  year    = {2022},
  doi     = {10.1287/opre.2021.2226}
}

@inproceedings{lin2021stochastic,
  title     = {Multi-Agent Reinforcement Learning in Stochastic Networked Systems},
  author    = {Lin, Yiheng and Qu, Guannan and Huang, Longbo and Wierman, Adam},
  booktitle = {Advances in Neural Information Processing Systems},
  volume    = {34},
  pages     = {7825--7837},
  year      = {2021}
}

@article{zhang2023lpi,
  title   = {Global Convergence of Localized Policy Iteration in Networked Multi-Agent Reinforcement Learning},
  author  = {Zhang, Yizhou and Qu, Guannan and Xu, Pan and Lin, Yiheng and Chen, Zaiwei and Wierman, Adam},
  journal = {Proceedings of the ACM on Measurement and Analysis of Computing Systems},
  volume  = {7},
  number  = {1},
  pages   = {1--51},
  year    = {2023},
  doi     = {10.1145/3579443}
}

@inproceedings{chu2020networked,
  title     = {Multi-Agent Reinforcement Learning for Networked System Control},
  author    = {Chu, Tianshu and Chinchali, Sandeep and Katti, Sachin},
  booktitle = {International Conference on Learning Representations},
  year      = {2020}
}

@inproceedings{olsson2024lqr,
  title     = {Scalable Reinforcement Learning for Linear-Quadratic Control of Networks},
  author    = {Olsson, Johan and Zhang, Runyu and Tegling, Emma and Li, Na},
  booktitle = {2024 American Control Conference},
  pages     = {1813--1818},
  year      = {2024},
  doi       = {10.23919/ACC60939.2024.10644413}
}

@article{wang2026continuous,
  title   = {Scalable Policy Optimization for Networked Multi-Agent Reinforcement Learning with Continuous State-Action Spaces},
  author  = {Wang, Dongming and Dai, Pengcheng and Yu, Wenwu and Ren, Wei},
  journal = {arXiv preprint arXiv:2607.18554},
  year    = {2026}
}

@inproceedings{ren2024spectral,
  title     = {Scalable Spectral Representations for Multiagent Reinforcement Learning in Network {MDP}s},
  author    = {Ren, Zhaolin and Zhang, Runyu and Dai, Bo and Li, Na},
  booktitle = {Proceedings of the 28th International Conference on Artificial Intelligence and Statistics},
  series    = {Proceedings of Machine Learning Research},
  volume    = {258},
  pages     = {550--558},
  publisher = {PMLR},
  year      = {2025}
}

@inproceedings{maran2024smooth,
  title     = {No-Regret Reinforcement Learning in Smooth {MDP}s},
  author    = {Maran, Davide and Metelli, Alberto Maria and Papini, Matteo and Restelli, Marcello},
  booktitle = {Proceedings of the 41st International Conference on Machine Learning},
  series    = {Proceedings of Machine Learning Research},
  volume    = {235},
  pages     = {34760--34789},
  publisher = {PMLR},
  year      = {2024}
}

@inproceedings{maran2024local,
  title     = {Local Linearity: The Key for No-Regret Reinforcement Learning in Continuous {MDP}s},
  author    = {Maran, Davide and Metelli, Alberto Maria and Papini, Matteo and Restelli, Marcello},
  booktitle = {Advances in Neural Information Processing Systems},
  volume    = {37},
   pages = {75986--76029},
  year      = {2024},
  doi       = {10.52202/079017-2419}
}

@inproceedings{qu2020average,
  title     = {Scalable Multi-Agent Reinforcement Learning for Networked Systems with Average Reward},
  author    = {Qu, Guannan and Lin, Yiheng and Wierman, Adam and Li, Na},
  booktitle = {Advances in Neural Information Processing Systems},
  pages = {2074--2086},
  volume    = {33},
  year      = {2020}
}

@inproceedings{ren2022spectral,
  title     = {Spectral Decomposition Representation for Reinforcement Learning},
  author    = {Ren, Tongzheng and Zhang, Tianjun and Lee, Lisa and Gonzalez, Joseph E. and Schuurmans, Dale and Dai, Bo},
  booktitle = {International Conference on Learning Representations},
  year      = {2023}
}

@article{ren2023stochastic,
  title   = {Stochastic Nonlinear Control Via Finite-Dimensional Spectral Dynamics Embedding},
  author  = {Ren, Zhaolin and Ren, Tongzheng and Ma, Haitong and Li, Na and Dai, Bo},
  journal = {IEEE Transactions on Automatic Control},
  volume  = {70},
  number  = {12},
  pages   = {7907--7922},
  year    = {2025},
  doi     = {10.1109/TAC.2025.3580296}
}

@inproceedings{rahimi2007random,
  title     = {Random Features for Large-Scale Kernel Machines},
  author    = {Rahimi, Ali and Recht, Benjamin},
  booktitle = {Advances in Neural Information Processing Systems},
  volume    = {20},
  pages     = {1177--1184},
  year      = {2007}
}

@inproceedings{sutton1999policy,
  title     = {Policy Gradient Methods for Reinforcement Learning with Function Approximation},
  author    = {Sutton, Richard S. and McAllester, David A. and Singh, Satinder P. and Mansour, Yishay},
  booktitle = {Advances in Neural Information Processing Systems},
  volume    = {12},
  pages     = {1057--1063},
  year      = {1999}
}

@book{sutton2018reinforcement,
  title={Reinforcement Learning: An Introduction},
  author={Sutton, Richard S. and Barto, Andrew G.},
  year={2018},
  edition={2},
  publisher={MIT Press}
}

@article{braverman2020taylor,
  title   = {On the Taylor Expansion of Value Functions},
  author  = {Braverman, Anton and Gurvich, Itai and Huang, Junfei},
  journal = {Operations Research},
  volume  = {68},
  number  = {2},
  pages   = {631--654},
  year    = {2020},
  doi     = {10.1287/opre.2019.1903}
}

@inproceedings{xu2020kernel,
  title     = {Kernel Taylor-Based Value Function Approximation for
               Continuous-State Markov Decision Processes},
  author    = {Xu, Junhong and Yin, Kai and Liu, Lantao},
  booktitle = {Robotics: Science and Systems},
  year      = {2020},
  doi       = {10.15607/RSS.2020.XVI.050}
}

@inproceedings{liang2025causality,
  title     = {Causality Meets Locality: Provably Generalizable and Scalable
               Policy Learning for Networked Systems},
  author    = {Liang, Hao and Shi, Shuqing and Zhang, Yudi and Huang, Biwei
               and Du, Yali},
  booktitle = {Advances in Neural Information Processing Systems},
  volume    = {38},
  year      = {2025},
  doi       = {10.52202/085713-5311}
}

@article{kalman1960control,
  title   = {Contributions to the Theory of Optimal Control},
  author  = {Kalman, Rudolf E.},
  journal = {Boletin de la Sociedad Matematica Mexicana},
  volume  = {5},
  pages   = {102--119},
  year    = {1960}
}

@book{costa2005markov,
  title     = {Discrete-Time Markov Jump Linear Systems},
  author    = {Costa, Oswaldo Luiz Valle and Marques, Ricardo Paulino and Fragoso, Marcelo Dutra},
  publisher = {Springer},
  series    = {Probability and Its Applications},
  year      = {2005},
  doi       = {10.1007/b138575}
}

@incollection{kuramoto1975,
  title     = {Self-Entrainment of a Population of Coupled Non-Linear Oscillators},
  author    = {Kuramoto, Yoshiki},
  booktitle = {International Symposium on Mathematical Problems in Theoretical Physics},
  series    = {Lecture Notes in Physics},
  volume    = {39},
  pages     = {420--422},
  publisher = {Springer},
  year      = {1975},
  doi       = {10.1007/BFb0013365}
}

@article{dorfler2014synchronization,
  title   = {Synchronization in Complex Networks of Phase Oscillators: A Survey},
  author  = {D{\"o}rfler, Florian and Bullo, Francesco},
  journal = {Automatica},
  volume  = {50},
  number  = {6},
  pages   = {1539--1564},
  year    = {2014},
  doi     = {10.1016/j.automatica.2014.04.012}
}

@article{dorfler2012power,
  title   = {Synchronization and Transient Stability in Power Networks and Non-Uniform Kuramoto Oscillators},
  author  = {D{\"o}rfler, Florian and Bullo, Francesco},
  journal = {SIAM Journal on Control and Optimization},
  volume  = {50},
  number  = {3},
  pages   = {1616--1642},
  year    = {2012},
  doi  = {10.1137/110851584}
}

@article{wiesenfeld1998josephson,
  title   = {Frequency Locking in Josephson Arrays: Connection with the Kuramoto Model},
  author  = {Wiesenfeld, Kurt and Colet, Pere and Strogatz, Steven H.},
  journal = {Physical Review E},
  volume  = {57},
  number  = {2},
  pages   = {1563--1569},
  year    = {1998},
  doi     = {10.1103/PhysRevE.57.1563}
}

@article{lu2016circadian,
  title   = {Resynchronization of Circadian Oscillators and the East-West Asymmetry of Jet-Lag},
  author  = {Lu, Zhixin and Klein-Carde{\~n}a, Kevin and Lee, Steven and Antonsen, Thomas M. and Girvan, Michelle and Ott, Edward},
  journal = {Chaos},
  volume  = {26},
  number  = {9},
  pages   = {094811},
  year    = {2016},
  doi     = {10.1063/1.4954275}
}

@article{yang2012lqr,
  title   = {Analytic {LQR} Design for Spacecraft Control System Based on Quaternion Model},
  author  = {Yang, Yaguang},
  journal = {Journal of Aerospace Engineering},
  volume  = {25},
  number  = {3},
  pages   = {448--453},
  year    = {2012},
  doi     = {10.1061/(ASCE)AS.1943-5525.0000142}
}

@article{hrovat1997suspension,
  title   = {Survey of Advanced Suspension Developments and Related Optimal Control Applications},
  author  = {Hrovat, Davorin},
  journal = {Automatica},
  volume  = {33},
  number  = {10},
  pages   = {1781--1817},
  year    = {1997},
  doi     = {10.1016/S0005-1098(97)00101-5}
}

@article{zhou2001manipulator,
  title   = {Optimal Construction and Control of Flexible Manipulators: A Case Study Based on {LQR} Output Feedback},
  author  = {Zhou, Pixuan and Wang, Fei-Yue and Chen, Weinong and Lever, Paul},
  journal = {Mechatronics},
  volume  = {11},
  number  = {1},
  pages   = {59--77},
  year    = {2001},
  doi     = {10.1016/S0957-4158(00)00004-0}
}

@article{capron2015process,
  title   = {An Extended Linear Quadratic Regulator with Zone Control and Input Targets},
  author  = {Capron, Bruno Didier Olivier and Odloak, Darci},
  journal = {Journal of Process Control},
  volume  = {29},
  pages   = {33--44},
  year    = {2015},
  doi     = {10.1016/j.jprocont.2015.03.005}
}

@article{carvalho2023power,
  title   = {Linear Quadratic Regulator Design via Metaheuristics Applied to the Damping of Low-Frequency Oscillations in Power Systems},
  author  = {Carvalho, Luis and Neto, Jozias R. L. and Rezende, Jefferson C. and Costa, Marcus V. S. and Fortes, Elenilson V. and Macedo, Leonardo H.},
  journal = {ISA Transactions},
  volume  = {134},
  pages   = {322--335},
  year    = {2023},
  doi     = {10.1016/j.isatra.2022.08.024}
}

@inproceedings{gray2003flight,
  title     = {Closed-Loop Performance Measures for Flight Controllers Subject to Neutron-Induced Upsets},
  author    = {Gray, W. Steven and Zhang, Hong and Gonz{\'a}lez, Oscar R.},
  booktitle = {Proceedings of the 42nd IEEE Conference on Decision and Control},
  volume    = {3},
  pages     = {2465--2470},
  year      = {2003},
  doi       = {10.1109/CDC.2003.1272990}
}

@article{soliman2015power,
  title   = {Robust Stabilisation of Power Systems with Random Abrupt Changes},
  author  = {Soliman, Hisham and Shafiq, Muhammad},
  journal = {IET Generation, Transmission \& Distribution},
  volume  = {9},
  number  = {15},
  pages   = {2159--2166},
  year    = {2015},
  doi     = {10.1049/iet-gtd.2014.1111}
}

@article{yin2017suspension,
  title   = {State Estimation via Markov Switching-Channel Network and Application to Suspension Systems},
  author  = {Yin, Xunyuan and Zhang, Lixian and Ning, Zepeng and Tian, Dapeng and Alsaedi, Ahmed and Ahmad, Bashir},
  journal = {IET Control Theory \& Applications},
  volume  = {11},
  number  = {3},
  pages   = {411--419},
  year    = {2017},
  doi     = {10.1049/iet-cta.2016.1108}
}

@inproceedings{xie2008packet,
  title     = {Stabilizing Sampled-Data Linear Systems with Markovian Packet Losses and Random Sampling},
  author    = {Xie, Li and Xie, Lihua},
  booktitle = {Proceedings of the 17th IFAC World Congress},
  volume    = {41},
  number    = {2},
  pages     = {599--604},
  year      = {2008},
  doi       = {10.3182/20080706-5-KR-1001.00101}
}

@inproceedings{fazel2018global,
  title     = {Global Convergence of Policy Gradient Methods for the Linear Quadratic Regulator},
  author    = {Fazel, Maryam and Ge, Rong and Kakade, Sham M. and Mesbahi, Mehran},
  booktitle = {Proceedings of the 35th International Conference on Machine Learning},
  pages     = {1466--1475},
  volume    = {80},
  year      = {2018},
  publisher = {PMLR}
}

@article{ogura2018hidden,
  title   = {State-feedback Control of Markov Jump Linear Systems with Hidden-Markov Mode Observation},
  author  = {Ogura, Masaki and Cetinkaya, Ahmet and Hayakawa, Tomohisa and Preciado, Victor M.},
  journal = {Automatica},
  volume  = {89},
  pages   = {65--72},
  year    = {2018},
  doi     = {10.1016/j.automatica.2017.11.022}
}

@article{shin2022exponential,
  title   = {Exponential Decay of Sensitivity in Graph-Structured Nonlinear Programs},
  author  = {Shin, Sungho and Anitescu, Mihai and Zavala, Victor M.},
  journal = {SIAM Journal on Optimization},
  volume  = {32},
  number  = {2},
  pages   = {1156--1183},
  year    = {2022},
  doi     = {10.1137/21M1391079}
}

@article{na2020exponential,
  title   = {Exponential Decay in the Sensitivity Analysis of Nonlinear Dynamic Programming},
  author  = {Na, Sen and Anitescu, Mihai},
  journal = {SIAM Journal on Optimization},
  volume  = {30},
  number  = {2},
  pages   = {1527--1554},
  year    = {2020},
  doi     = {10.1137/19M1265065}
}

@article{na2022schwarz,
  title   = {On the Convergence of Overlapping Schwarz Decomposition for Nonlinear Optimal Control},
  author  = {Na, Sen and Shin, Sungho and Anitescu, Mihai and Zavala, Victor M.},
  journal = {IEEE Transactions on Automatic Control},
  volume  = {67},
  number  = {11},
  pages   = {5996--6011},
  year    = {2022},
  doi     = {10.1109/TAC.2022.3194087}
}

@inproceedings{konda2000actor,
  title     = {Actor-Critic Algorithms},
  author    = {Konda, Vijay R. and Tsitsiklis, John N.},
  booktitle = {Advances in Neural Information Processing Systems},
  volume    = {12},
  pages     = {1008--1014},
  year      = {1999}
}

@article{sutton1988learning,
  title={Learning to Predict by the Methods of Temporal Differences},
  author={Sutton, Richard S.},
  journal={Machine Learning},
  volume={3},
  number={1},
  pages={9--44},
  year={1988},
  doi={10.1007/BF00115009}
}

@article{gu2022meanfield,
  title   = {Mean-Field Multiagent Reinforcement Learning: A Decentralized Network Approach},
  author  = {Gu, Haotian and Guo, Xin and Wei, Xiaoli and Xu, Renyuan},
  journal = {Mathematics of Operations Research},
  volume  = {50},
  number  = {1},
  pages   = {506--536},
  year    = {2025},
  doi     = {10.1287/moor.2022.0055}
}

@book{fan1996local,
  title     = {Local Polynomial Modelling and Its Applications},
  author    = {Fan, Jianqing and Gijbels, Irene},
  year      = {1996},
  publisher = {Chapman \& Hall},
  series    = {Monographs on Statistics and Applied Probability},
  volume    = {66}
}

@article{atkeson1997locally,
  title   = {Locally Weighted Learning},
  author  = {Atkeson, Christopher G. and Moore, Andrew W. and Schaal, Stefan},
  journal = {Artificial Intelligence Review},
  volume  = {11},
  number  = {1--5},
  pages   = {11--73},
  year    = {1997},
  doi     = {10.1023/A:1006559212014}
}

@article{byrd1995limited,
  title={A Limited Memory Algorithm for Bound Constrained Optimization},
  author={Byrd, Richard H. and Lu, Peihuang and Nocedal, Jorge and Zhu, Ciyou},
  journal={SIAM Journal on Scientific Computing},
  volume={16},
  number={5},
  pages={1190--1208},
  year={1995},
  publisher={SIAM},
  doi= {10.1137/0916069}
}

@article{cardot2012kmedians,
  title   = {A Fast and Recursive Algorithm for Clustering Large Datasets with K-Medians},
  author  = {Cardot, Herv{\'e} and C{\'e}nac, Peggy and Monnez, Jean-Marie},
  journal = {Computational Statistics \& Data Analysis},
  volume  = {56},
  number  = {6},
  pages   = {1434--1449},
  year    = {2012},
  doi     = {10.1016/j.csda.2011.11.019}
}

@inproceedings{he2015delving,
  title     = {Delving Deep into Rectifiers: Surpassing Human-Level Performance on ImageNet Classification},
  author    = {He, Kaiming and Zhang, Xiangyu and Ren, Shaoqing and Sun, Jian},
  booktitle = {Proceedings of the IEEE International Conference on Computer Vision},
  pages     = {1026--1034},
  year      = {2015},
  doi       = {10.1109/ICCV.2015.123}
}

@inproceedings{kingma2015adam,
  title     = {Adam: A Method for Stochastic Optimization},
  author    = {Kingma, Diederik P. and Ba, Jimmy},
  booktitle = {International Conference on Learning Representations},
  year      = {2015}
}

@article{pinelis1994optimum,
  title   = {Optimum Bounds for the Distributions of Martingales in Banach Spaces},
  author  = {Pinelis, Iosif},
  journal = {The Annals of Probability},
  volume  = {22},
  number  = {4},
  pages   = {1679--1706},
  year    = {1994},
  doi     = {10.1214/aop/1176988477}
}

@inproceedings{jin2020provably,
  title     = {Provably Efficient Reinforcement Learning with Linear Function Approximation},
  author    = {Jin, Chi and Yang, Zhuoran and Wang, Zhaoran and Jordan, Michael I.},
  booktitle = {Proceedings of the 33rd Conference on Learning Theory},
  pages     = {2137--2143},
  year      = {2020},
  volume    = {125},
  series    = {Proceedings of Machine Learning Research}
}
\bibliographystyle{biblio}

\appendix
\numberwithin{assumption}{section}
\numberwithin{theorem}{section}
\numberwithin{proposition}{section}
\numberwithin{lemma}{section}
\numberwithin{corollary}{section}
\numberwithin{definition}{section}

\section{Appendix}
\label{app:overview}
\paragraph{Overview}
Appendix~\ref{app:experimental-details}, \emph{``Definition of the environments''}, describes the three benchmark environments and the role played by each.

Appendix~\ref{app:experimental-protocol}, \emph{``Training and Evaluation Protocols''}, gives the evaluation protocol, implementation details, and hyperparameters of Taylor Centroid and of the competing baselines.

Appendix~\ref{app:further-exp}, \emph{``Further Experimental Results''}, reports the additional
experiments and ablations.

Appendix~\ref{app:regularity},
\emph{``Understanding the Expected Reward Regularity Assumption''}, studies the
regularity condition used throughout the theoretical analysis. It gives its
interpretation, describes a rollout-based empirical stress test,
presents a weaker discounted-summability formulation, and provides concrete
sufficient conditions under which the assumption holds. It also positions the condition relative to the smooth-MDP hierarchy of \citet{maran2024local,maran2024smooth}.

Appendix~\ref{app:proofs},
\emph{``Proofs of the Main-Text Results''}, contains the complete proofs of the
results stated in the main text. It proceeds through finite information
propagation and the exponential derivative decay, the localized Taylor Centroid
approximation, the finite-sample LTC--LSTD guarantee, the localized
policy-gradient and one-step TD analysis, the end-to-end policy-optimization guarantee,
and the computational-complexity result.

Appendix~\ref{app:auxiliary-proofs},
\emph{``Proofs of the Additional Appendix Results''}, contains the proofs of
the auxiliary results used in Appendix~\ref{app:regularity},
\emph{``Understanding the Expected Reward Regularity Assumption''}. These
include the discounted-summability extension, the sufficient conditions based
on smooth stochastic and spectral transition models, and the comparison with
the smoothness assumptions of
\citet{maran2024local,maran2024smooth}.

Finally, Appendix~\ref{app:ntc},
\emph{``Neural Network Implementation Details''}, describes the implementation of
neural network methods, including NTC.

\section{Definition of the environments}
\label{app:experimental-details}

This appendix provides the complete description and motivation of the
experimental benchmarks used in Section~\ref{sec:experiments}. The three
environments are deliberately complementary. Networked LQR provides a
controlled setting in which the exact critic is quadratic and the assumptions
underlying the constructive spectral baseline are satisfied. It is also a canonical environment for the random Fourier feature setting studied by \citet{ren2024spectral}. Switched LQR introduces potentially unobserved operating regimes through $\sigma_t$ and therefore tests the ability of the representation to exploit regime structure without access to the transition model. Finally, the Kuramoto oscillator removes the linear--quadratic structure altogether and tests Taylor Centroid critics on nonlinear graph-coupled dynamics.

\paragraph{Remark.} As in \citet{ren2024spectral}, although some of the benchmarks are naturally defined on unbounded state spaces, over any finite experimental horizon their trajectories lie in a sufficiently large compact set with arbitrarily high probability. One can therefore view the experiments as evolving inside such a compact set; equivalently, the dynamics could be clipped at its boundary without affecting the reported trajectories except on an arbitrarily low-probability event.

\subsection{Environment 1: Networked linear--quadratic regulation}
\label{app:lqr-environment}

\paragraph{Environment.}
We first consider a networked linear--quadratic regulator (LQR)
\citep{kalman1960control}. Here, we consider a ring of $N$ scalar subsystems,
with dynamics
\begin{equation}
s_{t+1}
=
A s_t + B a_t + \varepsilon_t,
\qquad
\varepsilon_t\sim\mathcal N(0,I_N), \qquad \gamma=0.75
\label{eq:exp-lqr}
\end{equation}
over an episode length of $T=20$. The matrices $A$ and $B$ are cyclic and banded, so each subsystem interacts
only with its two immediate neighbors. With indices understood modulo $N$,
the nonzero entries are
\[
A_{ii}=0.8,
\qquad
A_{i,i\pm1}=0.1,
\qquad
B_{ii}=1,
\qquad
B_{i,i\pm1}=0.7.
\]
All remaining entries are zero. The local stage reward is
\begin{equation}
r_i(s_i,a_i)=-(3s_i^2+a_i^2).
\label{eq:exp-lqr-cost}
\end{equation}

Thus, agents are encouraged to minimize their state value. LQR is one of the classical models of optimal control and is routinely used
either directly or as a local approximation of more complicated physical
systems. Applications include spacecraft attitude control
\citep{yang2012lqr}, automotive suspension
\citep{hrovat1997suspension}, robotic manipulators
\citep{zhou2001manipulator}, industrial process control
\citep{capron2015process}, and power-system stabilization
\citep{carvalho2023power}. In many such applications, the LQR model results
from linearizing nonlinear dynamics around a nominal operating point. This environment was also used as a benchmark in~\cite{ren2024spectral}'s released code.

\paragraph{Role of LQR for the benchmark.}
This benchmark serves two purposes. First, under our policy class the
true value function is exactly quadratic. It therefore provides a controlled
test of the Taylor mechanism.

Second, Equation~\ref{eq:exp-lqr} is a known local dynamics map with
additive, state-independent Gaussian noise, precisely the constructive case
for which both the random Fourier features of
\citet{ren2024spectral} and CDCPG of \citet{wang2026continuous} are naturally defined. LQR therefore provides a comparison in which the spectral method is evaluated inside its intended model
class, while Taylor Centroid uses only sampled trajectories.

\subsection{Environment 2: Switched LQR with an unobserved operating mode}
\label{app:switched-lqr-environment}

Many engineered systems do not evolve around a single stationary operating
point. Their dynamics can change abruptly because of faults, topology
changes, communication conditions, workloads, or external operating modes.
Switched LQR systems with Markovian switching provide a classical model for such phenomena
\citep{costa2005markov}. They have been used, for example, in fault-tolerant
flight control \citep{gray2003flight}, power systems subject to abrupt faults
or topology changes \citep{soliman2015power}, automotive systems with
switching communication conditions \citep{yin2017suspension}, and networked
control under correlated packet losses \citep{xie2008packet}.

This setting is particularly relevant to the motivation behind
Taylor Centroid critics: a physical system may repeatedly move between a
small number of qualitatively different operating regimes even though the
underlying state space is continuous. This is the motivation behind the centroids around which our Taylor expansions are performed.

\paragraph{Environment.}
We consider the same sparse ring structure as in the LQR benchmark, but let a
persistent mode
\[
Z_t:=\sigma_t\in\{1,\ldots,M_{\mathrm{mode}}\},
\qquad M_{\mathrm{mode}}=5,
\]
modify the dynamics:
\begin{equation}
s_{t+1}
=
A^{(\sigma_t)}s_t
+
B^{(\sigma_t)}a_t
+
C^{(\sigma_t)}
+
\varepsilon_t,
\qquad
\varepsilon_t\sim\mathcal N(0,I_N).
\label{eq:exp-switched-lqr}
\end{equation}
over an episode length of $T=20$. Each of the five matrices induces graph-local interactions, so switching changes the local physical behavior without changing the interaction graph. 

The mode follows an autonomous Markov chain: at each step, with probability $h$, the mode is redrawn (possibly to its current value). When redrawing, $\sigma'$ is drawn with probability $(6-\sigma')/15$. Hence
\begin{equation}
\Pr(\sigma_{t+1}=\sigma'\mid\sigma_t=\sigma)
=
\begin{cases}
1-h+h\,\dfrac{6-\sigma}{15}, & \sigma'=\sigma,\\[2mm]
h\,\dfrac{6-\sigma'}{15}, & \sigma'\neq\sigma.
\end{cases}
\label{eq:exp-mode-transition}
\end{equation}

 We choose $h\approx0.139$ so that a length-$20$ trajectory contains two mode changes on average. The local rewards and discount factor are identical
to those of the LQR benchmark. Our main experiments keep $\sigma_t$ completely \emph{unobserved}: neither the actor nor the critic receives the current mode.

\paragraph{The constructive RF and CDCPG baselines are not directly applicable here.}
The distinction between observed and unobserved switching is important.
Conditional on a known mode $\sigma_t=\sigma$,
Equation~\ref{eq:exp-switched-lqr} is indeed linear--Gaussian. However, in our experiments, the
controller does not observe $\sigma_t$. After marginalizing the persistent
mode, there is no single observable transition of the form
\begin{equation}
s_{t+1}
=
f(s_t,a_t)+\varepsilon_t,
\qquad
\varepsilon_t\sim\mathcal N(0,\Sigma)
\quad\text{independently over time},
\label{eq:rf-required-form}
\end{equation}
with a known deterministic local map $f$ and fixed homoskedastic noise.

This matters directly for the constructive spectral methods of
\citet{ren2024spectral} and \citet{wang2026continuous}. Both approaches build
their local critic representations from explicit random features of the
transition kernel and require access to the local dynamics map under structured
additive noise. We therefore do not include the constructive RF baseline of
\citet{ren2024spectral} nor CDCPG of \citet{wang2026continuous} on this
benchmark.

\paragraph{Role of switched LQR for the benchmark.}
Switched LQR highlights the flexibility of the model-free approach by evaluating Taylor Centroid in a setting where the dynamics are unknown to the learner. It also tests whether the method can extract from rollout data the information relevant for critic estimation and control when important aspects of the environment, such as the operating regime, are not directly observed.

\subsection{Environment 3: Kuramoto oscillator network}
\label{app:kuramoto-environment}

Our final benchmark removes the linear--quadratic structure entirely. The
Kuramoto model~\citep{kuramoto1975,dorfler2014synchronization} is a canonical model of synchronization in networks of coupled
oscillators. Closely related models arise in electrical power networks \citep{dorfler2012power},
Josephson-junction arrays \citep{wiesenfeld1998josephson}, and biological
oscillator networks \citep{lu2016circadian}.

Following the implementation of the Kuramoto benchmark of \citet{ren2024spectral}, the state of
oscillator $i$ is represented by
\[
s_i=(s_{i,1},s_{i,2})
=(\cos\theta_i,\sin\theta_i)\in\mathbb R^2.
\]
The phase is recovered as
\[
\theta_i=\operatorname{atan2}(s_{i,2},s_{i,1}),
\]
where $\operatorname{atan2}(y,x)\in(-\pi,\pi]$ is the polar angle of the vector
$(x,y)$. Its pulsation is
\[
\dot\theta_i
=
\omega_i+a_i
-
\sum_{j\in\mathcal N_i}
P_{ij}\sin(\theta_i-\theta_j),
\]
where $\omega_i$ is the natural frequency, $a_i$ is the control input, and
$P_{ij}$ is the coupling strength. The state dynamics are then
\begin{equation}
\label{eq:exp-kuramoto}
\begin{aligned}
s_{i,1}^{+}
&=
\cos\!\left(\theta_i+dt\,\dot\theta_i\right)
+\sigma\xi_{i,1},\\
s_{i,2}^{+}
&=
\sin\!\left(\theta_i+dt\,\dot\theta_i\right)
+\sigma\xi_{i,2},
\qquad
(\xi_{i,1},\xi_{i,2})\sim\mathcal N(0,I_2).
\end{aligned}
\end{equation}
As in \citet{ren2024spectral}, we clip the two state coordinates to
$[-1,1]$ in the implementation. We additionally impose a negligible radial
safeguard before evaluating $\operatorname{atan2}$ to avoid its singularity at
the origin. These are purely numerical safeguards and are essentially never
active in practice, since each deterministic update lies on the unit circle
and the process noise is small. We use $dt=0.01$, $\gamma=0.99$, and episodes of length $T=500$. Following the released code of \citet{ren2024spectral}, the oscillators form a ring, the natural frequencies $\omega_i\sim\mathcal U[-0.5,0.5]$ and the symmetric couplings $P_{ij}\sim\mathcal U[0.2,1.2]$ are drawn once per network size and shared by all runs, $\omega^\star=0.2$, $\sigma=0.0025$, and actions take values in $[-1,1]$. Thus the Kuramoto dynamics are evolved in phase coordinates,
mapped back through $(\cos,\sin)$, and then perturbed by isotropic Gaussian
noise directly in the two state coordinates.

The control objective is to synchronize the oscillators at a prescribed
angular frequency $\omega^\star$. Since $\dot\theta_i$ is precisely the
instantaneous angular frequency of oscillator $i$, we use the smooth
frequency-mismatch reward
\[
r_i(s_{\mathcal N_i},a_i)
=
-\sqrt{(\dot\theta_i-\omega^\star)^2+\varepsilon^2},
\]
where $\varepsilon=10^{-6}$ smooths the absolute-value loss; Spectral-SAC and RF-SAC keep $\varepsilon=0$, as in \citet{ren2024spectral}'s code. Maximizing this reward
therefore drives $\dot\theta_i$ toward $\omega^\star$. This reward depends on the neighbors of agent $i$, a case covered by the remark of Section~\ref{sec:setting}.

\paragraph{Role of the Kuramoto oscillator for the benchmark.}
Unlike LQR, the transition dynamics contain nonlinear sinusoidal interactions
and the critic is non-quadratic. The benchmark
therefore tests whether a \emph{local} Taylor representation remains
useful when the value function is more complex than in LQR. It also provides a natural stress test of the regime mechanism. Because the true critic is non-polynomial, the number of clusters plays a significant role. 

\section{Training and Evaluation Protocols}
\label{app:experimental-protocol}

This appendix gives the training, implementation, and evaluation details for
the experiments of Section~\ref{sec:experiments}. Environment definitions and
their role in the benchmark are given in
Appendix~\ref{app:experimental-details}. Detailed neural architectures are
given separately in Appendix~\ref{app:ntc}.

\subsection{Evaluation protocols}

\paragraph{Experimental outline.}
We compare LTC and NTC with five reference baselines whenever applicable:
(i) the constructive random-feature (RF) algorithm of
\citet{ren2024spectral},
(ii) their model-free Spectral-SAC alternative,
(iii) CDCPG of \citet{wang2026continuous},
(iv) a local MLP critic with our one-step TD actor, and
(v) a GNN critic paired with the actor of \citet{ren2024spectral} on LQR and Kuramoto (SAC) and our actor on Switched LQR.
The first three baselines compare Taylor Centroid against methods designed
specifically for networked MARL. The MLP provides a generic function-approximation baseline while the GNN uses information from the graph. For the Kuramoto oscillator, we use the SAC variant of RF (RF-SAC) following the approach of~\cite{ren2024spectral}. Unless stated otherwise, we use $N=50$ agents on LQR and switched LQR and $N=40$ on Kuramoto.

\paragraph{Critic estimation.}
Figure~\ref{fig:critic-calibration} reports the normalized MSE of the estimated global value
$V(s)=N^{-1}\sum_i V_i(s)$,
\[
100\,
\frac{\mathbb E[(\widetilde V(S)-V(S))^2]}
{\operatorname{Var}(V(S))},
\]
evaluated on an independent rollout.
For LQR, the target is the exact $V^\pi$; for switched LQR, it is the exact
mode-conditioned value (i.e. the value function that can observe $\sigma_t$ and which we want to recover); and for
Kuramoto it is estimated by Monte Carlo.

For each environment and $(\kappa,\kappa_\pi)$, we first train a policy with
Algorithm~\ref{alg:tc-actor-critic} for $80$ epochs on LQR, $120$ on switched
LQR, and $20$ on Kuramoto, and then fit the critic with this policy. For each configuration, we train five critics with independent random seeds and evaluate each of them on a separate rollout.

\paragraph{Policy performance.}
\label{par:policy-perf}
For LQR and switched LQR, Figure~\ref{fig:policy-convergence} reports the relative optimality gap

$$
100\,\frac{|J-J^\star|}{|J^\star|},
$$

where \(J\) is the expected sum of discounted rewards of the deterministic policy. For LQR, \(J^\star\) is given by the discounted Riccati solution; for switched LQR, where \(\sigma_t\) is unobserved, it is obtained by directly maximizing $J$ over mode-blind affine controllers with circulant gains. Each $\kappa_\pi$-class reference is similarly obtained by multistart L-BFGS-B~\citep{byrd1995limited} on the exact $J$ and its exact gradient ($483$ starts per $\kappa_\pi$), and is a strict local maximum (vanishing gradient, negative-definite Hessian on the admissible gains). Structured static feedback is non-convex, so these values, and $J^\star$ on switched LQR, are only lower bounds on the corresponding optima.
For Kuramoto, where $J^\star$ is not available in closed form, we instead
report the undiscounted return of the deterministic policy, following
\citet{ren2024spectral}, with the myopic greedy controller (which chooses the actions so as to cancel the instantaneous frequency error of each oscillator) as reference.
Results are averaged over five independent seeds on LQR and switched LQR, and over three seeds on Kuramoto.

We evaluate the following configurations
\[
(\kappa,\kappa_\pi)
\in
\{(1,1),(2,1),(2,2),(3,2),(2,3)\},
\]
covering $\kappa<\kappa_\pi$, $\kappa=\kappa_\pi$, and
$\kappa>\kappa_\pi$.
We note that CDCPG requires $\kappa\ge\kappa_\pi+1$ for its policy-gradient estimator
\citep{wang2026continuous}. Hence, for CDCPG only, we use $\kappa'=\max\{\kappa,\kappa_\pi+1\}$, the smallest radius allowed by its analysis, while keeping $\kappa_\pi$ and
therefore the policy class unchanged across methods. For this figure, we use $N=50$ for LQR and switched LQR and $N=40$ for the Kuramoto oscillator.

\paragraph{Scalability.}
Runtime measurements use a single dedicated NVIDIA H200 GPU at
$(\kappa,\kappa_\pi)=(2,2)$ and exclude the first epoch.
Policy-scaling experiments use three seeds, up to $N=1000$ for LQR and
switched LQR. Performance is compared with $J^\star$ and with the $\kappa_\pi=2$ class reference described above.

\subsection{Common actor and training protocol}
\label{app:common-protocol}
On LQR, we use a Gaussian policy with linear mean \[ \mu_\theta(s)=Fs, \] and fixed covariance $\sigma_{\mathrm{pol}}^2 I$, where the $i$th row of $F$ depends only on $s_{\mathcal N_i^{\kappa_\pi}}$. On switched LQR, we instead use a Gaussian policy with affine mean \[ \mu_\theta(s)=Fs+f, \] with the same locality restriction on $F$ and fixed covariance $\sigma_{\mathrm{pol}}^2 I$. The switching mode $\sigma_t$ remains unobserved by both the actor and the critic. In both LQR environments, $\sigma_{\mathrm{pol}}=0.3$. On Kuramoto, each agent uses a Gaussian policy with local affine mean $\mu_{\theta_i}(o_i)$, and actions are sampled as \[ a_i = \operatorname{clip} \bigl( \mu_{\theta_i}(o_i) +\sigma_{\mathrm{pol}}\varepsilon_i, -1,1 \bigr), \qquad \varepsilon_i\sim\mathcal N(0,1), \] where $\sigma_{\mathrm{pol}}$ decreases log-linearly from $0.1$ to $0.012$ over the first $55\%$ of iterations, then stays constant. Actor updates use $\eta_m=\eta/\sqrt{m+1}$; the policy gradient is normalized to unit norm on LQR and Kuramoto and clipped at norm $1$ on switched LQR. The runs of Figure~\ref{fig:policy-convergence} last $200$ actor iterations for $\kappa_\pi=1$ and $800$ for $\kappa_\pi\ge2$ on LQR and switched LQR, and $2500$ on Kuramoto ($3500$ at $(3,2)$ and $(2,3)$). \\
Advantages are computed using the one-step TD estimator described in Section~\ref{sec:learning}. Critic and actor rollouts are sampled independently. Unless stated otherwise, each LQR and switched-LQR iteration uses $2500$ episodes in total, split between fitting the critic and updating the actor. LTC, NTC, MLP, and RF use $1250$ episodes for each. The GNN uses $1670/830$ critic/actor episodes for $\kappa\le2$ and $1750/750$ for $\kappa=3$. On Kuramoto, each iteration uses $300$ critic (raised when needed to reach $50$ critic transitions per critic parameter) and $2000$ actor episodes of length $500$; Spectral-SAC, RF-SAC, and the GNN keep the training schedule of the released code of \citet{ren2024spectral}.

\subsection{Method-specific details}

Architectures and training objectives of the neural critics are described in
Appendix~\ref{app:ntc}.

\paragraph{Taylor Centroid critics.}
LTC and NTC use all monomials of order $n\in\{1,2,3\}$ supported on
$\mathcal N_i^\kappa$, expanded around $K$-medians centroids. LTC is fitted by
LSTD with ridge $\lambda=10^{-6}$ on LQR and switched LQR, and with the ridge
of Table~\ref{tab:kur-hyper} on Kuramoto.NTC keeps the same basis, with its own number of centroids (Tables~\ref{tab:lqr-hyper}--\ref{tab:kur-hyper}), and replaces the affine coefficient map by a neural network trained
by semi-gradient TD. On LQR and switched LQR, its value estimates average $16$
action samples, and it uses $1500$ TD updates in the standalone
critic-estimation experiment.

\paragraph{Generic critics.}
The MLP critic is paired with the same model-free actor as Taylor Centroid.
It has two hidden layers of width $128$ on LQR and switched LQR and $64$ on
Kuramoto, and is trained with Adam~\citep{kingma2015adam} at learning rate $2\times10^{-3}$, using
$200$ TD updates per iteration on batches of $4096$ transitions ($50$ updates
on batches of $1024$ on Kuramoto) and a target network synchronized every $25$
updates; on LQR and switched LQR, gradients are clipped at norm $1$. The GNN
critic, whose receptive field matches $\kappa$ on LQR and switched LQR and is $\min\{\kappa,2\}$ on Kuramoto, is paired with the actor of the RF baseline on LQR and Kuramoto (and ours on switched LQR): RF's model-based policy gradient with
stepsize $\eta=0.2$ on LQR, and the SAC training loop of
RF-SAC on Kuramoto, where the GNN replaces the fixed random-feature map of its
critic. On LQR and switched LQR, it has $\max(2,\kappa)$ layers of width $128$ shared across agents and it is trained like the MLP ($200$ TD
updates, batches of up to $4096$ transitions, target synchronization every
$25$ updates, gradient clipping at $1$); on Kuramoto, it has width $64$ and
follows the RF-SAC training loop. 

\paragraph{Networked-MARL baselines.}
Following~\citep{ren2024spectral}, on LQR, RF uses $200$ cosine features at $\kappa=1$ and $800$ at $\kappa\ge2$,
with actor stepsize $\eta=0.2$; on Kuramoto, its SAC variant RF-SAC uses
$1024$ features. Spectral-SAC uses $512$ random features, hidden width $256$, batch size $128$, $128$ parallel environments, one update per environment step, and automatic entropy tuning with initial coefficient $0.1$. On LQR and switched LQR, its actor and critic learning rates are $2\times10^{-4}$, its actor is linear and unbounded with action scale $3$, and its replay capacity is $10^7$ for $\kappa_\pi=1$ and $4\times10^7$ for $\kappa_\pi\ge2$. On Kuramoto, its actor is a tanh-Gaussian MLP with two hidden layers of width $256$ (three for RF-SAC and the GNN), its replay capacity is $2\times10^7$, and the learning rates are $10^{-3}$ for $(\kappa,\kappa_\pi)\in\{(1,1),(2,1),(2,2)\}$ and, at $(3,2)$ and $(2,3)$, $5\times10^{-4}$ for Spectral-SAC and $3\times10^{-4}$ for RF-SAC and the GNN.
On LQR, CDCPG uses ridge $10^{-4}$, policy standard deviation $0.3$, $\theta_{\max}=1$ and actor stepsize $3\times10^{-4}$; its $12{,}500$ trajectories per iteration, of mean length $1/(1-\gamma)=4$, match our $2500\times20$ transitions. On Kuramoto, it uses policy standard deviation $0.05$, $\theta_{\max}=1.5$ and stepsize $10^{-5}$. Its
critic radius is $\max\{\kappa,\kappa_\pi+1\}$, the smallest value allowed by
its analysis, while $\kappa_\pi$, and hence the policy class, is unchanged.

\paragraph{Parameter count} 
Table~\ref{tab:param-counts} reports the number of learned parameters of every critic.

\begin{table}[h!]
\centering
\caption{Learned parameters of each critic in the policy experiments of Figure~\ref{fig:policy-convergence}, per agent unless marked as total ($N=50$ on LQR and switched LQR, $N=40$ on Kuramoto). LTC counts its coefficients, RF, RF-SAC and CDCPG their linear weights, and the other critics their network weights. The GNNs share their weights across agents and are reported as totals; twin critics count both copies.}
\label{tab:param-counts}
\small
\setlength{\tabcolsep}{5pt}
\begin{tabular}{lrrrrr}
\toprule
$(\kappa,\kappa_\pi)$ & $(1,1)$ & $(2,1)$ & $(2,2)$ & $(3,2)$ & $(2,3)$ \\
\midrule
\multicolumn{6}{l}{\emph{LQR and switched LQR}} \\
LTC, $n=1/2/3$ & 28/70/140 & 66/231/616 & 66/231/616 & 120/540/1{,}800 & 66/231/616 \\
NTC, $n=1$ & 4{,}932 & 5{,}318 & 5{,}318 & 5{,}704 & 5{,}318 \\
NTC, $n=2$ & 5{,}322 & 6{,}293 & 6{,}293 & 7{,}524 & 6{,}293 \\
NTC, $n=3$ & 5{,}972 & 8{,}568 & 8{,}568 & 12{,}984 & 8{,}568 \\
MLP & 17{,}153 & 17{,}409 & 17{,}409 & 17{,}665 & 17{,}409 \\
GNN (total) & 17{,}025 & 33{,}409 & 33{,}409 & 66{,}305 & 33{,}409 \\
RF (LQR) & 200 & 800 & 800 & 800 & 800 \\
CDCPG (LQR) & 120 & 120 & 190 & 190 & 276 \\
Spectral-SAC & 200{,}192 & 201{,}216 & 201{,}216 & 202{,}240 & 201{,}216 \\
\midrule
\multicolumn{6}{l}{\emph{Kuramoto}} \\
LTC, $n=1$ & 70 & 176 & 176 & 330 & 176 \\
LTC, $n=2$ & 280 & 1{,}056 & 1{,}056 & 2{,}640 & 1{,}056 \\
LTC, $n=3$ & 840 & 4{,}576 & 4{,}576 & 14{,}960 & 4{,}576 \\
NTC, $n=1$ & 5{,}319 & 5{,}963 & 5{,}963 & 6{,}607 & 5{,}963 \\
NTC, $n=2$ & 6{,}684 & 9{,}538 & 9{,}538 & 13{,}432 & 9{,}538 \\
NTC, $n=3$ & 10{,}324 & 23{,}838 & 23{,}838 & 49{,}832 & 23{,}838 \\
MLP & 4{,}673 & 4{,}929 & 4{,}929 & 5{,}185 & 4{,}929 \\
GNN (total) & 17{,}282 & 17{,}282 & 17{,}282 & 17{,}282 & 17{,}282 \\
RF-SAC & 4{,}096 & 4{,}096 & 4{,}096 & 4{,}096 & 4{,}096 \\
CDCPG & 257 & 257 & 257 & 257 & 257 \\
Spectral-SAC & 200{,}960 & 202{,}496 & 202{,}496 & 204{,}032 & 202{,}496 \\
\bottomrule
\end{tabular}
\end{table}

\subsection{Hyperparameter selection}

Tables~\ref{tab:lqr-hyper}--\ref{tab:kur-hyper} report the selected
hyperparameters. On both LQR settings, the GNN keeps the fixed stepsize $\eta=0.2$, while we follow the configuration designs of \citet{ren2024spectral} for RF, RF-SAC and Spectral-SAC.

\begin{table}[h]
\centering
\caption{Selected $(\eta,K)$ on LQR.}
\label{tab:lqr-hyper}
\small
\setlength{\tabcolsep}{5pt}
\begin{tabular}{lcccc}
\toprule
$(\kappa,\kappa_\pi)$ & $(1,1)$ & $(2,1)$ & $(2,2)$, $(3,2)$ & $(2,3)$ \\
\midrule
LTC $n=1$ & $(0.50,40)$  & $(0.35,40)$  & $(0.50,100)$  & $(0.50,100)$ \\
LTC $n=2$ & $(0.35,25)$  & $(0.35,25)$  & $(0.50,40)$   & $(0.50,40)$ \\
LTC $n=3$ & $(0.35,25)$  & $(0.35,40)$  & $(0.50,40)$   & $(0.50,40)$ \\
NTC $n=1$ & $(0.50,400)$ & $(0.50,400)$ & $(0.50,2000)$ & $(0.50,2000)$ \\
NTC $n=2$ & $(0.35,40)$  & $(0.35,25)$  & $(0.50,40)$   & $(0.50,40)$ \\
NTC $n=3$ & $(0.35,40)$  & $(0.35,25)$  & $(0.50,25)$   & $(0.50,25)$ \\
MLP $\eta$ & $0.35$ & $0.35$ & $0.50$ & $0.50$ \\
\bottomrule
\end{tabular}
\end{table}

\begin{table}[h]
\centering
\caption{Selected $(\eta,K)$ on switched LQR.}
\label{tab:mjls-hyper}
\small
\setlength{\tabcolsep}{5pt}
\begin{tabular}{lcccc}
\toprule
$(\kappa,\kappa_\pi)$ & $(1,1)$ & $(2,1)$ & $(2,2)$, $(3,2)$ & $(2,3)$ \\
\midrule
LTC $n=1$ & $(0.35,40)$ & $(0.35,40)$ & $(0.50,40)$ & $(0.50,40)$ \\
LTC $n=2$ & $(0.50,25)$ & $(0.50,25)$ & $(0.50,25)$ & $(0.50,25)$ \\
LTC $n=3$ & $(0.50,25)$ & $(0.50,25)$ & $(0.50,25)$ & $(0.50,25)$ \\
NTC $n=1$ & $(0.35,40)$ & $(0.35,40)$ & $(0.50,40)$ & $(0.50,40)$ \\
NTC $n=2$ & $(0.50,40)$ & $(0.50,40)$ & $(0.50,25)$ & $(0.50,40)$ \\
NTC $n=3$ & $(0.50,40)$ & $(0.50,40)$ & $(0.50,40)$ & $(0.50,40)$ \\
MLP $\eta$ & $0.35$ & $0.35$ & $0.50$ & $0.50$ \\
\bottomrule
\end{tabular}
\end{table}

\begin{table}[h]
\centering
\caption{Selected hyperparameters on Kuramoto, at the Taylor order $n$ shown
in each panel of Figure~\ref{fig:policy-convergence}; $\lambda$ is the LSTD
ridge of LTC.}
\label{tab:kur-hyper}
\small
\setlength{\tabcolsep}{5pt}
\begin{tabular}{lccccc}
\toprule
$(\kappa,\kappa_\pi)$ & $(1,1)$ & $(2,1)$ & $(2,2)$ & $(3,2)$ & $(2,3)$ \\
\midrule
LTC $n$ & $1$ & $2$ & $2$ & $2$ & $2$ \\
LTC $(\eta,K)$ & $(0.25,16)$ & $(0.25,16)$ & $(0.50,4)$ & $(0.50,16)$ & $(0.50,16)$ \\
LTC $\lambda$ & $10^{-3}$ & $10^{-1}$ & $10^{-3}$ & $10^{-1}$ & $10^{-1}$ \\
NTC $n$ & $1$ & $1$ & $2$ & $2$ & $2$ \\
NTC $(\eta,K)$ & $(0.25,16)$ & $(0.25,16)$ & $(0.50,16)$ & $(0.50,16)$ & $(0.50,16)$ \\
MLP $\eta$ & $0.25$ & $0.25$ & $0.25$ & $0.50$ & $0.50$ \\
\bottomrule
\end{tabular}
\end{table}

\section{Further Experimental Results}
\label{app:further-exp}

This appendix provides additional experiments supporting the structural claims of the paper. We first verify empirically the exponential decay of critic derivatives predicted by Lemma~\ref{lem:graded-derivatives}. We then study the main design choices of the Taylor Centroid representation: the number and location of expansion points, the dependence of approximation error on distance to the assigned centroid, and the effect of the Taylor order. We next report representative trajectories on all three environments to complement the aggregate policy metrics. Finally, on hidden switched LQR, we examine whether centroids fitted from rollout data alone recover structure associated with the unobserved operating regimes, and how this structure evolves under the learned policy.

\subsection{Empirical exponential derivative decay}
\label{app:exp-graded-decay}

We empirically examine the central structural prediction of
Lemma~\ref{lem:graded-derivatives}: derivatives of an agent's action-value
function should become smaller when they involve state variables located
farther away in the interaction graph. For a derivative order $p$ and graph
reach $d$, define
\[
H_p(d)
:=
\max_i
\max_{\substack{|\beta|=p\\ d_i(\beta)=d}}
\sup_{z,s,a}
\left|
\partial_s^\beta Q_i^\pi(z,s,a)
\right|.
\]
Lemma~\ref{lem:graded-derivatives} gives
\[
H_p(d)\le L_Q\rho^d,
\qquad
\rho=\gamma^{1/(1+\kappa_\pi)}.
\]
Since the multiplicative constant $L_Q$ is unknown, the experiment is designed
to test the \emph{geometric dependence on graph reach}, rather than the
absolute magnitude of the bound. We therefore plot $H_p(d)/H_p(0)$ and fit an
empirical rate $\widehat\rho_p$ through
\[
\log H_p(d)
=
\alpha_p+d\log\widehat\rho_p.
\]
Normalizing by $H_p(0)$ only fixes the common value at $d=0$ and does not force
the measured curves to satisfy the theoretical rate.

For LQR and switched LQR, we evaluate the converged LTC($n=2$) affine-mean Gaussian policy at
each $\kappa_\pi\in\{0,1,2,3\}$. Under such a policy, the corresponding
action-value functions are quadratic (conditionally on $\sigma$ for switched LQR). In fact, we can compute the real critics, and therefore their derivatives. The first-order supremum is evaluated exactly over the
declared state--action box, second-order derivatives are constant in $(s,a)$,
and third- and higher-order derivatives vanish. For Kuramoto, we evaluate the converged LTC $(n=2)$ policy at $\kappa_\pi=1$; in this experiment, its critic uses $\gamma=0.75$, as on LQR, instead of the benchmark value $0.99$, hence $\rho=\sqrt{0.75}$. Since no closed-form
critic is available, we enumerate all distinct state-derivative multi-indices
up to order three and numerically maximize their magnitudes over states and
actions using multi-start projected ascent. The resulting Kuramoto statistics
are therefore numerical lower bounds on the corresponding suprema.

\begin{figure}[h!]
    \centering
    \includegraphics[width=\linewidth]{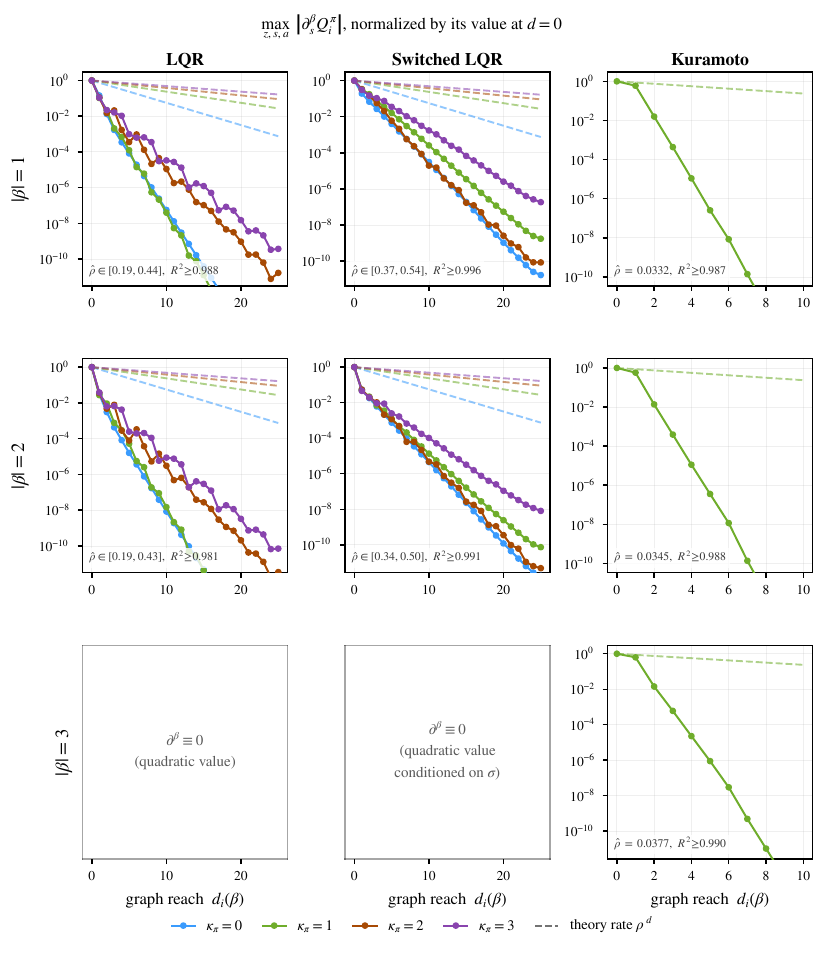}
    \caption{\textbf{Empirical exponential decay of action-value derivatives.}
    For each derivative order $p=|\beta|$, we plot $H_p(d)/H_p(0)$ against
    the graph reach $d_i(\beta)$. Solid curves show measured derivative
    magnitudes and dashed curves show the explicit theoretical rate
    $\rho^d$, with $\rho=\gamma^{1/(1+\kappa_\pi)}$; colors indicate the
    policy locality radius $\kappa_\pi$. LQR and switched LQR derivatives are
    computed exactly from their quadratic and quadratic conditioned on $\sigma$ value functions, so their
    third-order derivatives vanish. 
    For Kuramoto ($N=20$), every derivative up to order three is maximized over states and actions: all multi-indices are enumerated at $256$ visited states, and the $4$ best starts per graph reach are refined by projected signed-gradient ascent (step $0.05$, fixed budget, no tolerance), giving lower bounds on the suprema. Reported
    $\widehat\rho$ values are least-squares fits of $\log H_p(d)$ over reaches
    above the numerical precision threshold.}
    \label{fig:graded-decay}
\end{figure}

Figure~\ref{fig:graded-decay} exhibits the exponential decay structure predicted by
the theory across all three environments. At every nonzero derivative order,
sensitivities involving increasingly distant state variables decrease
approximately geometrically with graph reach. The phenomenon therefore extends
beyond ordinary spatial localization of $Q_i^\pi$: graph distance influences
the individual higher-order sensitivities of the critic, including the
derivatives that form its Taylor coefficients.

 Across the LQR experiments, the fitted rates lie in
$\widehat\rho\in[0.19,0.44]$, while switched LQR gives
$\widehat\rho\in[0.34,0.54]$.

We emphasize that the explicit rate
\[
\rho=\gamma^{1/(1+\kappa_\pi)}
\]
is not intended to be sharp. The measured curves lie far below the dashed ones: the theoretical rate is conservative, and the fitted rates $\widehat\rho$ are much smaller in practice.

The gap is largest on the nonlinear Kuramoto system, where
\[
\widehat\rho_1=0.0332,\qquad
\widehat\rho_2=0.0345,\qquad
\widehat\rho_3=0.0377,
\]
with $R^2\ge0.987$, whereas the explicit theorem gives
$\rho=\sqrt{0.75}\simeq0.866$ for $\kappa_\pi=1$. We observe additionally that the fitted rates are also
nearly unchanged across derivative orders.

\subsection{Centroids and local approximation quality}
\label{app:centroid-ablation}

The finite-sample LTC bound of Theorem~2 contains the estimation error term
\[
\mathcal E_i^{(\kappa)}
=
\bar\varepsilon_{\mathrm{lin}}
\sum_{p=0}^{n}\frac{\mathcal R_p}{p!},
\]
which exposes two competing effects of the centroid construction.
Increasing the number of centroids $K$ makes the Taylor expansion more local,
thereby reducing the centroid radii and thus $\mathcal R_p$. At the same time, LTC uses
a single shared affine coefficient map across all clusters, so increasing $K$
can make the collection of centroid-dependent Taylor coefficients harder to
represent and thus can increase $\bar\varepsilon_{\mathrm{lin}}$. Hence
$\mathcal E_i^{(\kappa)}$ need not decrease monotonically with $K$. 

Figure~\ref{fig:centroid-ablation}a directly illustrates this tradeoff.
For LTC, increasing $K$ initially improves the held-out critic fit, consistent
with smaller local Taylor regions, but performance eventually deteriorates as
 estimation error becomes dominant. 
The turnover occurs earlier at larger Taylor orders: richer local polynomials require fewer expansion
points.

In contrast, NTC does not exhibit the same
pronounced degradation at large $K$, consistent with its more expressive
neural coefficient map being able to accommodate the variation of the Taylor
coefficients across centroids.

Figure~\ref{fig:centroid-ablation}b isolates the complementary geometric
effect underlying the factors $\mathcal R_p$. On LQR, critic error increases with the
distance to the assigned centroid, as expected from a local Taylor
approximation, while increasing the Taylor order typically attenuates
this dependence. Together, the two panels empirically illustrate the two
components entering $\mathcal E_i^{(\kappa)}$: centroids control the size of
the local Taylor regions, whereas the coefficient parameterization controls
how accurately the corresponding local Taylor coefficients can be represented.

\begin{figure}[h!]
    \centering
    \includegraphics[width=1\linewidth]{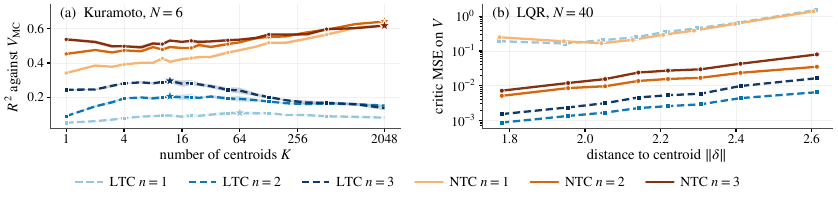}%
    \caption{\textbf{Centroid and local-approximation ablations.}
(a) held-out critic fit versus the number of centroids $K$ on Kuramoto ($N=6$, with encoding $s_i=\theta_i$), measured by $R^2=1-\sum_{s,i}\bigl(\tilde{V_i}(s)-V^{\mathrm{MC}}_i(s)\bigr)^2/\sum_{s,i}\bigl(V^{\mathrm{MC}}_i(s)-\bar V^{\mathrm{MC}}\bigr)^2$ over $4096$ held-out states, where $V^{\mathrm{MC}}_i(s)$ averages the discounted return of agent $i$ over $20$ Monte Carlo continuations of $500$ steps from $s$. (b) Mean squared error of the critic on the exact agent-mean value $\bar V=\frac1N\sum_iV_i$ on LQR ($N=40$), as a function of the distance of the evaluated state to its assigned centroid, for LTC and NTC with $n=1,2,3$.}
\label{fig:centroid-ablation}
\end{figure}

\subsection{Taylor Order Ablation}

In LQR and switched LQR, under our affine-mean Gaussian policies, the critic is a quadratic (more precisely, quadratic conditioned on $\sigma$ for switched LQR) function. As a result, we expect $n=2$ to yield far better results than $n=1$ when training our policy and $n=3$ to provide relatively little improvement. We made this observation in the case of critic fitting in Figure~\ref{fig:critic-calibration}. Figure~\ref{fig:taylor-ablation} shows the same pattern during policy training as soon as $\kappa_\pi\geq 2$: the gap to the optimum is much lower for $n=2$ than for $n=1$, and $n=3$ provides no improvement.

\begin{figure}[h!]
    \centering
    \includegraphics[width=1\linewidth]{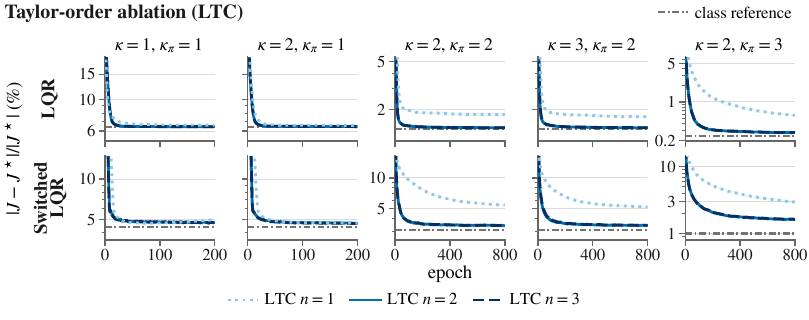}
    \caption{\textbf{Taylor order ablation.} Convergence of LTC for $n \in \{1,2,3\}$ for different configurations $(\kappa,\kappa_\pi)$ on LQR and switched LQR.}
    \label{fig:taylor-ablation}
\end{figure}

\subsection{Episode trajectories on the three environments}
\label{app:trajectories}

Figure~\ref{fig:trajectories} complements the aggregate performance metrics
with representative trajectories. The trajectories are run on one seed simply as a qualitative overview of the policy's behavior in practice. On LQR, all methods stabilize
the network rapidly around the target. On switched LQR, the trajectories exhibit the expected
abrupt deviations at mode changes, yet LTC and NTC repeatedly recover stable behavior
despite the switching mode being unobserved by the controller. Finally, on
Kuramoto, Taylor Centroid rapidly synchronizes the agents' frequencies and
maintains synchronization over the rollout. Hence, the policy obtained at convergence leads to the expected trajectories for these environments.

\begin{figure}[h!]
    \centering
    \includegraphics[width=\linewidth]{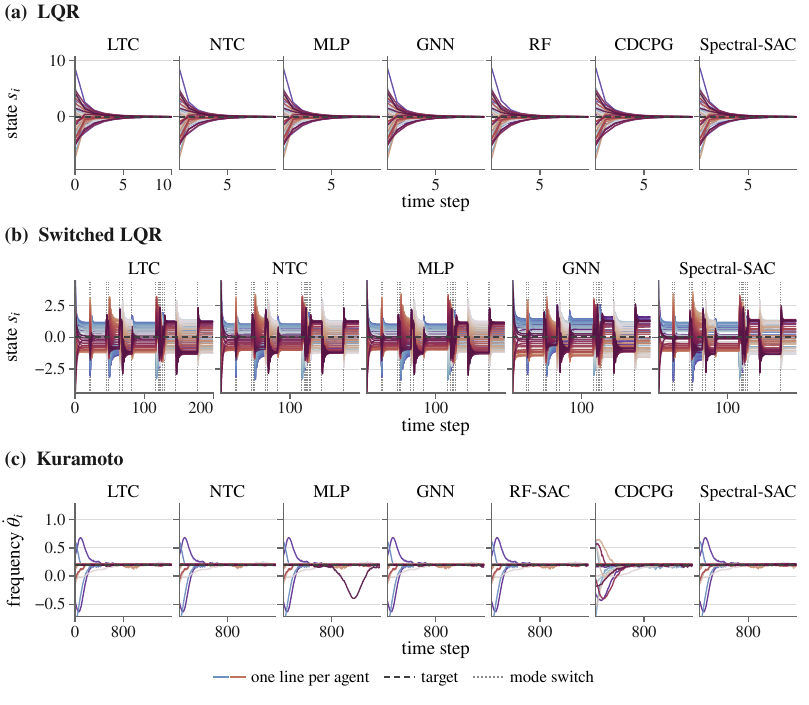}
    \caption{\textbf{Representative trajectories.} Trajectories of the learned policies on networked LQR, hidden-mode switched LQR, and the Kuramoto oscillator.}
    \label{fig:trajectories}
\end{figure}

\subsection{The latent modes appear as data-driven operating regimes.}
\label{app:switched-clusters}

A useful feature of switched LQR is that clusters can often be identified with interpretable regimes. In this environment, one can think of each $\sigma_t$ as an operating regime for the system. Here, the meaning of a ``regime'' is known to the experimenter while remaining hidden from the learning algorithm. This
allows us to check whether the data-driven centroids introduced in our method
recover meaningful structure rather than merely producing an arbitrary
partition of state space.
\begin{figure}[h!]
    \centering
    \includegraphics[width=\linewidth]{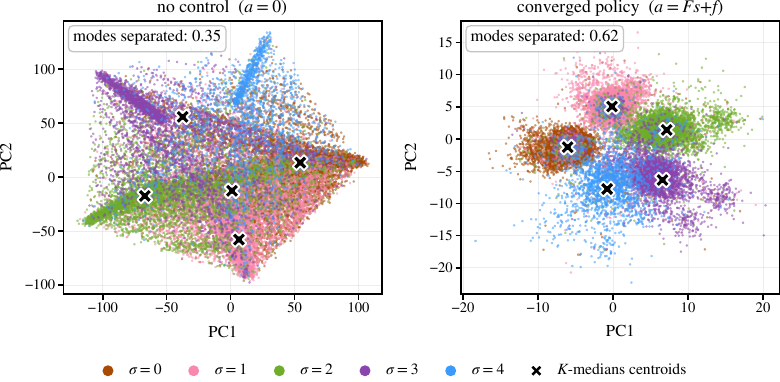}%
    \caption{\textbf{Operating regimes in hidden switched LQR.}
PCA projection of state trajectories, colored by latent mode $\sigma_t$ for visualization; crosses denote the $K=5$ Taylor Centroid medians. Mode separation, the fraction of the variance of the two principal-component scores $z_t$ explained by the latent mode, $\sum_m n_m\|\bar z_m-\bar z\|^2/\sum_t\|z_t-\bar z\|^2$ (with $\bar z_m$ the mean of the $n_m$ points with $\sigma_t=m$ and $\bar z$ the overall mean; $0$ when all modes share the same mean, $1$ when each mode collapses to a point), increases from $0.35$ without control (left) to $0.62$ after training with $\sigma_t$ unobserved (right), as trajectories form clearer recurrent regions and the centroids align with the latent regimes.
    }
    \label{fig:switched-clusters}
\end{figure}

In this regard, Figure~\ref{fig:switched-clusters} provides a direct empirical motivation for
the cluster-based representation. The $K$-medians centroids are fitted solely
from physical states, yet, after control, the regions of state space visited
under different latent modes become clearly separated and the learned
centroids lie inside these recurrent regions. Thus the centroids used by the critic can acquire an interpretable operational meaning even without access
to the variable generating the regimes. Figure~\ref{fig:switched-clusters} also reveals a significant state-distribution shift during training. Crucially, this motivates refitting the centroids throughout training.

\section{Understanding the Expected Reward Regularity Assumption}
\label{app:regularity}

Throughout this appendix, $q\ge2$ denotes the regularity order in
Assumption~\ref{asm:orbit}, while $1\le n\le q-1$ denotes the Taylor order
used in the critic representation. For a sufficiently differentiable
function $f(z,s,a)$, define
\begin{equation}
\|f\|_{C_s^q}
:=
\max_{|\beta|\le q}
\sup_{z,s,a}
\left|
\partial_s^\beta f(z,s,a)
\right|.
\label{eq:app-Cs-norm}
\end{equation}

For convenience, recall the expected rewards at time $t$:
\begin{equation}
g_{i,t}^\pi(z,s,a)
:=
\mathbb E_\pi
\left[
r_i(S_{i,t},A_{i,t})
\mid
Z_0=z,S_0=s,A_0=a
\right].
\label{eq:app-reg-forward}
\end{equation}
Assumption~\ref{asm:orbit} can then be written as
\begin{equation}
\sup_{t\ge0}
\|g_{i,t}^\pi\|_{C_s^q}
\le \bar L
\label{eq:app-orbit-recall}
\end{equation}
uniformly in the network size, agent, and policy in the considered class.

We first give an operational interpretation of
Equation~\ref{eq:app-orbit-recall} and describe how it can be stress-tested from
rollouts. We then show that its uniform-in-time form can be weakened to a
discounted-summability condition. Finally, we provide several sufficient
conditions and position the assumption relative to the smooth-MDP hierarchy
of \citet{maran2024local,maran2024smooth}.

\subsection{Interpretation}
\label{app:regularity-interpretation}

Assumption~\ref{asm:orbit} has a direct physical interpretation. Consider a
small perturbation of the current physical state. The assumption requires that,
as this perturbation propagates through the system dynamics, it does not create
arbitrarily sharp variations in the expected local reward of agent $i$.
Thus, nearby physical configurations may evolve differently, but the expected
reward received after a finite number of steps must remain a sufficiently
smooth function of the current state.

Letting
\[
g_{i,t}^\pi(z,s,a)
=
\mathbb E_\pi\!\left[
r_i(S_{i,t},A_{i,t})
\mid Z_0=z,S_0=s,A_0=a
\right]
\]
The assumption requires for the derivatives of $g_{i,t}^\pi$ in $s$ to remain bounded for all $t$.



\subsection{A rollout-based stress test}
\label{app:regularity-stress}

A practical advantage of Assumption~\ref{asm:orbit} is that it is stated in
terms of a physically interpretable quantity, the future reward predictions $g_{i,t}^\pi$ of Equation~\ref{eq:app-reg-forward}, that can be probed directly from rollouts. This is considerably easier to inspect empirically than smoothness conditions
imposed on an entire Bellman operator, such as Mild smoothness in
\citet{maran2024local}.


\begin{figure}[h!]
    \centering
    \includegraphics[width=1\linewidth]{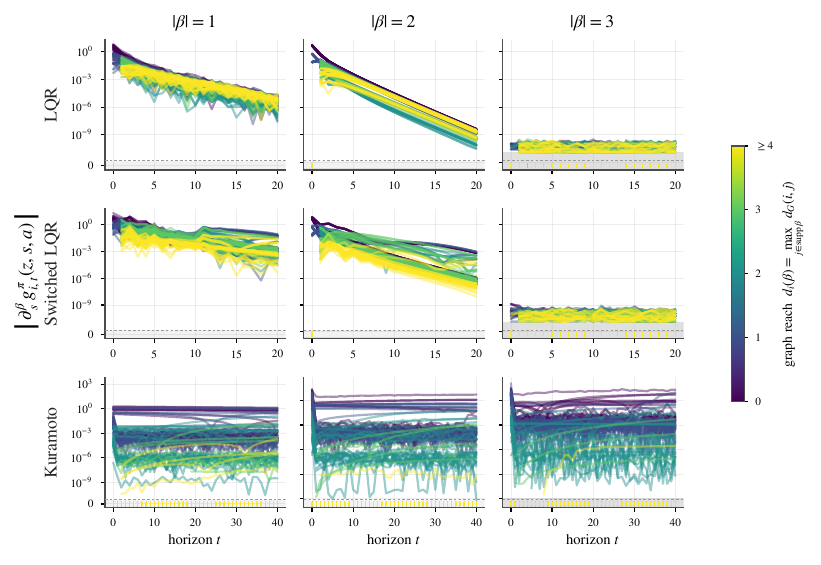}%
    \caption{
\textbf{Empirical stress test of expected reward regularity.}
$|\partial_s^\beta g_{i,t}^{\pi}(z,s,a)|$ versus horizon $t$ for $|\beta|\in\{1,2,3\}$ on LQR and switched LQR ($N=50$, learned policies with $\kappa_\pi=2$) and Kuramoto ($N=40$, greedy controller), at $16$ visited anchors $(z,s,a)$ and $5$ agents per environment, with $256$ rollouts per expectation. Colors: graph reach $d_i(\beta)=\max_{j\in\operatorname{supp}_G(\beta)} d_G(i,j)$; marks in the bottom strip: numerically zero derivatives; gray: finite-difference resolution range.
}
    \label{fig:assumption-stress-test}
\end{figure}

Figure~\ref{fig:assumption-stress-test} evaluates$
\left|
\partial_s^\beta g_{i,t}^{\pi}(z,s,a)
\right|
$ for derivative orders $|\beta|\in\{1,2,3\}$ on LQR, switched LQR, and
Kuramoto. Each curve follows one sampled initial value $(z,s,a)$, reference
agent $i$, and multi-index $\beta$ over the prediction horizon. Finite
differences perturb only the initial physical state $s$, while the initial
context $z$ and action $a$ are held fixed.

Derivatives are estimated by symmetric finite differences around states
visited by the policy. Each state coordinate is perturbed by $h$ times its empirical standard deviation, where $h$ is a dimensionless step size.
We repeat the computation for
$h\in\{0.01,0.02,0.05\}$ and display $h=0.02$ in the figure. The estimates are
generally stable across these choices. Relative differences across $h$ should
be interpreted cautiously when a derivative is itself extremely small, since
tiny absolute changes can then correspond to large relative variations; this
is most visible for the third-order Kuramoto derivatives, which are somewhat
less stable across $h$ than the other panels while remaining controlled over
the tested horizon. 

Curves are colored by the graph reach
\(
d_i(\beta)
:=
\max_{j\in\operatorname{supp}_G(\beta)} d_G(i,j),
\)
from derivatives involving agent $i$ itself to derivatives involving more
distant agents. Colored marks in the lower strip indicate numerically zero
derivatives, while the gray region marks values below the numerical resolution
of the finite-difference estimate and is not interpreted quantitatively.

The results are consistent with the regularity assumption. Across all three
environments, derivative magnitudes remain controlled as the horizon grows,
with no systematic amplification in $t$. On LQR and switched LQR,
first- and second-order sensitivities decrease over the prediction horizon,
while third-order derivatives lie at the numerical resolution floor, as
expected from the quadratic (and quadratic conditioned on $\sigma$) dependence of the expected rewards on the initial
state under the affine-mean policies considered here. Kuramoto, whose dynamics are
nonlinear, retains nonzero derivatives through third order while their
magnitudes remain controlled over the tested horizon. The plots also exhibit
the expected spatial structure: derivatives involving more distant agents are
typically smaller.

Because Assumption~\ref{asm:orbit} is uniform over states, policies, agents,
and prediction horizons, no finite experiment can verify it formally.
Figure~\ref{fig:assumption-stress-test} should instead be viewed as a stress
test: systematic growth of these expected reward sensitivities with the
prediction horizon on highly visited states would provide direct evidence
against the assumption. We observe no such behavior on any of the three
benchmarks; instead, expected reward sensitivities remain controlled over the
tested horizons, together with the distance-dependent attenuation expected
from graph-local dynamics.

\subsection{A weaker discounted-summability condition}
\label{app:summable-regularity}

For simplicity of exposition, Assumption~\ref{asm:orbit} is stated as a uniform-in-$t$ bound. This is stronger than mathematically necessary: it suffices that the smoothness
of the expected reward predictions grows slowly enough to be dominated by
discounting as shown in Assumption~\ref{asm:app-summable}.

\begin{assumption}[Discounted-summable expected reward regularity]
\label{asm:app-summable}
There exists a nonnegative sequence $\bar L(t)$, uniform in the network size,
agent, initial global state, and policy in the considered class, such that
\begin{equation}
\|g_{i,t}^\pi\|_{C_s^q}
\le
\bar L(t),
\qquad t\ge0,
\label{eq:app-summable-bound}
\end{equation}
and
\begin{equation}
\sum_{t=0}^{\infty}\gamma^t\bar L(t)<\infty.
\label{eq:app-summability}
\end{equation}
\end{assumption}

Assumption~\ref{asm:orbit} is recovered by taking
$\bar L(t)\equiv\bar L$. More generally, the condition allows finite-horizon
sensitivities to grow with time, provided that discounting dominates this
growth. We now show that we recover the main results of our main theory under this weaker assumption.

\begin{proposition}[Regularity under a summable envelope]
\label{prop:app-summable}
Under Assumption~\ref{asm:app-summable}, for every multi-index
$|\beta|\le q$,
\begin{equation}
\partial_s^\beta Q_i^\pi(z,s,a)
=
\sum_{t=0}^{\infty}
\gamma^t
\partial_s^\beta g_{i,t}^\pi(z,s,a),
\label{eq:app-summable-diff}
\end{equation}
and
\begin{equation}
\sup_{z,s,a}
\left|
\partial_s^\beta Q_i^\pi(z,s,a)
\right|
\le
\sum_{t=0}^{\infty}
\gamma^t\bar L(t).
\label{eq:app-summable-Q}
\end{equation}
\end{proposition}

The exponential decay extends to this setting. For $d\ge0$, define
\begin{equation}
\tau_d
:=
\left\lceil
\frac{d}{1+\kappa_\pi}
\right\rceil,
\qquad
\Lambda_d
:=
\sum_{t=\tau_d}^{\infty}
\gamma^t\bar L(t).
\label{eq:app-Lambda}
\end{equation}
The sequence $(\Lambda_d)_{d\ge0}$ is nonincreasing and tends to $0$ as $d\to +\infty$.

\begin{proposition}[Exponential decay under a summable envelope]
\label{prop:app-summable-graded}
Under Assumption~\ref{asm:app-summable}, for every nonzero multi-index
$|\beta|\le q$,
\begin{equation}
\left|
\partial_s^\beta Q_i^\pi(z,s,a)
\right|
\le
\Lambda_{d_i(\beta)}.
\label{eq:app-summable-graded}
\end{equation}
\end{proposition}

Thus the geometric envelope used in the main text is only one particular
case. The decay rate is determined by the discounted sum of $\bar L(t)$.


For the uniform envelope $\bar L(t)\equiv\bar L$, this reduces to
\[
\Lambda_d
=
\frac{\bar L}{1-\gamma}\gamma^{\tau_d}
\le
\frac{\bar L}{1-\gamma}\rho^d
=
L_Q\rho^d,
\]
recovering Lemma~\ref{lem:graded-derivatives}.

The Taylor Centroid approximation can be extended explicitly to this
general envelope. For $0\le d\le\kappa$, define the same local centroid radii
as in Theorem~\ref{thm:trunc},
\[
W_d
:=
\sup_s
\left\|
s_{\mathcal N_i^d}
-\mu_{k(s),\mathcal N_i^d}
\right\|_1,
\]
and, for every integer $m\ge1$, define
\begin{equation}
\mathcal R_m^\Lambda
:=
\Lambda_\kappa W_\kappa^m
+
\sum_{d=0}^{\kappa-1}
\bigl(\Lambda_d-\Lambda_{d+1}\bigr)W_d^m.
\label{eq:app-general-remainder}
\end{equation}
Because $(\Lambda_d)_d$ is nonincreasing, all coefficients in
Equation~\ref{eq:app-general-remainder} are nonnegative.

\begin{proposition}[Scalable local Taylor representation under a summable envelope]
\label{prop:app-summable-tc}
Under Assumption~\ref{asm:app-summable}, let $1\le n\le q-1$ and let
$\widehat Q_i^\pi$ be the order-$n$ Taylor Centroid approximation defined in
Equation~\ref{eq:ideal-tc}. Then
\begin{equation}
\sup_{z,s,a}
\left|
Q_i^\pi(z,s,a)-\widehat Q_i^\pi(z,s,a)
\right|
\le
\frac{2R_{\max}}{1-\gamma}\rho^{\kappa+1}
+
\frac{\mathcal R_{n+1}^\Lambda}{(n+1)!},
\qquad
\rho:=\gamma^{1/(1+\kappa_\pi)}.
\label{eq:app-summable-tc}
\end{equation}
Moreover, $\widehat Q_i^\pi$ contains
\[
\binom{|\mathcal N_i^\kappa|d_s+n}{n}
\]
state monomials and therefore remains independent of $N$ for bounded graph
degree and fixed $(\kappa,n,d_s)$.
\end{proposition}

The two terms in Equation~\ref{eq:app-summable-tc} have different origins.
The localization term is unchanged because spatial decay  uses
only bounded rewards, graph locality, and discounting. The Taylor term instead
inherits the derivative envelope $(\Lambda_d)_d$. The latter is obtained by
grouping Taylor monomials according to their graph reach and applying
summation by parts:
\begin{equation}
\sum_{d=0}^{\kappa}
\Lambda_d\bigl(A_{m,d}-A_{m,d-1}\bigr)
=
\Lambda_\kappa A_{m,\kappa}
+
\sum_{d=0}^{\kappa-1}
(\Lambda_d-\Lambda_{d+1})A_{m,d},
\label{eq:app-general-shell-sbp}
\end{equation}
where
\[
A_{m,d}
=
\frac{
\|h_{\mathcal N_i^d}\|_1^m
}{m!},
\qquad A_{m,-1}:=0.
\]
For the uniform envelope $\bar L(t)\equiv\bar L$, the geometric majorant
$\Lambda_d\le L_Q\rho^d$ gives
\[
\mathcal R_m^\Lambda
\le
L_Q
\left[
\rho^\kappa W_\kappa^m
+
(1-\rho)
\sum_{d=0}^{\kappa-1}\rho^dW_d^m
\right]
=
L_Q\mathcal R_m,
\]
so Proposition~\ref{prop:app-summable-tc} recovers the bound of
Theorem~\ref{thm:trunc}.

The remaining results adapt in the same way. Value-localization
terms retain their geometric factor $\rho$, while every use of the exponential decay $L_Q\rho^{d_i(\beta)}$ is replaced by the corresponding term
$\Lambda_{d_i(\beta)}$. The LSTD concentration argument itself is unchanged, and the actor
analysis only uses the fitted critic through its final error certificate
$\Delta_Q$.


\subsection{When does expected reward regularity hold?}
\label{app:regularity-sufficient}
 We now give sufficient conditions under which Assumption~\ref{asm:orbit} holds.

\subsubsection{Smooth stochastic dynamics.}

In many cases, the regularity we require is provided by the transition kernel itself. Indeed, a broad instance of problems satisfying Assumption~\ref{asm:orbit} consists of graph-local systems with smooth stochastic transitions. Consider
\begin{equation}
S_i'
=
f_i\!\left(
Z,S_{\mathcal N_i},A_{\mathcal N_i}
\right)
+
M_i\!\left(
Z,S_{\mathcal N_i},A_{\mathcal N_i}
\right)\eta_i,
\label{eq:app-smooth-stochastic}
\end{equation}
where $\eta_i$ is independent of the current state.

\begin{proposition}[Smooth stochastic dynamics are sufficient]
\label{prop:app-smooth-stochastic}
Suppose first that, uniformly in $i$ and $a_i$,
$r_i(\cdot,a_i)$ is $C^q$ with bounded derivatives through order $q$. Suppose further that, uniformly in $i$ and $z$, the maps
\[
f_i(z,\cdot,\cdot),\qquad
M_i(z,\cdot,\cdot),\qquad
M_i(z,\cdot,\cdot)^{-1}
\]
are $C^q$ in the physical-state variables with bounded derivatives through
order $q$, and that $M_i$ is uniformly invertible. Suppose finally that the
density $p_i$ of $\eta_i$ satisfies
\[
\max_{|\alpha|\le q}
\int_{\mathbb R^{d_s}}
(1+\|u\|)^q
\left|
\partial^\alpha p_i(u)
\right|\,du
<\infty
\]
uniformly in $i$. Then Assumption~\ref{asm:orbit} holds with a constant
independent of $N$ whenever the graph degree is bounded.
\end{proposition}

Here, noise acts as a smoothing mechanism: even if
individual trajectories are sensitive, expected rewards remain
uniformly smooth functions of the current state due to the regularizing effect of the kernel.

\paragraph{Low-noise limit.}
For isotropic noise $M_i=\sigma I$, the constants produced by the above
certificate may deteriorate as $\sigma\to0$: at derivative order $q$, the
proof gives a worst-case scaling of order $O(\sigma^{-q})$ (up to local
smoothness and graph-degree factors). This reflects a loss of smoothing as
the noise vanishes. Importantly, it is the sufficient certificate that
degrades, not necessarily the true derivatives of the critic; for instance,
in linear--quadratic systems the higher-order critic derivatives need not
depend on $\sigma$.
The proposition covers several standard stochastic systems. The additive-noise
model
\[
S_i'
=
f_i\!\left(
Z,S_{\mathcal N_i},A_{\mathcal N_i}
\right)
+\varepsilon_i
\]
is recovered by taking $M_i\equiv I$, and includes graph-local systems with
Gaussian noise as well as many smooth heavy-tailed noise distributions.
Allowing $M_i$ to depend on the current state additionally covers smooth
heteroskedastic and state-dependent uncertainty. The required density
condition is satisfied, for example, by Gaussian and Gaussian-mixture noise,
smooth compactly supported densities, and smooth heavy-tailed distributions
with sufficiently many finite moments.
The main additive-noise settings studied by
\citet{wang2026continuous,ren2024spectral} are therefore special cases of
Proposition~\ref{prop:app-smooth-stochastic}. Thus our regularity class
contains their principal model classes while also allowing state-dependent
stochastic dynamics. In particular, LQR, switched LQR and the Kuramoto oscillator all satisfy Assumption~\ref{asm:orbit} through Proposition~\ref{prop:app-smooth-stochastic}.





\subsubsection{Smooth linear transition models.}

Our assumption also follows from sufficiently regular linear
factorizations of the transition kernel. Suppose that a local transition
density admits
\begin{equation}
p_j
\left(
s_j'
\mid
s_{\mathcal N_j},
a_{\mathcal N_j},
z
\right)
=
\sum_{h=1}^{d}
\varphi_{j,h}
\left(
s_{\mathcal N_j},
a_{\mathcal N_j},
z
\right)
\mu_{j,h}(s_j').
\label{eq:app-spectral-factorization}
\end{equation}

\begin{proposition}[Smooth spectral representations are sufficient]
\label{prop:app-spectral}
Suppose first that, uniformly in $i$ and $a_i$,
$r_i(\cdot,a_i)$ is $C^q$ with bounded derivatives through order $q$.
Suppose further that the current-state feature maps $\varphi_{j,h}$ have
uniformly bounded physical-state derivatives through order $q$ and that the
next-state factors satisfy
\[
\sup_{j,h}\|\mu_{j,h}\|_{L^1}<\infty.
\]
Then Assumption~\ref{asm:orbit} holds with an $N$-independent constant for
bounded graph degree.
\end{proposition}

Thus exact finite-dimensional spectral models of the type used in continuous
networked MDPs provide another one-step smoothing certificate whenever their
current-state features are sufficiently smooth. This is, in particular, the main setting targeted by~\citep{ren2024spectral,wang2026continuous} and prior Linear MDP works~\citep{jin2020provably}.



\subsubsection{Comparison between smoothness notions}
\label{app:maran-regularity}

We finally compare Assumption~\ref{asm:orbit} with the smooth-MDP hierarchy
of \citet{maran2024local,maran2024smooth}. In contrast to our assumption, Mild smoothness is a model-level one-step smoothing certificate: it requires a Bellman
update to smooth every bounded continuation value.



For $x=(z,s)$, let
\[
(\mathcal T^\star f)(x,a)
:=
r(x,a)
+
\int
\sup_{a'}f(x',a')\,P(dx'\mid x,a)
\]
denote the one-step Bellman optimality operator, written without discount to
match the finite-horizon convention of
\citet{maran2024local,maran2024smooth}. Including a discount factor only
changes the corresponding constants.

The original smooth-MDP hierarchy of \citet{maran2024smooth} distinguishes
Strongly and Weakly Smooth MDPs. In what follows, $C^{m,1}$ denotes the class of $m$-times differentiable functions whose derivatives of order $m$ are Lipschitz continuous.

\begin{definition}[Strongly Smooth MDP,~\citet{maran2024smooth}]
For an integer $m\ge0$, an MDP is $m$-Strongly Smooth if the reward and
transition density are themselves smooth in the current state--action pair:
\[
r(\cdot,\cdot)\in C^{m,1},
\qquad
p(x'\mid\cdot,\cdot)\in C^{m,1}
\quad\text{for every }x',
\]
with the corresponding $C^{m,1}$ norms bounded uniformly in $x'$ (and,
in the finite-horizon formulation, uniformly over the time step).
\end{definition}

Thus Strong smoothness imposes regularity directly on the model. In
particular, it requires the transition kernel to admit a sufficiently regular
density.

\begin{definition}[Weakly Smooth MDP, \citet{maran2024smooth}]
For an integer $m\ge0$, an MDP is $m$-Weakly Smooth if the Bellman optimality
operator is bounded from $C^{m,1}$ to itself: there exists
$C_{\mathcal T}<\infty$ such that
\[
\|\mathcal T^\star f\|_{C^{m,1}}
\le
C_{\mathcal T}
\bigl(
\|f\|_{C^{m,1}}+1
\bigr)
\]
for every $f\in C^{m,1}$.
\end{definition}

Weak smoothness therefore does not require the transition density itself to
be smooth. Instead, it requires Bellman updates to preserve smoothness of
functions that are already smooth. This is intuitively similar to our assumption which requires the expected value of the rewards (a smooth function) to remain smooth.

\citet{maran2024local} introduce an intermediate notion, Mild smoothness.
For $\nu>0$, let
\[
\nu_\star:=\lceil \nu-1\rceil.
\]
Their space $\mathcal C^\nu$ consists of functions with $\nu_\star$
continuous derivatives whose derivatives of order $\nu_\star$ are
H\"older continuous with exponent $\nu-\nu_\star$. In particular, at
integer indices,
\[
\mathcal C^\nu
=
C^{\nu-1,1}
\]
up to equivalence of the corresponding norms.

\begin{definition}[Mildly Smooth MDP, \citep{maran2024local}]
An MDP is $\nu$-Mildly Smooth if the Bellman optimality operator is bounded
from $L^\infty$ to $\mathcal C^\nu$: there exists
$C_{\mathcal T}<\infty$ such that
\[
\|\mathcal T^\star f\|_{\mathcal C^\nu}
\le
C_{\mathcal T}
\bigl(
\|f\|_\infty+1
\bigr)
\]
for every bounded measurable $f$.
\end{definition}
The key distinction is that Mild smoothness does not require the input to be
smooth: a Bellman update must produce a smooth function from an arbitrary
bounded input.
At integer orders, Appendix~E.3 (Relation between the settings) of
\citet{maran2024local} proves
\begin{equation}
m\text{-Strong}
\quad\Longrightarrow\quad
(m+1)\text{-Mild}
\quad\Longrightarrow\quad
m\text{-Weak}.
\label{eq:app-maran-hierarchy}
\end{equation}

Mild smoothness provides a particularly simple sufficient condition for our assumption.

\begin{proposition}[Smooth MDPs satisfy expected reward regularity]
\label{prop:app-mild-orbit}
Suppose that the local rewards are $C^q$ with uniformly bounded derivatives.
If the augmented MDP is $\nu$-Mildly Smooth for some
\[
\nu>q,
\]
then Assumption~\ref{asm:orbit} of order $q$ holds.

Consequently, at integer indices,
\begin{equation}
q\text{-Strong}
\quad\Longrightarrow\quad
(q+1)\text{-Mild}
\quad\Longrightarrow\quad
\text{Assumption~\ref{asm:orbit} of order }q.
\label{eq:app-maran-to-orbit}
\end{equation}
\end{proposition}

In particular, \cite{maran2024local,maran2024smooth} show that Mild smoothness captures a wide variety of MDPs which therefore satisfy Assumption~\ref{asm:orbit}. The reason is that every finite-horizon local reward prediction remains
uniformly bounded:
\[
\|g_{i,t}^\pi\|_\infty\le R_{\max}.
\]

The condition $\nu>q$ comes from our choice to formulate
Assumption~\ref{asm:orbit} using $C^q$  following the classical notation.  This one-order shift is therefore a matter of
smoothness convention rather than a substantive difference in the Taylor
approximation.

\paragraph{The implication is strict.}

Mild smoothness is substantially stronger in scope than what our analysis
requires. It controls the Bellman operator on every bounded input function,
whereas Assumption~\ref{asm:orbit} only controls the particular expected reward
sequence.




To understand why that is the case, let us consider an example with a single-agent MDP and a singleton action space,
\[
s=(x,y)\in[-1,1]^2,
\qquad
r(x,y)=x^2,
\]
We consider the following deterministic dynamics
\[
F(x,y)
=
\left(
x,\mathbf 1_{\{x\ge0\}}
\right).
\]
The reward depends only on the first coordinate, which is preserved by the
dynamics. Consequently,
\[
g_t(x,y)=x^2
\qquad
\text{for every }t\ge0,
\]
so the entire expected reward sequence is smooth at every finite order and
Assumption~\ref{asm:orbit} holds.

The Bellman operator, however, also probes parts of the dynamics that never
affect this reward sequence. Consider the smooth bounded function
\[
f(x',y')=y'.
\]
Then
\[
\mathcal T^\star f(x,y)
=
x^2+\mathbf 1_{\{x\ge0\}},
\]
which is discontinuous at $x=0$. Thus the Bellman operator \emph{does not even
preserve continuity} of smooth inputs.
The discontinuity in the second state coordinate is entirely irrelevant to every future reward and therefore to the critic itself. Assumption~\ref{asm:orbit} correctly ignores it. Mild
smoothness cannot make this distinction because it asks the Bellman operator
to regularize arbitrary functions, including functions specifically chosen to
probe that irrelevant coordinate. 

Beyond this informative example, the following proposition shows that this distinction in notions also arises in switched LQR, a well-established control model.

\begin{proposition}[A continuous-context switched LQR separation]
\label{prop:app-continuous-switch}
Consider a single-agent switched LQR on compact state and action domains, with
a continuous global context $Z_t\in[-1,1]$,
\[
Z_{t+1}=-Z_t,
\qquad
\sigma(z):=\mathbf 1_{\{z\ge0\}},
\]
and physical dynamics
\[
S_{t+1}
=
A^{(\sigma(Z_t))}S_t
+
B^{(\sigma(Z_t))}A_t
+
\varepsilon_t.
\]
Assume that $\varepsilon_t$ is centered and has a $C^\infty$ full-dimensional
density satisfying the conditions of
Proposition~\ref{prop:app-smooth-stochastic} at every finite order, and choose
the domains, mode matrices, and noise support so that the physical state
remains in its compact domain. Let
\[
r(s,a)=-(s^\top Qs+a^\top Ra),
\]
and assume that the two regimes are distinct on the domain, i.e.,
there exist $(\bar s,\bar a)$ such that
\[
d:=
\left(A^{(1)}-A^{(0)}\right)\bar s
+
\left(B^{(1)}-B^{(0)}\right)\bar a
\neq0.
\]
Then Assumption~\ref{asm:orbit} holds at every finite order $q$, whereas the augmented MDP on $(Z,S)$ is not Mildly
Smooth at any positive order.
\end{proposition}

Here $Z_t$ is a continuous switching variable: except at $Z_t=0$, its
sign flips at every step, so the physical dynamics alternate between the two
linear regimes selected by $\sigma_t=\mathbf 1_{\{Z_t\ge0\}}$. This example
shows a second way in which expected reward regularity can be strictly more
targeted than operator-wide smoothness. The mode switch is discontinuous as a
function of the continuous context $z$, so Bellman outputs on the full
augmented state need not be smooth in $z$. 

\paragraph{Discrete and hidden operating modes.}

There is a further distinction concerning the global context $Z$ allowed in
our model. The standard Mild-smoothness formulation is stated on a continuous
state--action domain equipped with a $C^\nu$ smoothness structure. Consequently,
when the Markov state is hybrid,
\[
(Z,S),
\]
with $Z$ taking values in a finite set, Mild smoothness is not directly defined
in its standard form without first specifying how the smoothness space should
be extended to the discrete coordinate.

This situation arises naturally in switched LQR
\citep{costa2005markov,ogura2018hidden}.
Assumption~\ref{asm:orbit} requires no such extension:
derivatives are taken only with respect to the continuous physical state $s$,
while the corresponding bounds are required uniformly over the context $z$. 

\section{Proofs of the main-text results}
\label{app:proofs}

This section contains the proofs of the results stated in the main text. We
first establish the finite-propagation and decay properties,
then derive the Taylor Centroid approximation bounds, prove the finite-sample
LTC--LSTD guarantee, and finally analyze the localized actor and the
end-to-end policy-optimization update.

\paragraph{Conditioning convention.}
The algorithm is adaptive across actor iterations. Fix one such iteration.
All critic-side statements are understood conditionally on the sigma-field
generated before the current critic rollouts. Under this conditioning, the
current policy, centroids, cluster assignment, and feature map are fixed, while
the $M_c$ critic rollouts are independent trajectories generated by the current
policy. After fitting the critic, actor-side statements are understood
conditionally on the enlarged sigma-field containing the fitted critics and
any randomness used to construct action-independent baselines. Under this
conditioning, the $M_s$ fresh actor rollouts are independent. Conditional
failure-probability bounds below are uniform in the realized history, so the
corresponding unconditional statements follow from the tower property. We
suppress these conditionings whenever no ambiguity can arise.

For $t\ge0$, let
\begin{equation}
 g_{i,t}^\pi(z,s,a)
 :=
 \mathbb E_\pi\!\left[
 r_i(S_{i,t},A_{i,t})
 \,\middle|\,
 Z_0=z,S_0=s,A_0=a
 \right].
 \label{eq:app-forward-reward}
\end{equation}
Then
\begin{equation}
 Q_i^\pi(z,s,a)
 =
 \sum_{t=0}^\infty \gamma^t g_{i,t}^\pi(z,s,a),
 \label{eq:app-q-series}
\end{equation}
and the bounded-reward assumption gives
\begin{equation}
 |g_{i,t}^\pi(z,s,a)|\le R_{\max},
 \qquad
 |Q_i^\pi(z,s,a)|
 \le
 \frac{R_{\max}}{1-\gamma}.
 \label{eq:app-basic-q-bound}
\end{equation}

\subsection{Finite propagation and exponential decay}
\label{app:locality-proofs}

\paragraph{Proof of Proposition~\ref{prop:smoothness}.}
\begin{proof}
By Equation~\ref{eq:app-basic-q-bound},
\[
 \sup_{z,s,a}
 \left|\gamma^t g_{i,t}^\pi(z,s,a)\right|
 \le
 \gamma^tR_{\max},
\]
and $\sum_{t\ge0}\gamma^tR_{\max}<\infty$. Hence
Equation~\ref{eq:app-q-series} converges uniformly.

Fix a multi-index $\beta$ with $|\beta|\le q$. Assumption~\ref{asm:orbit}
gives
\[
 \sup_{z,s,a}
 \left|\partial_s^\beta g_{i,t}^\pi(z,s,a)\right|
 \le
 \bar L
 \qquad
 \text{for every }t\ge0.
\]
Therefore
\[
 \sum_{t=0}^\infty
 \sup_{z,s,a}
 \left|
 \gamma^t\partial_s^\beta g_{i,t}^\pi(z,s,a)
 \right|
 \le
 \frac{\bar L}{1-\gamma}.
\]
Thus the derivative series converges uniformly for every derivative order up
to $q$. Applying the standard theorem on termwise differentiation of uniformly
convergent series, successively in each state coordinate, shows that
$s\mapsto Q_i^\pi(z,s,a)$ is of class $C^q$ and yields
\begin{equation}
 \partial_s^\beta Q_i^\pi(z,s,a)
 =
 \sum_{t=0}^\infty
 \gamma^t\partial_s^\beta g_{i,t}^\pi(z,s,a),
 \qquad |\beta|\le q.
 \label{eq:app-differentiated-series}
\end{equation}
The derivatives are continuous because, by
Assumption~\ref{asm:orbit}, each $g_{i,t}^\pi$ is of class $C^q$,
and the derivative series converges uniformly. Moreover,
\begin{equation}
 \sup_{z,s,a}
 \left|\partial_s^\beta Q_i^\pi(z,s,a)\right|
 \le
 \frac{\bar L}{1-\gamma}
 =L_Q,
 \qquad |\beta|\le q.
 \label{eq:app-ungraded-derivative-bound}
\end{equation}
This proves Proposition~\ref{prop:smoothness}.
\end{proof}

\begin{lemma}[Finite propagation]
\label{lem:app-light-cone}
Let
\[
 R_t:=t(1+\kappa_\pi).
\]
Conditionally on the global trajectory $(Z_0,\ldots,Z_t)$, the conditional law
of $(S_{i,t},A_{i,t})$ given $(S_0,A_0)=(s,a)$ depends on the initial
distributed state--action pair only through its restriction to
$\mathcal N_i^{R_t}$.
\end{lemma}

\begin{proof}
We argue by induction on $t$, simultaneously for all agents. At $t=0$,
$(S_{i,0},A_{i,0})=(s_i,a_i)$, so the statement holds with $R_0=0$.

Assume that it holds at time $t-1$. Conditional on
$(Z_{t-1},S_{t-1},A_{t-1})$, the factorized transition kernel implies that the
law of $S_{i,t}$ depends only on
\[
 \bigl(S_{j,t-1},A_{j,t-1}\bigr)_{j\in\mathcal N_i}.
\]
For $j\in\mathcal N_i$, the induction hypothesis restricts the dependence of
these variables on the initial distributed coordinates to
$\mathcal N_j^{R_{t-1}}$, and
\[
 \mathcal N_j^{R_{t-1}}
 \subseteq
 \mathcal N_i^{R_{t-1}+1}.
\]
Hence the conditional law of $S_{i,t}$ depends on the initial distributed
variables only through $\mathcal N_i^{R_{t-1}+1}$.

Next,
\[
 A_{i,t}
 \sim
 \pi_{\theta_i}\!\left(
 \cdot\mid
 S_{\mathcal N_i^{\kappa_\pi}}(t),h_i(Z_t)
 \right).
\]
After conditioning on the global trajectory, $h_i(Z_t)$ is fixed. For every
$j\in\mathcal N_i^{\kappa_\pi}$, the preceding state argument shows that
$S_{j,t}$ depends on the initial distributed coordinates only through
$\mathcal N_j^{R_{t-1}+1}$, and
\[
 \mathcal N_j^{R_{t-1}+1}
 \subseteq
 \mathcal N_i^{R_{t-1}+1+\kappa_\pi}.
\]
Thus the conditional law of $(S_{i,t},A_{i,t})$ depends on the initial
distributed state--action pair only through
$\mathcal N_i^{R_{t-1}+1+\kappa_\pi}$. Since
\[
 R_{t-1}+1+\kappa_\pi
 =
 (t-1)(1+\kappa_\pi)+1+\kappa_\pi
 =
 t(1+\kappa_\pi)
 =
 R_t,
\]
the induction closes.
\end{proof}

\begin{lemma}[Spatial decay]
\label{lem:app-spatialdecay}
If $(s,a)$ and $(s',a')$ agree on all state--action coordinates indexed by
$\mathcal N_i^\kappa$, then
\[
 \left|
 Q_i^\pi(z,s,a)-Q_i^\pi(z,s',a')
 \right|
 \le
 \frac{2R_{\max}}{1-\gamma}\rho^{\kappa+1}.
\]
\end{lemma}

\begin{proof}
Set
\[
 \tau_\kappa
 :=
 \left\lfloor
 \frac{\kappa}{1+\kappa_\pi}
 \right\rfloor.
\]
Since $P_g(z'\mid z)$ does not depend on the distributed state--action variables, the global process $(Z_t)_{t\ge 0}$ has the same law under the two initializations. For every
$t\le\tau_\kappa$,
\[
 R_t=t(1+\kappa_\pi)\le\kappa.
\]
By Lemma~\ref{lem:app-light-cone}, conditionally on any realization of the
global trajectory, $(S_{i,t},A_{i,t})$ therefore has the same law under the two
initializations. Hence:
\[
 g_{i,t}^\pi(z,s,a)
 =
 g_{i,t}^\pi(z,s',a')
 \qquad
 (0\le t\le\tau_\kappa).
\]
Put simply, at time $t$, information propagating from outside $\mathcal{N}_i^\kappa$ has not yet reached agent $i$, so only $\mathcal{N}_i^\kappa$ is relevant to $g_{i,t}$. Using Equation~\ref{eq:app-q-series} and Equation~\ref{eq:app-basic-q-bound},
\begin{align*}
 \left|
 Q_i^\pi(z,s,a)-Q_i^\pi(z,s',a')
 \right|
 &\le
 \sum_{t=\tau_\kappa+1}^\infty
 \gamma^t
 \left|
 g_{i,t}^\pi(z,s,a)-g_{i,t}^\pi(z,s',a')
 \right|\\
 &\le
 2R_{\max}
 \sum_{t=\tau_\kappa+1}^\infty\gamma^t\\
 &=
 \frac{2R_{\max}}{1-\gamma}
 \gamma^{\tau_\kappa+1}.
\end{align*}
Moreover,
\[
 \tau_\kappa+1
 =
 \left\lfloor
 \frac{\kappa}{1+\kappa_\pi}
 \right\rfloor+1
 \ge
 \frac{\kappa+1}{1+\kappa_\pi}.
\]
Since $\gamma\in(0,1)$,
\[
 \gamma^{\tau_\kappa+1}
 \le
 \gamma^{(\kappa+1)/(1+\kappa_\pi)}
 =
 \rho^{\kappa+1},
\]
which proves the claim.
\end{proof}

\paragraph{Proof of Lemma~\ref{lem:graded-derivatives}.}
\begin{proof}
 Let $\beta\neq0$ with
$|\beta|\le n+1$, and set
\[
 d:=d_i(\beta),
 \qquad
 \tau_\beta
 :=
 \left\lceil
 \frac{d}{1+\kappa_\pi}
 \right\rceil.
\]
If $t<\tau_\beta$, then $R_t=t(1+\kappa_\pi)<d$. Conditional on the global
trajectory $(Z_t)_{t \geq 0}$, Lemma~\ref{lem:app-light-cone} implies that
$g_{i,t}^\pi$ depends on the initial distributed variables only through
$\mathcal N_i^{R_t}$. By the definition of $d_i(\beta)$, the support of
$\beta$ contains a state coordinate belonging to an agent at distance $d$ from
$i$. Since this coordinate lies outside $\mathcal N_i^{R_t}$,
$g_{i,t}^\pi$ is constant in that coordinate, and therefore
\[
 \partial_s^\beta g_{i,t}^\pi(z,s,a)=0
 \qquad
 (t<\tau_\beta).
\]
The law of the global trajectory $\left(Z_t\right)_{t\geq0}$ does not depend on the initial distributed
state, so averaging over it does not reintroduce such dependence.

Using Equation~\ref{eq:app-differentiated-series},
\begin{align*}
 \left|
 \partial_s^\beta Q_i^\pi(z,s,a)
 \right|
 &=
 \left|
 \sum_{t=\tau_\beta}^\infty
 \gamma^t
 \partial_s^\beta g_{i,t}^\pi(z,s,a)
 \right|\\
 &\le
 \bar L
 \sum_{t=\tau_\beta}^\infty\gamma^t\\
 &=
 L_Q\gamma^{\tau_\beta}.
\end{align*}
Since
$\tau_\beta\ge d_i(\beta)/(1+\kappa_\pi)$ and $\gamma\in(0,1)$,
\[
 \gamma^{\tau_\beta}
 \le
 \gamma^{d_i(\beta)/(1+\kappa_\pi)}
 =
 \rho^{d_i(\beta)}.
\]
This proves the claimed bound for $\beta\neq0$. For $\beta=0$,
$d_i(0)=0$ and Equation~\ref{eq:app-ungraded-derivative-bound} gives
$|Q_i^\pi|\le L_Q=L_Q\rho^0$.
\end{proof}

\subsection{Localized Taylor Centroid approximation}
\label{app:tc-approximation}

\begin{lemma}[Localization error]
\label{lem:app-localization-error}
For every $(z,s,a)$,
\begin{equation}
 \left|
 Q_i^\pi(z,s,a)
 -\bar Q_i^\pi
 (z,s_{\mathcal N_i^\kappa},a_{\mathcal N_i^\kappa})
 \right|
 \le
 \frac{2R_{\max}}{1-\gamma}\rho^{\kappa+1}.
 \label{eq:app-localization-error}
\end{equation}
\end{lemma}

\begin{proof}
Write
\[
 x:=(s_{\mathcal N_i^\kappa},a_{\mathcal N_i^\kappa}),
 \qquad
 y:=(s_{\mathcal N_{-i}^{\kappa}},a_{\mathcal N_{-i}^{\kappa}}).
\]
By definition of $\bar Q_i^\pi=\mathscr D_{i,0}$,
\[
 \bar Q_i^\pi(z,x)
 =
 \int Q_i^\pi(z,x,\widetilde y)\,
 \nu_{i,\mathrm{far}}(d\widetilde y).
\]
Hence
\[
 Q_i^\pi(z,x,y)-\bar Q_i^\pi(z,x)
 =
 \int
 \left[
 Q_i^\pi(z,x,y)-Q_i^\pi(z,x,\widetilde y)
 \right]
 \nu_{i,\mathrm{far}}(d\widetilde y).
\]
For every $\widetilde y$, the two arguments inside the brackets agree on all
state--action coordinates in $\mathcal N_i^\kappa$. Lemma~\ref{lem:app-spatialdecay}
therefore bounds the absolute value of the integrand by
$2R_{\max}(1-\gamma)^{-1}\rho^{\kappa+1}$. Since
$\nu_{i,\mathrm{far}}$ is a probability measure, the same bound holds after
integration.
\end{proof}

This is the same idea as the one used by~\citep{qu2020localized,ren2024spectral,wang2026continuous}.

\begin{lemma}[Localized derivatives]
\label{lem:app-localized-derivatives}
For every multi-index $\beta$ supported in $\mathcal N_i^\kappa$ with
$|\beta|\le n+1$,
\begin{equation}
 \partial_s^\beta \bar Q_i^\pi
 (z,s_{\mathcal N_i^\kappa},a_{\mathcal N_i^\kappa})
 =
 \mathscr D_{i,\beta}
 (z,s_{\mathcal N_i^\kappa},a_{\mathcal N_i^\kappa}),
 \label{eq:app-localized-derivative}
\end{equation}
and
\begin{equation}
 \left|
 \mathscr D_{i,\beta}
 (z,s_{\mathcal N_i^\kappa},a_{\mathcal N_i^\kappa})
 \right|
 \le
 L_Q\rho^{d_i(\beta)}.
 \label{eq:app-G-bound}
\end{equation}
\end{lemma}

\begin{proof}
For every such $\beta$, Lemma~\ref{lem:graded-derivatives} gives
\[
 \left|
 \partial_s^\beta Q_i^\pi(z,s,a)
 \right|
 \le
 L_Q\rho^{d_i(\beta)}
 \le
 L_Q
\]
uniformly in all far coordinates. This uniform bound is integrable with
respect to the probability measure $\nu_{i,\mathrm{far}}$. Repeated
application of dominated convergence therefore permits differentiation under
the integral defining $\bar Q_i^\pi$ and yields
Equation~\ref{eq:app-localized-derivative}. The same bound gives
\[
 \left|\mathscr D_{i,\beta}\right|
 \le
 \int
 \left|\partial_s^\beta Q_i^\pi\right|
 d\nu_{i,\mathrm{far}}
 \le
 L_Q\rho^{d_i(\beta)},
\]
which proves Equation~\ref{eq:app-G-bound}.
\end{proof}

\begin{lemma}[Graph-shell identity]
\label{lem:app-shell}
Let
\[
 h:=s_{\mathcal N_i^\kappa}
 -\mu_{k(s),\mathcal N_i^\kappa}.
\]
For every integer $p\ge0$ and $0\le d\le\kappa$, define
\[
 A_{p,d}
 :=
 \sum_{\substack{|\beta|=p\\ d_i(\beta)\le d}}
 \frac{|h^\beta|}{\beta!},
 \qquad
 A_{p,-1}:=0.
\]
Then
\begin{equation}
 A_{p,d}
 =
 \frac{\|h_{\mathcal N_i^d}\|_1^p}{p!},
 \label{eq:app-shell-cumulative}
\end{equation}
and
\begin{equation}
 \sum_{d=0}^\kappa
 \rho^d(A_{p,d}-A_{p,d-1})
 =
 \rho^\kappa A_{p,\kappa}
 +(1-\rho)
 \sum_{d=0}^{\kappa-1}\rho^dA_{p,d}.
 \label{eq:app-shell-sbp}
\end{equation}
Consequently,
\begin{equation}
 \sum_{\substack{|\beta|=p\\
 \operatorname{supp}_G(\beta)\subseteq\mathcal N_i^\kappa}}
 \rho^{d_i(\beta)}
 \frac{|h^\beta|}{\beta!}
 \le
 \frac{\mathcal R_p}{p!}.
 \label{eq:app-shell-R}
\end{equation}
\end{lemma}

\begin{proof}
The condition $d_i(\beta)\le d$ is equivalent to
$\operatorname{supp}_G(\beta)\subseteq\mathcal N_i^d$. The multinomial
theorem therefore gives
\begin{align*}
 A_{p,d}
 &=
 \sum_{\substack{|\beta|=p\\
 \operatorname{supp}_G(\beta)\subseteq\mathcal N_i^d}}
 \frac{|h^\beta|}{\beta!}\\
 &=
 \frac1{p!}
 \left(
 \sum_{j\in\mathcal N_i^d}
 \sum_{c=1}^{d_s}|h_{j,c}|
 \right)^p
 =
 \frac{\|h_{\mathcal N_i^d}\|_1^p}{p!},
\end{align*}
which proves Equation~\ref{eq:app-shell-cumulative}.

The difference $A_{p,d}-A_{p,d-1}$ is precisely the contribution of
monomials of graph reach $d$. Hence
\begin{align*}
 \sum_{d=0}^\kappa
 \rho^d(A_{p,d}-A_{p,d-1})
 &=
 \sum_{d=0}^\kappa\rho^dA_{p,d}
 -
 \sum_{d=0}^{\kappa-1}\rho^{d+1}A_{p,d}\\
 &=
 \rho^\kappa A_{p,\kappa}
 +(1-\rho)
 \sum_{d=0}^{\kappa-1}\rho^dA_{p,d},
\end{align*}
which is Equation~\ref{eq:app-shell-sbp}. Finally, if
$s\in C_{k(s)}$ (where $C_1,\dots, C_K$ are the fitted clusters), the definition of $W_d$ gives
$\|h_{\mathcal N_i^d}\|_1\le W_d$. Substituting
Equation~\ref{eq:app-shell-cumulative} into Equation~\ref{eq:app-shell-sbp} and applying
these bounds yields
\[
 \sum_{\substack{|\beta|=p\\
 \operatorname{supp}_G(\beta)\subseteq\mathcal N_i^\kappa}}
 \rho^{d_i(\beta)}
 \frac{|h^\beta|}{\beta!}
 \le
 \frac1{p!}
 \left[
 \rho^\kappa W_\kappa^p
 +(1-\rho)
 \sum_{d=0}^{\kappa-1}\rho^dW_d^p
 \right]
 =
 \frac{\mathcal R_p}{p!}.
\]
\end{proof}

\begin{proposition}[Approximation error of the ideal Taylor Centroid critic]
\label{prop:app-tc-error}
Under the assumptions of Lemma~\ref{lem:graded-derivatives},
\begin{equation}
 \sup_{z,s,a}
 \left|
 Q_i^\pi(z,s,a)-\widehat Q_i^\pi(z,s,a)
 \right|
 \le
 \Delta_{\rm TC}
 =
 \frac{2R_{\max}}{1-\gamma}\rho^{\kappa+1}
 +\frac{L_Q}{(n+1)!}\mathcal R_{n+1}.
 \label{eq:app-delta-tc}
\end{equation}
\end{proposition}

\begin{proof}
Decompose
\[
 Q_i^\pi-\widehat Q_i^\pi
 =
 \bigl(Q_i^\pi-\bar Q_i^\pi\bigr)
 +
 \bigl(\bar Q_i^\pi-\widehat Q_i^\pi\bigr).
\]
The first term is bounded by Lemma~\ref{lem:app-localization-error}.

For the second term, fix $(z,s,a)$ and set
\[
 \mu:=\mu_{k(s),\mathcal N_i^\kappa},
 \qquad
 h:=s_{\mathcal N_i^\kappa}-\mu,
 \qquad
 p:=n+1.
\]
By Lemma~\ref{lem:app-localized-derivatives},
Equation~\ref{eq:ideal-tc} is exactly the order-$n$ Taylor polynomial of
$\bar Q_i^\pi$ in the local state variables at $\mu$. Since the state spaces
are convex, the line segment $\mu+uh$, $u\in[0,1]$, remains in the state
domain. The multivariate Taylor formula with integral remainder therefore
gives
\begin{align*}
 &\bar Q_i^\pi
 (z,s_{\mathcal N_i^\kappa},a_{\mathcal N_i^\kappa})
 -
 \widehat Q_i^\pi(z,s,a)\\
 &\quad=
 p
 \sum_{\substack{|\beta|=p\\
 \operatorname{supp}_G(\beta)\subseteq\mathcal N_i^\kappa}}
 \frac{h^\beta}{\beta!}
 \int_0^1
 (1-u)^{p-1}
 \partial_s^\beta\bar Q_i^\pi
 (z,\mu+uh,a_{\mathcal N_i^\kappa})\,du.
\end{align*}
Using Equation~\ref{eq:app-G-bound} and
$p\int_0^1(1-u)^{p-1}\,du=1$,
\begin{align*}
 \left|\bar Q_i^\pi-\widehat Q_i^\pi\right|
 &\le
 L_Q
 \sum_{\substack{|\beta|=p\\
 \operatorname{supp}_G(\beta)\subseteq\mathcal N_i^\kappa}}
 \rho^{d_i(\beta)}
 \frac{|h^\beta|}{\beta!}\\
 &\le
 \frac{L_Q}{(n+1)!}\mathcal R_{n+1},
\end{align*}
where the last inequality is Lemma~\ref{lem:app-shell}. Adding the
localization error proves Equation~\ref{eq:app-delta-tc}.
\end{proof}

\begin{corollary}[Distribution-weighted Taylor Centroid error]
\label{cor:distributional-tc}
Let $\nu$ be any distribution over states and put $p=n+1$. Define
\[
\mathcal R_p(s)
:=
\rho^\kappa
\left\|
s_{\mathcal N_i^\kappa}
-\mu_{k(s),\mathcal N_i^\kappa}
\right\|_1^p
+
(1-\rho)
\sum_{d=0}^{\kappa-1}
\rho^d
\left\|
s_{\mathcal N_i^d}
-\mu_{k(s),\mathcal N_i^d}
\right\|_1^p .
\]
Then, for every $z\in\mathcal Z$ and $a\in\mathcal A$,
\[
\mathbb E_{S\sim\nu}
\left[
\left|
Q_i^\pi(z,S,a)-\widehat Q_i^\pi(z,S,a)
\right|
\right]
\le
\frac{2R_{\max}}{1-\gamma}\rho^{\kappa+1}
+
\frac{L_Q}{(n+1)!}
\mathbb E_{S\sim\nu}
\!\left[
\mathcal R_{n+1}(S)
\right].
\]
Moreover, if, for every $0\le d\le\kappa$,
\[
\left\|
S_{\mathcal N_i^d}
-\mu_{k(S),\mathcal N_i^d}
\right\|_1
\le D_d
\qquad\text{almost surely},
\]
then
\[
\begin{aligned}
\mathbb E[\mathcal R_p(S)]
\le\;&
\rho^\kappa D_\kappa^{p-1}
\mathbb E
\left\|
S_{\mathcal N_i^\kappa}
-\mu_{k(S),\mathcal N_i^\kappa}
\right\|_1
\\
&+
(1-\rho)
\sum_{d=0}^{\kappa-1}
\rho^d D_d^{p-1}
\mathbb E
\left\|
S_{\mathcal N_i^d}
-\mu_{k(S),\mathcal N_i^d}
\right\|_1 .
\end{aligned}
\]
\end{corollary}

\begin{proof}
Repeat the proof of Theorem~\ref{thm:trunc} without replacing
the state-dependent distances to the assigned centroid by their uniform
bounds $W_d$, and then take expectations. The second inequality follows from
$x^p\le D_d^{p-1}x$ for $0\le x\le D_d$.
\end{proof}

The K-medians objective used by Taylor Centroid therefore reduces
a valid upper bound on the average Taylor remainder under the policy-induced
state distribution: it aims to minimize the Taylor representation error \emph{on average}.

\begin{proposition}[Approximation by the affine coefficient map]
\label{prop:app-lin-error}
Let $\mathcal E_i^{(\kappa)}$ be as defined in
Theorem~\ref{thm:ltc-finite}. For every $\zeta>0$, there exists an LTC parameter
$w_i^{(\zeta)}\in\mathbb R^{p_i}$ such that
\begin{equation}
 \sup_{z,s,a}
 \left|
 \widehat Q_i^\pi(z,s,a)
 -
 \widetilde Q_i^{\mathrm{LTC}}
 (h_i(z),s,a_{\mathcal N_i^\kappa};w_i^{(\zeta)})
 \right|
 \le
 \mathcal E_i^{(\kappa)}+\zeta.
 \label{eq:app-lin-zeta}
\end{equation}
In particular, $\bar\varepsilon_{\rm lin}\le L_Q$.
\end{proposition}

\begin{proof}
For
$\beta\in\mathcal M_{\mathcal N_i^\kappa}^n$, let
\[
 \varepsilon_{\rm lin}(\beta)
 :=
 \inf_{v\in\mathbb R^{d_\psi^i}}
 \sup_{z,k,a}
 \left|
 \mathscr D_{i,\beta}
 (z,\mu_{k,\mathcal N_i^\kappa},a_{\mathcal N_i^\kappa})
 -
 v^\top
 \psi_i(k,h_i(z),a_{\mathcal N_i^\kappa})
 \right|.
\]
By Lemma~\ref{lem:app-localized-derivatives},
\[
 \sup_{z,k,a}
 \left|
 \mathscr D_{i,\beta}
 (z,\mu_{k,\mathcal N_i^\kappa},a_{\mathcal N_i^\kappa})
 \right|
 \le
 L_Q\rho^{d_i(\beta)}.
\]
The zero vector is an admissible affine coefficient, hence
\[
 \varepsilon_{\rm lin}(\beta)
 \le
 L_Q\rho^{d_i(\beta)}.
\]
Multiplying by $\rho^{-d_i(\beta)}$ and taking the maximum over $\beta$
proves $\bar\varepsilon_{\rm lin}\le L_Q$.

Now fix $\zeta>0$. Because
$\mathcal M_{\mathcal N_i^\kappa}^n$ is finite, for every $\varepsilon>0$
and every $\beta$ one can choose
$v_{i,\beta}^{(\varepsilon)}$ such that
\begin{align*}
 &\sup_{z,k,a}
 \left|
 \mathscr D_{i,\beta}
 (z,\mu_{k,\mathcal N_i^\kappa},a_{\mathcal N_i^\kappa})
 -
 (v_{i,\beta}^{(\varepsilon)})^\top
 \psi_i(k,h_i(z),a_{\mathcal N_i^\kappa})
 \right|\\
 &\qquad\le
 \varepsilon_{\rm lin}(\beta)
 +
 \varepsilon\rho^{d_i(\beta)}
 \le
 (\bar\varepsilon_{\rm lin}+\varepsilon)
 \rho^{d_i(\beta)}.
\end{align*}
Stack these coefficient vectors into
$w_i^{(\varepsilon)}\in\mathbb R^{p_i}$. With
$h=s_{\mathcal N_i^\kappa}
-\mu_{k(s),\mathcal N_i^\kappa}$, the definitions of LTC and
$\widehat Q_i^\pi$ imply
\begin{align*}
 &\left|
 \widehat Q_i^\pi(z,s,a)
 -
 \widetilde Q_i^{\mathrm{LTC}}
 (h_i(z),s,a_{\mathcal N_i^\kappa};w_i^{(\varepsilon)})
 \right|\\
 &\qquad\le
 (\bar\varepsilon_{\rm lin}+\varepsilon)
 \sum_{\beta\in\mathcal M_{\mathcal N_i^\kappa}^n}
 \rho^{d_i(\beta)}
 \frac{|h^\beta|}{\beta!}\\
 &\qquad\le
 (\bar\varepsilon_{\rm lin}+\varepsilon)
 \sum_{p=0}^n\frac{\mathcal R_p}{p!},
\end{align*}
where the second inequality follows by applying
Lemma~\ref{lem:app-shell} degree by degree. Since
\[
 \mathcal E_i^{(\kappa)}
 =
 \bar\varepsilon_{\rm lin}
 \sum_{p=0}^n\frac{\mathcal R_p}{p!},
\]
choosing $\varepsilon>0$ so that
\[
 \varepsilon
 \sum_{p=0}^n\frac{\mathcal R_p}{p!}
 \le
 \zeta
\]
gives Equation~\ref{eq:app-lin-zeta}.
\end{proof}

\begin{lemma}[Feature and critic envelopes]
\label{lem:app-feature-envelope}
Let $\mathcal G_i^{(\kappa)}$ and $\mathcal V_i^{(\kappa)}$ be as defined in
Theorem~\ref{thm:ltc-finite}. For every state $s$,
\begin{equation}
 \|T_i(s)\|_2^2
 \le
 \sum_{p=0}^n\frac{W_\kappa^{2p}}{p!},
 \qquad
 \|\Phi_i\|^2
 \le
 \mathcal G_i^{(\kappa)}.
 \label{eq:app-feature-bound}
\end{equation}
Moreover,
\begin{equation}
 \sup_{z,s,a}
 |\widehat Q_i^\pi(z,s,a)|
 \le
 \mathcal V_i^{(\kappa)},
 \label{eq:app-ideal-critic-envelope}
\end{equation}
and the approximant in Proposition~\ref{prop:app-lin-error} may be chosen so
that
\begin{equation}
 \left\|
 \widetilde Q_i^{\mathrm{LTC}}
 (\cdot;w_i^{(\zeta)})
 \right\|_\infty
 \le
 \mathcal V_i^{(\kappa)}+\mathcal E_i^{(\kappa)}+\zeta.
 \label{eq:app-approx-critic-envelope}
\end{equation}
\end{lemma}

\begin{proof}
Let
$h=s_{\mathcal N_i^\kappa}
-\mu_{k(s),\mathcal N_i^\kappa}$. For every $p\le n$,
\begin{align*}
 \sum_{|\beta|=p}
 \frac{|h^\beta|^2}{(\beta!)^2}
 &\le
 \sum_{|\beta|=p}
 \frac{(|h|^2)^\beta}{\beta!}\\
 &=
 \frac{\|h\|_2^{2p}}{p!}
 \le
 \frac{\|h\|_1^{2p}}{p!}
 \le
 \frac{W_\kappa^{2p}}{p!}.
\end{align*}
Summing over $p=0,\ldots,n$ proves the first inequality in
Equation~\ref{eq:app-feature-bound}. Since the Euclidean norm is multiplicative
under Kronecker products,
\[
 \|\Phi_i\|^2
 =
 \|\psi_i\|^2\,\|T_i\|_2^2
 \le
 R_\psi^2
 \sum_{p=0}^n
 \frac{W_\kappa^{2p}}{p!}
 =
 \mathcal G_i^{(\kappa)}.
\]

For the ideal critic, the degree-zero coefficient satisfies
$|\bar Q_i^\pi|\le R_{\max}/(1-\gamma)$ by
Equation~\ref{eq:app-basic-q-bound}. For the positive-degree terms,
Lemma~\ref{lem:app-localized-derivatives} and
Lemma~\ref{lem:app-shell} give
\begin{align*}
 |\widehat Q_i^\pi|
 &\le
 \frac{R_{\max}}{1-\gamma}
 +
 L_Q
 \sum_{p=1}^n
 \sum_{\substack{|\beta|=p\\
 \operatorname{supp}_G(\beta)\subseteq\mathcal N_i^\kappa}}
 \rho^{d_i(\beta)}
 \frac{|h^\beta|}{\beta!}\\
 &\le
 \frac{R_{\max}}{1-\gamma}
 +
 L_Q
 \sum_{p=1}^n
 \frac{\mathcal R_p}{p!}
 =
 \mathcal V_i^{(\kappa)}.
\end{align*}
Finally, Equation~\ref{eq:app-approx-critic-envelope} follows from
Equation~\ref{eq:app-ideal-critic-envelope},
Equation~\ref{eq:app-lin-zeta}, and the triangle inequality.
\end{proof}

\subsection{Finite-sample analysis of LTC--LSTD}
\label{app:lstd-proof}

For rollout $j\in[M_c]$ and $0\le t\le T-2$, write
\[
 \Phi_i^{j,t}
 :=
 \Phi_i\!\left(
 h_i(Z_t^{(j)}),S_t^{(j)},
 A_{t,\mathcal N_i^\kappa}^{(j)}
 \right),
 \qquad
 n_c:=M_c(T-1).
\]
The empirical LSTD system is
\begin{equation}
 \widehat A_i\widehat w_i=\widehat b_i,
 \label{eq:app-lstd-system}
\end{equation}
with
\begin{equation}
 \widehat A_i
 =
 \frac1{n_c}
 \sum_{j=1}^{M_c}
 \sum_{t=0}^{T-2}
 \Phi_i^{j,t}
 \left(
 \Phi_i^{j,t}
 -\gamma\Phi_i^{j,t+1}
 \right)^\top,
 \qquad
 \widehat b_i
 =
 \frac1{n_c}
 \sum_{j=1}^{M_c}
 \sum_{t=0}^{T-2}
 \Phi_i^{j,t}r_i^{j,t}.
 \label{eq:app-lstd-matrices}
\end{equation}

\paragraph{Proof of Theorem~\ref{thm:ltc-finite}.}
\begin{proof}
Fix an agent $i$ and condition as described at the beginning of this section.
The policy, centroids, cluster assignment, and feature map are therefore fixed.

Fix $\zeta>0$ and let $w_i^{(\zeta)}$ be supplied by
Proposition~\ref{prop:app-lin-error}. Define
\[
 f^{(\zeta)}
 (y_i,s,a_{\mathcal N_i^\kappa})
 :=
 \widetilde Q_i^{\mathrm{LTC}}
 (y_i,s,a_{\mathcal N_i^\kappa};w_i^{(\zeta)})
\]
and abbreviate
\[
 \Delta_\zeta
 :=
 \Delta_{\rm TC}+\mathcal E_i^{(\kappa)}+\zeta,
 \qquad
 \mathcal V_\zeta
 :=
 \mathcal V_i^{(\kappa)}+\mathcal E_i^{(\kappa)}+\zeta.
\]
Propositions~\ref{prop:app-tc-error} and~\ref{prop:app-lin-error} imply
\begin{equation}
 \sup_{z,s,a}
 \left|
 Q_i^\pi(z,s,a)
 -
 f^{(\zeta)}
 (h_i(z),s,a_{\mathcal N_i^\kappa})
 \right|
 \le
 \Delta_\zeta.
 \label{eq:app-best-class-error}
\end{equation}

\paragraph{Step 1: reduction to the empirical LSTD residual.}
The triangle inequality gives
\begin{align}
 &\sup_{z,s,a}
 \left|
 Q_i^\pi(z,s,a)
 -
 \widetilde Q_i^{\mathrm{LTC}}
 (h_i(z),s,a_{\mathcal N_i^\kappa};\widehat w_i)
 \right|
 \nonumber\\
 &\qquad\le
 \Delta_\zeta
 +
 \sup_{z,s,a}
 \left|
 f^{(\zeta)}
 (h_i(z),s,a_{\mathcal N_i^\kappa})
 -
 \widetilde Q_i^{\mathrm{LTC}}
 (h_i(z),s,a_{\mathcal N_i^\kappa};\widehat w_i)
 \right|.
 \label{eq:app-lstd-triangle}
\end{align}
Both functions in the second term belong to the same linear feature class.
By Lemma~\ref{lem:app-feature-envelope},
\[
 \sup_{z,s,a}
 \left|
 f^{(\zeta)}
 -
 \widetilde Q_i^{\mathrm{LTC}}(\cdot;\widehat w_i)
 \right|
 \le
 \sqrt{\mathcal G_i^{(\kappa)}}
 \left\|
 w_i^{(\zeta)}-\widehat w_i
 \right\|.
\]
On the event
$\sigma_{\min}(\widehat A_i)\ge c_{\mathrm{LSTD}}>0$,
$\widehat A_i$ is invertible and
\begin{align}
 \left\|
 w_i^{(\zeta)}-\widehat w_i
 \right\|
 &=
 \left\|
 \widehat A_i^{-1}
 \left(
 \widehat A_iw_i^{(\zeta)}-\widehat b_i
 \right)
 \right\|
 \nonumber\\
 &\le
 \frac1{c_{\mathrm{LSTD}}}
 \left\|
 \widehat A_iw_i^{(\zeta)}-\widehat b_i
 \right\|.
 \label{eq:app-parameter-residual}
\end{align}
Consequently,
\begin{equation}
 \sup_{z,s,a}
 \left|
 f^{(\zeta)}
 -
 \widetilde Q_i^{\mathrm{LTC}}(\cdot;\widehat w_i)
 \right|
 \le
 \frac{\sqrt{\mathcal G_i^{(\kappa)}}}{c_{\mathrm{LSTD}}}
 \left\|
 \widehat A_iw_i^{(\zeta)}-\widehat b_i
 \right\|.
 \label{eq:app-function-residual}
\end{equation}

\paragraph{Step 2: Bellman bias and martingale fluctuation.}
For each $(j,t)$ set
\[
 f_{j,t}^{(\zeta)}
 :=
 f^{(\zeta)}
 \left(
 h_i(Z_t^{(j)}),S_t^{(j)},
 A_{t,\mathcal N_i^\kappa}^{(j)}
 \right).
\]
By Equation~\ref{eq:app-lstd-matrices},
\begin{equation}
 \widehat A_iw_i^{(\zeta)}-\widehat b_i
 =
 \frac1{n_c}
 \sum_{j=1}^{M_c}
 \sum_{t=0}^{T-2}
 \Phi_i^{j,t}
 \left[
 f_{j,t}^{(\zeta)}
 -
 r_i^{j,t}
 -
 \gamma f_{j,t+1}^{(\zeta)}
 \right].
 \label{eq:app-lstd-raw-residual}
\end{equation}
Let
\[
 \widetilde f^{(\zeta)}(z,s,a)
 :=
 f^{(\zeta)}
 (h_i(z),s,a_{\mathcal N_i^\kappa})
\]
and define the local Bellman evaluation operator
\[
 (\mathcal T_i^\pi u)(z,s,a)
 :=
 r_i(s_i,a_i)
 +
 \gamma
 \mathbb E_\pi
 \left[
 u(Z_1,S_1,A_1)
 \,\middle|\,
 Z_0=z,S_0=s,A_0=a
 \right].
\]
Adding and subtracting the one-step conditional mean of
$f_{j,t+1}^{(\zeta)}$ decomposes
\[
 \widehat A_iw_i^{(\zeta)}-\widehat b_i
 =
 R_{\rm Bell}^{(\zeta)}
 +
 R_{\rm MG}^{(\zeta)},
\]
where
\begin{align*}
 R_{\rm Bell}^{(\zeta)}
 &:=
 \frac1{n_c}
 \sum_{j,t}
 \Phi_i^{j,t}
 \left[
 f_{j,t}^{(\zeta)}
 -
 r_i^{j,t}
 -
 \gamma
 \mathbb E
 \left[
 f_{j,t+1}^{(\zeta)}
 \mid
 Z_t^{(j)},S_t^{(j)},A_t^{(j)}
 \right]
 \right],\\
 R_{\rm MG}^{(\zeta)}
 &:=
 \frac{\gamma}{n_c}
 \sum_{j,t}
 \Phi_i^{j,t}
 \left[
 \mathbb E
 \left[
 f_{j,t+1}^{(\zeta)}
 \mid
 Z_t^{(j)},S_t^{(j)},A_t^{(j)}
 \right]
 -
 f_{j,t+1}^{(\zeta)}
 \right].
\end{align*}
Because $Q_i^\pi=\mathcal T_i^\pi Q_i^\pi$ and conditional expectation is a
contraction in the sup norm,
\begin{align*}
 \left\|
 \widetilde f^{(\zeta)}
 -
 \mathcal T_i^\pi\widetilde f^{(\zeta)}
 \right\|_\infty
 &\le
 \left\|
 \widetilde f^{(\zeta)}-Q_i^\pi
 \right\|_\infty
 +
 \left\|
 \mathcal T_i^\pi\widetilde f^{(\zeta)}
 -
 \mathcal T_i^\pi Q_i^\pi
 \right\|_\infty\\
 &\le
 (1+\gamma)\Delta_\zeta.
\end{align*}
Using $\|\Phi_i^{j,t}\|\le\sqrt{\mathcal G_i^{(\kappa)}}$,
\begin{equation}
 \left\|
 R_{\rm Bell}^{(\zeta)}
 \right\|
 \le
 \sqrt{\mathcal G_i^{(\kappa)}}(1+\gamma)\Delta_\zeta.
 \label{eq:app-rbell-bound}
\end{equation}

For the stochastic term, define
\[
 Z_{j,t}^{(\zeta)}
 :=
 \gamma\Phi_i^{j,t}
 \left[
 \mathbb E
 \left[
 f_{j,t+1}^{(\zeta)}
 \mid
 Z_t^{(j)},S_t^{(j)},A_t^{(j)}
 \right]
 -
 f_{j,t+1}^{(\zeta)}
 \right].
\]
Order the pairs $(j,t)$ by revealing each rollout sequentially and then moving
to the next independent rollout. Relative to the resulting natural
filtration, $Z_{j,t}^{(\zeta)}$ is a martingale difference: the factor
$\Phi_i^{j,t}$ is measurable at the current step and the term in brackets has
conditional mean zero. Moreover, by
Equation~\ref{eq:app-approx-critic-envelope},
\[
 \left\|
 Z_{j,t}^{(\zeta)}
 \right\|
 \le
 2\gamma\sqrt{\mathcal G_i^{(\kappa)}}\,\mathcal V_\zeta.
\]
The Hilbert-space Azuma--Hoeffding inequality~\citep{pinelis1994optimum} therefore yields, for this
fixed agent, with probability at least $1-\delta/N$,
\begin{equation}
 \left\|
 R_{\rm MG}^{(\zeta)}
 \right\|
 \le
 2\gamma\sqrt{\mathcal G_i^{(\kappa)}}\,\mathcal V_\zeta
 \sqrt{
 \frac{2\log(2N/\delta)}
 {M_c(T-1)}
 }.
 \label{eq:app-rmg-bound}
\end{equation}
In particular, no independence assumption is made between transitions from the
same rollout.

\paragraph{Step 3: conclusion.}
Combining Equation~\ref{eq:app-rbell-bound} and Equation~\ref{eq:app-rmg-bound}, then using
Equation~\ref{eq:app-function-residual} and Equation~\ref{eq:app-lstd-triangle}, gives
\begin{align*}
 &\sup_{z,s,a}
 \left|
 Q_i^\pi(z,s,a)
 -
 \widetilde Q_i^{\mathrm{LTC}}
 (h_i(z),s,a_{\mathcal N_i^\kappa};\widehat w_i)
 \right|\\
 &\quad\le
 \left(
 1+
 \frac{(1+\gamma)\mathcal G_i^{(\kappa)}}{c_{\mathrm{LSTD}}}
 \right)\Delta_\zeta
 +
 \frac{2\gamma \mathcal G_i^{(\kappa)}\mathcal V_\zeta}{c_{\mathrm{LSTD}}}
 \sqrt{
 \frac{2\log(2N/\delta)}
 {M_c(T-1)}
 }.
\end{align*}
The inequality holds for every $\zeta>0$. Letting $\zeta\downarrow0$ gives
\begin{align*}
 &\sup_{z,s,a}
 \left|
 Q_i^\pi(z,s,a)
 -
 \widetilde Q_i^{\mathrm{LTC}}
 (h_i(z),s,a_{\mathcal N_i^\kappa};\widehat w_i)
 \right|\\
 &\quad\le
 \left(
 1+
 \frac{(1+\gamma)\mathcal G_i^{(\kappa)}}{c_{\mathrm{LSTD}}}
 \right)
 (\Delta_{\rm TC}+\mathcal E_i^{(\kappa)})
 \\
 &\qquad+
 \frac{
 2\gamma \mathcal G_i^{(\kappa)}
 (\mathcal V_i^{(\kappa)}+\mathcal E_i^{(\kappa)})
 }{c_{\mathrm{LSTD}}}
 \sqrt{
 \frac{2\log(2N/\delta)}
 {M_c(T-1)}
 }.
\end{align*}
For each fixed $i$, the bound fails with probability at most $\delta/N$.
A union bound over $i\in[N]$ therefore yields the stated simultaneous
probability at least $1-\delta$. This is exactly the bound stated in
Theorem~\ref{thm:ltc-finite}.
\end{proof}

\subsection{Localized policy gradient with one-step TD advantages}
\label{app:actor-proof}

Throughout this subsection, in addition to the assumptions stated in the main
text, we use the standard regularity conditions for the likelihood-ratio
policy-gradient identity: the policy is differentiable in its parameters, its
support is parameter-independent, and differentiation may be interchanged with
the relevant action integrals. The bounded-score assumption supplies the
required domination.

Define
\[
 \mathcal E_i
 :=
 \mathcal N_i^{\kappa+\kappa_\pi},
 \qquad
 n_i:=|\mathcal E_i|.
\]
For a local reward stream $\ell$, let
\[
 V_\ell^\pi(z,s)
 :=
 \mathbb E_{a\sim\pi_\theta(\cdot\mid z,s)}
 \left[
 Q_\ell^\pi(z,s,a)
 \right],
 \qquad
 A_\ell^\pi
 :=
 Q_\ell^\pi-V_\ell^\pi.
\]

\begin{lemma}[Policy-gradient identity]
\label{lem:app-pg}
For every agent $i$,
\begin{equation}
 \nabla_{\theta_i}J(\theta)
 =
 \mathbb E_{\pi_\theta}
 \left[
 \sum_{t=0}^\infty
 \gamma^t
 \mathsf s_i(O_{i,t},A_{i,t})
 \frac1N
 \sum_{\ell=1}^N
 A_\ell^\pi(Z_t,S_t,A_t)
 \right].
 \label{eq:app-exact-pg}
\end{equation}
More generally, any bounded baseline depending on $(Z_t,S_t)$ but not on the
current action may be subtracted inside the expectation.
\end{lemma}

\begin{proof}
See~\citep{sutton1999policy}. We also use that the global advantage is the average of individual advantages.
\end{proof}

\paragraph{Localization of discarded agents.}
Set $R:=\kappa+\kappa_\pi$ and
$\mathcal E_i:=\mathcal N_i^R$. For any integer $r\ge0$, let
$\bar Q_{\ell,r}^\pi$ denote the marginalized critic of agent $\ell$
localized to $\mathcal N_\ell^r$, defined exactly as the localized critic used
above but with localization radius $r$. The localization bound
of Lemma~\ref{lem:app-localization-error} applies for every such $r$.

For each discarded agent $\ell\notin\mathcal E_i$, write
\[
d_{i\ell}:=d_G(i,\ell),
\qquad
r_{i\ell}:=d_{i\ell}-1,
\]
and define the auxiliary localized advantage
\[
\bar A_{\ell\mid i}^\pi(z,s,a)
:=
\bar Q_{\ell,r_{i\ell}}^\pi
\bigl(
z,s_{\mathcal N_\ell^{r_{i\ell}}},
a_{\mathcal N_\ell^{r_{i\ell}}}
\bigr)
-
V_\ell^\pi(z,s).
\]
These auxiliary critics are used only in the analysis; the algorithm still
uses the common critic radius $\kappa$.

\begin{lemma}[Vanishing localized contribution of far reward streams]
\label{lem:app-far-vanish}
For every $t$,
\begin{equation}
\mathbb E
\left[
\mathsf s_i(O_{i,t},A_{i,t})
\frac1N
\sum_{\ell\notin\mathcal E_i}
\bar A_{\ell\mid i}^\pi(Z_t,S_t,A_t)
\right]
=
0.
\label{eq:app-far-vanish}
\end{equation}
\end{lemma}

\begin{proof}
Fix $\ell\notin\mathcal E_i$. By construction,
$r_{i\ell}=d_G(i,\ell)-1$, hence
$i\notin\mathcal N_\ell^{r_{i\ell}}$. Therefore
$\bar Q_{\ell,r_{i\ell}}^\pi$ does not depend on the current action $a_i$;
neither does $V_\ell^\pi(z,s)$. Thus
$\bar A_{\ell\mid i}^\pi$ is independent of $a_i$.

Condition on $(Z_t,S_t,A_{-i,t})$. We have,
\[
A_{i,t}
\sim
\pi_{\theta_i}
\left(
\cdot\mid
S_{t,\mathcal N_i^{\kappa_\pi}},h_i(Z_t)
\right).
\]
Under this conditioning,
$\bar A_{\ell\mid i}^\pi(Z_t,S_t,A_t)$ is constant as a function of
$A_{i,t}$, whereas the conditional expectation of
$\mathsf s_i(O_{i,t},A_{i,t})$ is zero. Hence
\[
\mathbb E
\left[
\mathsf s_i(O_{i,t},A_{i,t})
\bar A_{\ell\mid i}^\pi(Z_t,S_t,A_t)
\right]
=
0.
\]
Summing over $\ell\notin\mathcal E_i$ proves
Equation~\ref{eq:app-far-vanish}.
\end{proof}

Define the distance-resolved tail
\begin{equation}
\tau_i(R)
:=
\frac1N
\sum_{\ell:\,d_G(i,\ell)>R}
\rho^{d_G(i,\ell)}.
\label{eq:app-distance-tail}
\end{equation}

\begin{lemma}[Spatial truncation bias of the actor]
\label{lem:app-actor-truncation}
Let
\[
g_i^{\rm loc}
:=
\mathbb E_{\pi_\theta}
\left[
\sum_{t=0}^\infty
\gamma^t
\mathsf s_i(O_{i,t},A_{i,t})
\frac1N
\sum_{\ell\in\mathcal E_i}
A_\ell^\pi(Z_t,S_t,A_t)
\right].
\]
Then
\begin{equation}
\left\|
\nabla_{\theta_i}J(\theta)-g_i^{\rm loc}
\right\|
\le
\frac{2B_{\mathsf s,i}R_{\max}}{(1-\gamma)^2}
\tau_i(R).
\label{eq:app-actor-truncation-bound}
\end{equation}
\end{lemma}

\begin{proof}
By Lemma~\ref{lem:app-pg}, the difference between the exact and localized
gradients is the contribution of agents $\ell\notin\mathcal E_i$. For each
such agent,
\[
A_\ell^\pi
=
\bar A_{\ell\mid i}^\pi
+
\bigl(
A_\ell^\pi-\bar A_{\ell\mid i}^\pi
\bigr).
\]
The contribution of $\bar A_{\ell\mid i}^\pi$ vanishes by
Lemma~\ref{lem:app-far-vanish}. Moreover,
\[
A_\ell^\pi-\bar A_{\ell\mid i}^\pi
=
Q_\ell^\pi-\bar Q_{\ell,r_{i\ell}}^\pi.
\]
Applying Lemma~\ref{lem:app-localization-error} at radius
$r_{i\ell}=d_{i\ell}-1$ gives
\[
\left|
A_\ell^\pi-\bar A_{\ell\mid i}^\pi
\right|
\le
\frac{2R_{\max}}{1-\gamma}
\rho^{r_{i\ell}+1}
=
\frac{2R_{\max}}{1-\gamma}
\rho^{d_G(i,\ell)}.
\]
Consequently,
\[
\left|
\frac1N
\sum_{\ell\notin\mathcal E_i}
\bigl(
A_\ell^\pi-\bar A_{\ell\mid i}^\pi
\bigr)
\right|
\le
\frac{2R_{\max}}{1-\gamma}\tau_i(R).
\]
Using $\|\mathsf s_i\|\le B_{\mathsf s,i}$ and
$\sum_{t\ge0}\gamma^t=(1-\gamma)^{-1}$ proves
Equation~\ref{eq:app-actor-truncation-bound}.
\end{proof}

For the remainder of the actor analysis, suppose that the fitted critics
satisfy
\begin{equation}
 \sup_{z,s,a}
 \left|
 \widetilde Q_\ell
 (h_\ell(z),s,a_{\mathcal N_\ell^\kappa})
 -
 Q_\ell^\pi(z,s,a)
 \right|
 \le
 \Delta_Q,
 \qquad
 \ell\in[N].
 \label{eq:app-generic-critic-certificate}
\end{equation}
Define
\begin{equation}
 \widetilde V_\ell(z,s)
 :=
 \mathbb E_{a\sim\pi_\theta(\cdot\mid z,s)}
 \left[
 \widetilde Q_\ell
 (h_\ell(z),s,a_{\mathcal N_\ell^\kappa})
 \right].
 \label{eq:app-vtilde}
\end{equation}

\begin{lemma}[One-step TD advantage error]
\label{lem:app-td-advantage}
Let
\[
 e_\ell(z,s)
 :=
 \widetilde V_\ell(z,s)-V_\ell^\pi(z,s).
\]
Then $\|e_\ell\|_\infty\le\Delta_Q$. Writing $X_t=(Z_t,S_t)$, the one-step
TD advantage satisfies
\begin{equation}
 \mathbb E
 \left[
 \widehat A_{\ell,t}
 \,\middle|\,
 X_t,A_t
 \right]
 =
 A_\ell^\pi(X_t,A_t)
 +
 \gamma\,
 \mathbb E\!\left[e_\ell(X_{t+1})\mid X_t,A_t\right]
 -
 e_\ell(X_t).
 \label{eq:app-td-decomp}
\end{equation}
Consequently, for every agent $i$,
\begin{equation}
 \left\|
 \mathbb E\!\left[
 \mathsf s_i(O_{i,t},A_{i,t})
 \bigl(
 \widehat A_{\ell,t}-A_\ell^\pi(X_t,A_t)
 \bigr)
 \right]
 \right\|
 \le
 \gamma B_{\mathsf s,i}\Delta_Q.
 \label{eq:app-td-score-bias}
\end{equation}
\end{lemma}

\begin{proof}
By Equations~\ref{eq:app-vtilde} and
\ref{eq:app-generic-critic-certificate},
\[
 |e_\ell(z,s)|
 \le
 \mathbb E_{A\sim\pi_\theta(\cdot\mid z,s)}
 \left|
 \widetilde Q_\ell
 (h_\ell(z),s,A_{\mathcal N_\ell^\kappa})
 -
 Q_\ell^\pi(z,s,A)
 \right|
 \le
 \Delta_Q.
\]
Let
\[
 \delta_{\ell,t}^\pi
 :=
 r_{\ell,t}
 +
 \gamma V_\ell^\pi(X_{t+1})
 -
 V_\ell^\pi(X_t).
\]
Since $\widetilde V_\ell=V_\ell^\pi+e_\ell$,
\[
 \widehat A_{\ell,t}
 =
 \delta_{\ell,t}^\pi
 +
 \gamma e_\ell(X_{t+1})
 -
 e_\ell(X_t).
\]
The Bellman identity gives
$\mathbb E[\delta_{\ell,t}^\pi\mid X_t,A_t]
=A_\ell^\pi(X_t,A_t)$, proving
Equation~\ref{eq:app-td-decomp}.

Multiplying by the policy score and taking expectations, the term involving
$e_\ell(X_t)$ vanishes because it is action-independent and
$\mathbb E[\mathsf s_i(O_{i,t},A_{i,t})\mid X_t]=0$. Hence
\begin{align*}
 &\left\|
 \mathbb E\!\left[
 \mathsf s_i(O_{i,t},A_{i,t})
 \bigl(
 \widehat A_{\ell,t}-A_\ell^\pi(X_t,A_t)
 \bigr)
 \right]
 \right\|\\
 &\qquad\le
 \gamma\,
 \mathbb E\!\left[
 \|\mathsf s_i(O_{i,t},A_{i,t})\|
 \left|
 \mathbb E[e_\ell(X_{t+1})\mid X_t,A_t]
 \right|
 \right]
 \le
 \gamma B_{\mathsf s,i}\Delta_Q,
\end{align*}
which proves Equation~\ref{eq:app-td-score-bias}.
\end{proof}

\begin{lemma}[Bias of the localized finite-horizon actor estimator]
\label{lem:app-actor-bias}
Let
\[
\Gamma_T
:=
\sum_{t=0}^{T-1}\gamma^t
=
\frac{1-\gamma^T}{1-\gamma}.
\]
Then
\begin{align}
\left\|
\mathbb E[\widehat g_i]
-
\nabla_{\theta_i}J(\theta)
\right\|
\le{}&
\frac{2B_{\mathsf s,i}R_{\max}}{(1-\gamma)^2}
\tau_i(R)
\nonumber\\
&+
\frac{B_{\mathsf s,i}n_i}{N}
\gamma\Gamma_T\Delta_Q
+
\frac{
2B_{\mathsf s,i}n_iR_{\max}\gamma^T
}{
N(1-\gamma)^2
}.
\label{eq:app-component-bias}
\end{align}
The expectation is over the fresh actor-rollout randomness under the
conditioning convention above.
\end{lemma}

\begin{proof}
Define the exact localized policy gradient
\[
g_i^{\rm loc}
:=
\mathbb E_{\pi_\theta}
\left[
\sum_{t=0}^{\infty}
\gamma^t
\mathsf s_i(O_{i,t},A_{i,t})
\frac1N
\sum_{\ell\in\mathcal E_i}
A_\ell^\pi(Z_t,S_t,A_t)
\right],
\]
and its finite-horizon counterpart
\[
g_{i,T}^{\rm loc}
:=
\mathbb E_{\pi_\theta}
\left[
\sum_{t=0}^{T-1}
\gamma^t
\mathsf s_i(O_{i,t},A_{i,t})
\frac1N
\sum_{\ell\in\mathcal E_i}
A_\ell^\pi(Z_t,S_t,A_t)
\right].
\]
By the definition of the actor estimator,
\[
\mathbb E[\widehat g_i]
=
\mathbb E_{\pi_\theta}
\left[
\sum_{t=0}^{T-1}
\gamma^t
\mathsf s_i(O_{i,t},A_{i,t})
\frac1N
\sum_{\ell\in\mathcal E_i}
\widehat A_{\ell,t}
\right].
\]
Hence,
\begin{align*}
\left\|
\mathbb E[\widehat g_i]
-
\nabla_{\theta_i}J(\theta)
\right\|
\le{}&
\left\|
\nabla_{\theta_i}J(\theta)-g_i^{\rm loc}
\right\|
\\
&+
\left\|
g_i^{\rm loc}-g_{i,T}^{\rm loc}
\right\|
\\
&+
\left\|
g_{i,T}^{\rm loc}-\mathbb E[\widehat g_i]
\right\|.
\end{align*}

The first term is the spatial-localization error. By
Lemma~\ref{lem:app-actor-truncation},
\[
\left\|
\nabla_{\theta_i}J(\theta)-g_i^{\rm loc}
\right\|
\le
\frac{2B_{\mathsf s,i}R_{\max}}
{(1-\gamma)^2}
\tau_i(R).
\]

For the finite-horizon truncation term, using
$|A_\ell^\pi|\le 2R_{\max}/(1-\gamma)$,
$\|\mathsf s_i\|\le B_{\mathsf s,i}$, and
$|\mathcal E_i|=n_i$,
\begin{align*}
\left\|
g_i^{\rm loc}-g_{i,T}^{\rm loc}
\right\|
&\le
\sum_{t=T}^{\infty}
\gamma^t
B_{\mathsf s,i}
\frac{n_i}{N}
\frac{2R_{\max}}{1-\gamma}
\\
&=
\frac{
2B_{\mathsf s,i}n_iR_{\max}\gamma^T
}{
N(1-\gamma)^2
}.
\end{align*}

Finally, for the critic-approximation term,
Equation~\ref{eq:app-td-score-bias} gives, for every
$\ell\in\mathcal E_i$ and $t<T$,
\[
\left\|
\mathbb E\left[
\mathsf s_i(O_{i,t},A_{i,t})
\bigl(
\widehat A_{\ell,t}
-
A_\ell^\pi(Z_t,S_t,A_t)
\bigr)
\right]
\right\|
\le
\gamma B_{\mathsf s,i}\Delta_Q.
\]
Therefore,
\begin{align*}
\left\|
g_{i,T}^{\rm loc}
-
\mathbb E[\widehat g_i]
\right\|
&\le
\frac1N
\sum_{t=0}^{T-1}
\gamma^t
\sum_{\ell\in\mathcal E_i}
\gamma B_{\mathsf s,i}\Delta_Q
\\
&=
\frac{B_{\mathsf s,i}n_i}{N}
\gamma\Gamma_T\Delta_Q.
\end{align*}

Combining the three bounds gives
Equation~\ref{eq:app-component-bias}.
\end{proof}

\begin{lemma}[Actor-rollout concentration]
\label{lem:app-actor-concentration}
Let
\[
 D
 :=
 \frac{2R_{\max}}{1-\gamma}
 +
 (1+\gamma)\Delta_Q,
 \qquad
 G_i
 :=
 \frac{2B_{\mathsf s,i}n_iD}{N}\,\Gamma_T.
\]
For a single actor rollout $j$, let $\widehat g_i^{(j)}$ denote its
contribution to Equation~\ref{eq:actor-gradient-estimator}, so that
$\widehat g_i=M_s^{-1}\sum_{j=1}^{M_s}\widehat g_i^{(j)}$. Then
\[
 \|\widehat g_i^{(j)}\|
 \le
 \frac{G_i}{2}.
\]
Consequently, with probability at least $1-\delta$,
\begin{equation}
 \left\|
 \widehat g_i
 -
 \mathbb E[\widehat g_i]
 \right\|
 \le
 G_i
 \sqrt{
 \frac{2\log(2/\delta)}{M_s}
 }.
 \label{eq:app-actor-concentration}
\end{equation}
\end{lemma}

\begin{proof}
Equation~\ref{eq:app-generic-critic-certificate} implies
\[
 \|\widetilde V_\ell\|_\infty
 \le
 \frac{R_{\max}}{1-\gamma}
 +
 \Delta_Q.
\]
Hence every one-step TD advantage satisfies
\[
 |\widehat A_{\ell,t}|
 \le
 R_{\max}
 +
 (1+\gamma)
 \left(
 \frac{R_{\max}}{1-\gamma}
 +
 \Delta_Q
 \right)
 =
 D.
\]
Therefore, for one rollout,
\[
 \|\widehat g_i^{(j)}\|
 \le
 \frac{B_{\mathsf s,i}n_iD}{N}
 \sum_{t=0}^{T-1}\gamma^t
 =
 \frac{G_i}{2}.
\]
Thus
\[
 \left\|
 \widehat g_i^{(j)}
 -
 \mathbb E[\widehat g_i^{(j)}]
 \right\|
 \le
 G_i.
\]
The $M_s$ actor rollouts are independent under the conditioning convention.
Applying the Hilbert-space Hoeffding inequality to their centered
contributions gives Equation~\ref{eq:app-actor-concentration}.
\end{proof}

\subsection{End-to-end policy optimization}
\label{app:end-to-end-proof}

For actor iteration $m$, define
\begin{align*}
\epsilon_{{\rm bias},i,m}
:={}&
\frac{2B_{\mathsf s,i}R_{\max}}{(1-\gamma)^2}
\tau_i(R)
+
\frac{B_{\mathsf s,i}n_i}{N}
\gamma\Gamma_T\Delta_Q^{(m)}\\
&+
\frac{
2B_{\mathsf s,i}n_iR_{\max}\gamma^T
}{
N(1-\gamma)^2
}.
\end{align*}
Let $G_{i,m}$ denote the quantity $G_i$ in
Lemma~\ref{lem:app-actor-concentration} with
$\Delta_Q=\Delta_Q^{(m)}$, and set
\[
G_m
:=
\left(
\sum_{i=1}^N G_{i,m}^2
\right)^{1/2}.
\]
Also recall
\[
\mathcal T_R
=
\left(
\sum_{i=1}^N
B_{\mathsf s,i}^2\tau_i(R)^2
\right)^{1/2},
\qquad
\mathcal B_R
=
\left(
\sum_{i=1}^N
\left(\frac{B_{\mathsf s,i}n_i}{N}\right)^2
\right)^{1/2}.
\]
Then the componentwise bias bounds and Minkowski's inequality give
\begin{align}
\left(
\sum_{i=1}^N
\epsilon_{{\rm bias},i,m}^2
\right)^{1/2}
\le{}&
\frac{2R_{\max}}{(1-\gamma)^2}\mathcal T_R
\nonumber\\
&+
\mathcal B_R
\left(
\gamma\Gamma_T\Delta_Q^{(m)}
+
\frac{2R_{\max}\gamma^T}{(1-\gamma)^2}
\right).
\label{eq:app-global-bias}
\end{align}
Consequently, the quantity $\epsilon_m$ appearing in
Theorem~\ref{thm:end-to-end} may be taken as
\[
\epsilon_m
=
\frac{2R_{\max}}{(1-\gamma)^2}\mathcal T_R
+
\mathcal B_R
\left(
\gamma\Gamma_T\Delta_Q^{(m)}
+
\frac{2R_{\max}\gamma^T}{(1-\gamma)^2}
\right)
+
G_m
\sqrt{
\frac{2\log(4M/\delta)}{M_s}
}.
\]

\begin{lemma}[Simultaneous gradient control]
\label{lem:app-simultaneous-gradient}
Conditionally on the critic certificates required by
Theorem~\ref{thm:end-to-end}, with probability at least $1-\delta$ over the
fresh actor rollouts, simultaneously for all $m=0,\ldots,M-1$,
\begin{equation}
 \left\|
 \widehat g_m-\nabla J(\theta_m)
 \right\|
 \le
 \epsilon_m,
 \qquad
 \|\widehat g_m\|
 \le
 \frac{G_m}{2}.
 \label{eq:app-simultaneous-gradient}
\end{equation}
\end{lemma}

\begin{proof}
Fix $m$ and condition on the history preceding its fresh actor rollouts.
For one rollout, concatenate the agent-wise contributions:
\[
 \widehat g_m^{(j)}
 :=
 \bigl(
 \widehat g_{1,m}^{(j)},\ldots,
 \widehat g_{N,m}^{(j)}
 \bigr).
\]
Lemma~\ref{lem:app-actor-concentration} gives
\[
 \|\widehat g_{i,m}^{(j)}\|
 \le
 \frac{G_{i,m}}2.
\]
Therefore
\[
 \|\widehat g_m^{(j)}\|^2
 =
 \sum_{i=1}^N
 \|\widehat g_{i,m}^{(j)}\|^2
 \le
 \frac14
 \sum_{i=1}^N
 G_{i,m}^2
 =
 \frac{G_m^2}{4}.
\]
Consequently,
\[
 \left\|
 \widehat g_m^{(j)}
 -
 \mathbb E[\widehat g_m^{(j)}]
 \right\|
 \le
 G_m.
\]
Hilbert-space Hoeffding, applied conditionally with failure probability
$\delta/(2M)$, gives
\[
 \left\|
 \widehat g_m
 -
 \mathbb E[\widehat g_m]
 \right\|
 \le
 G_m
 \sqrt{
 \frac{2\log(4M/\delta)}{M_s}
 }.
\]
By Lemma~\ref{lem:app-actor-bias},
\[
 \left\|
 \mathbb E[\widehat g_m]
 -
 \nabla J(\theta_m)
 \right\|
 \le
 \left(
 \sum_{i=1}^N
 \epsilon_{{\rm bias},i,m}^2
 \right)^{1/2}.
\]
The triangle inequality gives the first inequality in
Equation~\ref{eq:app-simultaneous-gradient}. The second is deterministic under the
conditioning, since $\widehat g_m$ is the average of vectors of norm at most
$G_m/2$.

A union bound over $m$ gives conditional failure probability at most
$\delta/2$, hence in particular at most $\delta$. Because each conditional
bound holds uniformly over the realized past, the tower property yields the
same unconditional conclusion. If the critic certificates themselves are
obtained from Theorem~\ref{thm:ltc-finite}, their failure probabilities may be
allocated separately and combined by another union bound.
\end{proof}

\paragraph{Proof of Theorem~\ref{thm:end-to-end}.}
\begin{proof}
Work on the event in Lemma~\ref{lem:app-simultaneous-gradient} and define
\[
 e_m
 :=
 \widehat g_m-\nabla J(\theta_m),
 \qquad
 \|e_m\|\le\epsilon_m.
\]
Set
\[
 \eta_m
 :=
 \frac{\eta}{\sqrt{m+1}}.
\]
The $L_J$-smoothness of $J$ and the ascent update
$\theta_{m+1}=\theta_m+\eta_m\widehat g_m$ imply
\begin{align*}
 J(\theta_{m+1})
 &\ge
 J(\theta_m)
 +
 \eta_m
 \left\langle
 \nabla J(\theta_m),
 \widehat g_m
 \right\rangle
 -
 \frac{L_J}2
 \eta_m^2
 \|\widehat g_m\|^2\\
 &=
 J(\theta_m)
 +
 \eta_m
 \|\nabla J(\theta_m)\|^2
 +
 \eta_m
 \left\langle
 \nabla J(\theta_m),e_m
 \right\rangle
 -
 \frac{L_J}2
 \eta_m^2
 \|\widehat g_m\|^2\\
 &\ge
 J(\theta_m)
 +
 \eta_m
 \|\nabla J(\theta_m)\|^2
 -
 \eta_m
 \|\nabla J(\theta_m)\|
 \epsilon_m
 -
 \frac{L_J}8
 \eta_m^2
 G_m^2,
\end{align*}
where the last inequality uses
$\|\widehat g_m\|\le G_m/2$. Young's inequality gives
\[
 \|\nabla J(\theta_m)\|\epsilon_m
 \le
 \frac12
 \|\nabla J(\theta_m)\|^2
 +
 \frac12\epsilon_m^2.
\]
Hence
\[
 \frac{\eta_m}{2}
 \|\nabla J(\theta_m)\|^2
 \le
 J(\theta_{m+1})-J(\theta_m)
 +
 \frac{\eta_m}{2}\epsilon_m^2
 +
 \frac{L_J}8
 \eta_m^2G_m^2.
\]
Summing over $m=0,\ldots,M-1$, using
$J(\theta_M)\le J_{\max}=R_{\max}/(1-\gamma)$, and dividing by
$\frac12\sum_m\eta_m$ gives
\begin{align}
 \frac{
 \sum_{m=0}^{M-1}
 \eta_m
 \|\nabla J(\theta_m)\|^2
 }{
 \sum_{m=0}^{M-1}\eta_m
 }
 \le{}&
 \frac{
 2(J_{\max}-J(\theta_0))
 }{
 \sum_{m=0}^{M-1}\eta_m
 }
 \nonumber\\
 &+
 \frac{
 \sum_{m=0}^{M-1}
 \eta_m\epsilon_m^2
 }{
 \sum_{m=0}^{M-1}\eta_m
 }
 +
 \frac{L_J}4
 \frac{
 \sum_{m=0}^{M-1}
 \eta_m^2G_m^2
 }{
 \sum_{m=0}^{M-1}\eta_m
 }.
 \label{eq:app-weighted-convergence}
\end{align}
Now
\[
 \sum_{m=0}^{M-1}\eta_m
 =
 \eta
 \sum_{m=1}^M
 m^{-1/2}
 \ge
 \eta\sqrt M,
\]
and
\[
 \sum_{m=0}^{M-1}\eta_m^2
 =
 \eta^2
 \sum_{m=1}^M\frac1m
 \le
 \eta^2(1+\log M).
\]
If $\epsilon_m\le\bar\epsilon$ and $G_m\le\bar G$ for all $m$, then
\[
 \frac{
 \sum_m\eta_m\epsilon_m^2
 }{
 \sum_m\eta_m
 }
 \le
 \bar\epsilon^2
\]
and
\[
 \frac{
 \sum_m\eta_m^2G_m^2
 }{
 \sum_m\eta_m
 }
 \le
 \eta\bar G^2
 \frac{1+\log M}{\sqrt M}.
\]
Substituting these estimates into
Equation~\ref{eq:app-weighted-convergence} gives
\[
 \frac{
 \sum_{m=0}^{M-1}
 \eta_m
 \|\nabla J(\theta_m)\|^2
 }{
 \sum_{m=0}^{M-1}\eta_m
 }
 \le
 \frac{
 2(J_{\max}-J(\theta_0))
 }{
 \eta\sqrt M
 }
 +
 \bar\epsilon^2
 +
 \frac{L_J\eta\bar G^2}{4}
 \frac{1+\log M}{\sqrt M}.
\]
The minimum of the nonnegative quantities
$\|\nabla J(\theta_m)\|^2$ is bounded by any positive weighted average of
them, which proves the theorem.
\end{proof}

\subsection{Computational complexity}
\label{app:complexity-proof}

We derive the cost of one LTC actor--critic iteration. Let
\[
p_i
=
\binom{|\mathcal N_i^\kappa|d_s+n}{n}
\left[
1+|\mathcal N_i^\kappa|(d_s+d_a)+d_g^i
\right]
\]
denote the LTC feature dimension for agent \(i\), and let
\[
n_c=M_c(T-1),
\qquad
n_s=M_sT.
\]

\paragraph{Critic fitting.}
For each agent $i$, LSTD forms a $p_i\times p_i$ linear system
$\widehat A_i w_i=\widehat b_i$ from the $n_c$ critic transitions.
Constructing $\widehat A_i$ costs $O(n_c p_i^2)$, while forming the
right-hand-side vector $\widehat b_i$ costs only $O(n_c p_i)$.
A dense solve then costs $O(p_i^3)$. Hence, fitting all local critics costs
\[
\mathcal{O}\left(
\sum_{i=1}^N
\left(
n_cp_i^2+p_i^3
\right)
\right),
\]
while one evaluation of critic \(i\) costs \(\mathcal{O}(p_i)\).

\paragraph{Actor update.}
For each actor transition, agent \(i\) uses the local advantages
\(\widehat A_{\ell,t}\) for
\(\ell\in\mathcal N_i^{\kappa+\kappa_\pi}\). The actor phase therefore costs
\[
\mathcal{O}\left(
n_s
\sum_{i=1}^N
\sum_{\ell\in\mathcal N_i^{\kappa+\kappa_\pi}}
p_\ell
\right).
\]
For fixed local interaction and policy radii, rollout generation is linear in
\(N\), since both the dynamics and policies access only local neighborhoods.

\paragraph{Centroid update.}
Suppose that \(K\)-medians is fitted on \(n_K\) sampled global states for
\(I_K\) clustering iterations. Comparing all samples with all \(K\) centroids
in \(\mathbb R^{Nd_s}\) costs
\[
\mathcal{O}\left(
I_K n_K K N d_s
\right),
\]
with coordinatewise median updates of the same or lower order.

Combining these terms, one complete LTC iteration has time complexity
\[
\mathcal{O}\left(
\sum_{i=1}^N
\left(
n_cp_i^2+p_i^3
\right)
+
n_s
\sum_{i=1}^N
\sum_{\ell\in\mathcal N_i^{\kappa+\kappa_\pi}}
p_\ell
+
(M_c+M_s)TN
+
I_K n_K K N d_s
\right).
\]

To make the dependence on the network structure explicit, suppose the graph
has maximum degree \(\Delta\). For fixed state, action, and context
dimensions,
\[
|\mathcal N_i^\kappa|
=
\mathcal{O}(\Delta^\kappa),
\qquad
p_i
=
\mathcal{O}\left(
\Delta^{\kappa(n+1)}
\right),
\]
and
\[
|\mathcal N_i^{\kappa+\kappa_\pi}|
=
\mathcal{O}\left(
\Delta^{\kappa+\kappa_\pi}
\right).
\]
Therefore, one LTC iteration costs
\[
\mathcal{O}\left(
N
\left[
n_c\Delta^{2\kappa(n+1)}
+
\Delta^{3\kappa(n+1)}
+
n_s\Delta^{\kappa(n+2)+\kappa_\pi}
+
(M_c+M_s)T
+
I_K n_K K
\right]
\right),
\]
where fixed state, action, and context dimensions are absorbed into the
constants.

For fixed sampling and clustering budgets, this reduces to
\[
\boxed{
\mathcal{O}\left(
N
\left[
\Delta^{3\kappa(n+1)}
+
\Delta^{\kappa(n+2)+\kappa_\pi}
\right]
\right)
}.
\]
Thus, for fixed local parameters
\((\Delta,\kappa,\kappa_\pi,n)\), the total work of one LTC iteration scales
linearly with \(N\). The critic and actor computations are also largely
parallelizable across agents and samples.

\section{Proofs of the additional appendix results}
\label{app:auxiliary-proofs}

This section contains the proofs of the auxiliary results stated in
Appendix~\ref{app:regularity}. We first treat the discounted-summability
extension and its generalized Taylor Centroid approximation bound, then the
sufficient conditions based on stochastic and spectral transition models, and
finally the comparison with Mild smoothness.

\subsection{Discounted-summability results}

\paragraph{Proof of Proposition~\ref{prop:app-summable}.}
\begin{proof}
Fix a multi-index $\beta$ with $|\beta|\le q$. By
Assumption~\ref{asm:app-summable},
\[
\sum_{t=0}^{\infty}
\sup_{z,s,a}
\left|
\gamma^t\partial_s^\beta g_{i,t}^\pi(z,s,a)
\right|
\le
\sum_{t=0}^{\infty}\gamma^t\bar L(t)
<\infty.
\]
Hence the series of derivatives converges uniformly. The value series
\[
Q_i^\pi(z,s,a)=\sum_{t=0}^{\infty}\gamma^t g_{i,t}^\pi(z,s,a)
\]
also converges uniformly because $|g_{i,t}^\pi|\le R_{\max}$. Standard
term-by-term differentiation therefore gives
\[
\partial_s^\beta Q_i^\pi(z,s,a)
=
\sum_{t=0}^{\infty}
\gamma^t\partial_s^\beta g_{i,t}^\pi(z,s,a),
\]
which is Equation~\ref{eq:app-summable-diff}. Taking absolute values yields
\[
\sup_{z,s,a}
\left|
\partial_s^\beta Q_i^\pi(z,s,a)
\right|
\le
\sum_{t=0}^{\infty}\gamma^t\bar L(t),
\]
proving Equation~\ref{eq:app-summable-Q}.
\end{proof}

\paragraph{Proof of Proposition~\ref{prop:app-summable-graded}.}
\begin{proof}
Fix a nonzero multi-index $\beta$ and let
\[
d:=d_i(\beta),
\qquad
\tau_d:=\left\lceil\frac{d}{1+\kappa_\pi}\right\rceil.
\]
By Lemma~\ref{lem:app-light-cone}, before time $\tau_d$ the expected reward of
agent $i$ cannot depend on every state coordinate appearing in $\beta$.
Therefore
\[
\partial_s^\beta g_{i,t}^\pi(z,s,a)=0,
\qquad t<\tau_d.
\]
Using Proposition~\ref{prop:app-summable},
\begin{align*}
\left|\partial_s^\beta Q_i^\pi(z,s,a)\right|
&=
\left|
\sum_{t=\tau_d}^{\infty}
\gamma^t\partial_s^\beta g_{i,t}^\pi(z,s,a)
\right|\\
&\le
\sum_{t=\tau_d}^{\infty}\gamma^t\bar L(t)
=\Lambda_d,
\end{align*}
which is Equation~\ref{eq:app-summable-graded}. Since $\tau_d$ is nondecreasing in
$d$, $\Lambda_d$ is nonincreasing. Moreover, $\tau_d\to\infty$ as
$d\to\infty$, so $\Lambda_d\to0$ because
$\sum_t\gamma^t\bar L(t)<\infty$.
\end{proof}


\paragraph{Proof of Proposition~\ref{prop:app-summable-tc}.}
\begin{proof}
The value-localization part does not use expected reward smoothness. Hence
Lemma~\ref{lem:app-localization-error} still gives
\begin{equation}
\sup_{z,s,a}
\left|
Q_i^\pi(z,s,a)-\bar Q_i^\pi
(z,s_{\mathcal N_i^\kappa},a_{\mathcal N_i^\kappa})
\right|
\le
\frac{2R_{\max}}{1-\gamma}\rho^{\kappa+1}.
\label{eq:app-summable-localization}
\end{equation}

It remains to bound the Taylor remainder of the localized critic.
For any multi-index $\beta$ supported in $\mathcal N_i^\kappa$,
Proposition~\ref{prop:app-summable-graded} gives, uniformly in the far
coordinates,
\[
\left|
\partial_s^\beta Q_i^\pi(z,s,a)
\right|
\le
\Lambda_{d_i(\beta)}.
\]
Differentiation under the probability integral defining the marginalized
critic is therefore justified exactly as in
Lemma~\ref{lem:app-localized-derivatives}, and yields
\begin{equation}
\left|
\partial_s^\beta
\bar Q_i^\pi
(z,s_{\mathcal N_i^\kappa},a_{\mathcal N_i^\kappa})
\right|
\le
\Lambda_{d_i(\beta)}.
\label{eq:app-summable-localized-derivative}
\end{equation}

Fix $(z,s,a)$, let
\[
\mu:=\mu_{k(s),\mathcal N_i^\kappa},
\qquad
h:=s_{\mathcal N_i^\kappa}-\mu,
\qquad
m:=n+1,
\]
and use the integral form of the order-$n$ Taylor remainder of
$\bar Q_i^\pi$ along the segment from $\mu$ to
$s_{\mathcal N_i^\kappa}$. Taking absolute values and applying
Equation~\ref{eq:app-summable-localized-derivative} gives
\begin{equation}
\left|
\bar Q_i^\pi-\widehat Q_i^\pi
\right|
\le
\sum_{\substack{|\beta|=m\\
\operatorname{supp}_G(\beta)\subseteq\mathcal N_i^\kappa}}
\Lambda_{d_i(\beta)}
\frac{|h^\beta|}{\beta!}.
\label{eq:app-summable-taylor-step}
\end{equation}

For $0\le d\le\kappa$, define
\[
A_{m,d}
:=
\sum_{\substack{|\beta|=m\\
\operatorname{supp}_G(\beta)\subseteq\mathcal N_i^d}}
\frac{|h^\beta|}{\beta!}
=
\frac{\|h_{\mathcal N_i^d}\|_1^m}{m!},
\qquad
A_{m,-1}:=0,
\]
where the equality is the multinomial theorem. The contribution of
multi-indices with graph reach exactly $d$ is
$A_{m,d}-A_{m,d-1}$, so
\begin{align*}
\sum_{\substack{|\beta|=m\\
\operatorname{supp}_G(\beta)\subseteq\mathcal N_i^\kappa}}
\Lambda_{d_i(\beta)}
\frac{|h^\beta|}{\beta!}
&=
\sum_{d=0}^{\kappa}
\Lambda_d
\bigl(A_{m,d}-A_{m,d-1}\bigr)\\
&=
\Lambda_\kappa A_{m,\kappa}
+
\sum_{d=0}^{\kappa-1}
(\Lambda_d-\Lambda_{d+1})A_{m,d}\\
&\le
\frac{1}{m!}
\left[
\Lambda_\kappa W_\kappa^m
+
\sum_{d=0}^{\kappa-1}
(\Lambda_d-\Lambda_{d+1})W_d^m
\right]\\
&=
\frac{\mathcal R_m^\Lambda}{m!}.
\end{align*}
Combining this with
Equation~\ref{eq:app-summable-localization} and setting $m=n+1$ proves
Equation~\ref{eq:app-summable-tc}. The monomial count is unchanged from
Theorem~\ref{thm:trunc}, since the generalized envelope changes only the
coefficient bounds and not the set
$\mathcal M_{\mathcal N_i^\kappa}^n$.
\end{proof}

\subsection{Smooth stochastic and spectral transition models}

\paragraph{Proof of Proposition~\ref{prop:app-smooth-stochastic}.}
\begin{proof}
Fix an agent $j$ and abbreviate
\[
x_j:=(s_{\mathcal N_j},a_{\mathcal N_j}).
\]
Under Equation~\ref{eq:app-smooth-stochastic}, the conditional density of $S_j'$ is
\[
p_j^{\rm tr}(s_j'\mid z,x_j)
=
\frac{1}{|\det M_j(z,x_j)|}
\,p_j\!\left(
M_j(z,x_j)^{-1}
\bigl(s_j'-f_j(z,x_j)\bigr)
\right),
\]
where $p_j$ is the density of $\eta_j$ from the statement of the proposition.
Set
\[
u
:=
M_j(z,x_j)^{-1}
\bigl(s_j'-f_j(z,x_j)\bigr).
\]
Because $f_j$, $M_j$, and $M_j^{-1}$ are $C^q$ with uniformly bounded
derivatives, repeated use of the product and chain rules gives the following
standard bound: for every state multi-index $\alpha$ with
$|\alpha|=m\le q$,
\[
\left|
\partial_s^\alpha p_j^{\rm tr}(s_j'\mid z,x_j)
\right|
\le
\frac{C_m}{|\det M_j(z,x_j)|}
(1+\|u\|)^m
\sum_{|\nu|\le m}
\left|\partial^\nu p_j(u)\right|,
\]
where $C_m$ is uniform in $j,z,x_j$ and in the network size.

Now change variables from $s_j'$ to $u$. The Jacobian
$|\det M_j|$ cancels the determinant factor above, so
\begin{align*}
\int
\left|
\partial_s^\alpha p_j^{\rm tr}(s_j'\mid z,x_j)
\right|\,ds_j'
&\le
C_m
\sum_{|\nu|\le m}
\int
(1+\|u\|)^m
\left|\partial^\nu p_j(u)\right|\,du\\
&\le C_m',
\end{align*}
where $C_m'<\infty$ is uniform by assumption. Thus every local transition
density has state derivatives through order $q$ whose $L^1$ norms are
uniformly bounded.

We next use the graph factorization. Conditional on the current global state,
the density of the distributed next state is
\[
\prod_{j=1}^N
p_j^{\rm tr}
\left(
 s_j'\mid z,s_{\mathcal N_j},a_{\mathcal N_j}
\right).
\]
Let $\beta$ be a state multi-index of total order $m\le q$. A scalar state
coordinate of agent $\ell$ appears only in transition factors associated with
agents $j$ such that $\ell\in\mathcal N_j$. Since the graph degree is bounded
by $\Delta$, there are at most $\Delta+1$ such factors. Applying the product
rule one differentiation at a time therefore produces at most
$(\Delta+1)^m$ terms. In each term, the undifferentiated factors integrate to
one, while the differentiated factors have uniformly bounded $L^1$ norms by
the preceding calculation. By Fubini's theorem, there is therefore a constant
$C_q<\infty$, depending only on $q$, $\Delta$, and the local bounds, such that
for every bounded measurable $H$,
\begin{equation}
\sup_{z,s,a}
\left|
\partial_s^\beta
\mathbb E\!
\left[
H(Z_1,S_1)
\mid Z_0=z,S_0=s,A_0=a
\right]
\right|
\le
C_q\|H\|_\infty,
\qquad |\beta|\le q.
\label{eq:app-one-step-smoothing}
\end{equation}
The uniform $L^1$ bounds above also justify differentiating under the
integral.

For $t\ge1$, define
\[
H_{i,t-1}(z',s')
:=
\mathbb E_\pi\!\left[
 r_i(S_{i,t},A_{i,t})
 \mid Z_1=z',S_1=s'
\right].
\]
Since $|r_i|\le R_{\max}$, we have $\|H_{i,t-1}\|_\infty\le R_{\max}$. By the
Markov property,
\[
g_{i,t}^\pi(z,s,a)
=
\mathbb E\!\left[
H_{i,t-1}(Z_1,S_1)
\mid Z_0=z,S_0=s,A_0=a
\right].
\]
Applying Equation~\ref{eq:app-one-step-smoothing} gives
\[
\sup_{z,s,a}
\left|
\partial_s^\beta g_{i,t}^\pi(z,s,a)
\right|
\le C_qR_{\max},
\qquad t\ge1,
\qquad |\beta|\le q.
\]
At $t=0$,
\[
g_{i,0}^\pi(z,s,a)=r_i(s_i,a_i),
\]
and its state derivatives through order $q$ are uniformly bounded by the
first hypothesis of the proposition. Taking the larger of these two bounds
gives Assumption~\ref{asm:orbit}, with a constant independent of $N$.

If $Z$ takes finitely many values, the same proof applies mode by mode; uniform
modewise bounds give the same conclusion for hidden or observed finite-mode
systems.
\end{proof}

\paragraph{Proof of Proposition~\ref{prop:app-spectral}.}
\begin{proof}
For the local spectral representation Equation~\ref{eq:app-spectral-factorization},
a state derivative acts only on the current-state factors $\varphi_{j,h}$.
Thus, for every local state multi-index $\alpha$ with $|\alpha|\le q$,
\[
\partial_s^\alpha
p_j(s_j'\mid s_{\mathcal N_j},a_{\mathcal N_j},z)
=
\sum_{h=1}^{d}
\partial_s^\alpha\varphi_{j,h}
(s_{\mathcal N_j},a_{\mathcal N_j},z)
\,\mu_{j,h}(s_j').
\]
Consequently,
\begin{align*}
\int
\left|
\partial_s^\alpha
p_j(s_j'\mid s_{\mathcal N_j},a_{\mathcal N_j},z)
\right|ds_j'
&\le
\sum_{h=1}^{d}
\left|
\partial_s^\alpha\varphi_{j,h}
(s_{\mathcal N_j},a_{\mathcal N_j},z)
\right|
\|\mu_{j,h}\|_{L^1}\\
&\le C_{\rm spec},
\end{align*}
where $C_{\rm spec}$ is uniform by the assumptions of the proposition.

We can now repeat the same product-rule argument as above. Because a state
coordinate appears in at most $\Delta+1$ local transition factors, the
$L^1$ norm of every derivative through order $q$ of the factorized joint
next-state density is bounded by a constant independent of $N$. Therefore,
for every bounded measurable $H$,
\[
\sup_{z,s,a}
\left|
\partial_s^\beta
\mathbb E\!\left[
H(Z_1,S_1)
\mid Z_0=z,S_0=s,A_0=a
\right]
\right|
\le
C_q^{\rm spec}\|H\|_\infty,
\qquad |\beta|\le q.
\]

For $t\ge1$, use the bounded continuation function
\[
H_{i,t-1}(z',s')
:=
\mathbb E_\pi\!\left[
 r_i(S_{i,t},A_{i,t})
 \mid Z_1=z',S_1=s'
\right],
\qquad
\|H_{i,t-1}\|_\infty\le R_{\max}.
\]
The Markov property gives
\[
g_{i,t}^\pi(z,s,a)
=
\mathbb E\!\left[
H_{i,t-1}(Z_1,S_1)
\mid Z_0=z,S_0=s,A_0=a
\right],
\]
and hence
\[
\sup_{z,s,a}
\left|
\partial_s^\beta g_{i,t}^\pi(z,s,a)
\right|
\le
C_q^{\rm spec}R_{\max},
\qquad t\ge1.
\]
At $t=0$, $g_{i,0}^\pi=r_i$, whose derivatives through order $q$ are
uniformly bounded by assumption. This proves Assumption~\ref{asm:orbit} with
an $N$-independent constant.
\end{proof}

\subsection{Relation to Mild smoothness}

\paragraph{Proof of Proposition~\ref{prop:app-mild-orbit}.}
\begin{proof}
Let $x=(z,s)$ denote the augmented state. For a bounded measurable function
$H(x')$, define the state--action function
\[
F_H(x',a'):=H(x').
\]
Since $F_H$ does not depend on the action,
$\sup_{a'}F_H(x',a')=H(x')$. Hence, with the Bellman operator used in
Appendix~\ref{app:regularity},
\[
(\mathcal T^\star F_H)(x,a)
-(\mathcal T^\star0)(x,a)
=
\mathbb E\!\left[H(X_1)\mid X_0=x,A_0=a\right].
\]
If the MDP is $\nu$-Mildly Smooth, both Bellman terms belong to
$\mathcal C^\nu$ and
\[
\left\|
\mathbb E\!\left[H(X_1)\mid X_0=\cdot,A_0=\cdot\right]
\right\|_{\mathcal C^\nu}
\le
C_{\mathcal T}(\|H\|_\infty+2).
\]
Thus one transition maps an arbitrary bounded continuation function to a
uniformly smooth function of the current state--action pair.

For $t\ge1$, take
\[
H_{i,t-1}(x')
:=
\mathbb E_\pi\!\left[
 r_i(S_{i,t},A_{i,t})
 \mid X_1=x'
\right].
\]
This function is bounded by $R_{\max}$, and the Markov property gives
\[
g_{i,t}^\pi(x,a)
=
\mathbb E\!\left[
H_{i,t-1}(X_1)
\mid X_0=x,A_0=a
\right].
\]
Therefore
\[
\sup_{t\ge1}
\|g_{i,t}^\pi\|_{\mathcal C^\nu}
\le
C_{\mathcal T}(R_{\max}+2).
\]
Since $\nu>q$, the $\mathcal C^\nu$ norm controls the classical state
derivatives through order $q$ on the compact domain. Hence
\[
\sup_{t\ge1}
\|g_{i,t}^\pi\|_{C_s^q}
<\infty.
\]
At $t=0$, $g_{i,0}^\pi=r_i$, which is $C^q$ with uniformly bounded
derivatives by assumption. This proves Assumption~\ref{asm:orbit} of order
$q$. The integer-order implication Equation~\ref{eq:app-maran-to-orbit} follows
immediately from Equation~\ref{eq:app-maran-hierarchy}.
\end{proof}



\paragraph{Proof of Proposition~\ref{prop:app-continuous-switch}.}
\begin{proof}
For fixed $z$, the physical drift
\[
f(z,s,a)
=
A^{(\sigma(z))}s+B^{(\sigma(z))}a
\]
is linear in the physical state and action. Since there are only two modes,
its physical-state derivatives are uniformly bounded over $z$. With
$M\equiv I$, the noise assumptions are exactly those of
Proposition~\ref{prop:app-smooth-stochastic}, uniformly over the two regimes.
That proposition therefore yields Assumption~\ref{asm:orbit} at every finite
order $q$. Importantly, no smoothness of the policy is needed: after the first
transition, the future policy is absorbed into a bounded continuation value,
while differentiation with respect to the initial physical state acts only on
the smooth transition density. Hence the bound is uniform over all admissible
policies.

It remains to show that the augmented MDP is not Mildly Smooth. Let
$x=(z,s)$ and use the nonzero vector $d$ from the statement to define the
smooth bounded Bellman input
\[
F(x',a')
=
d^\top s'.
\]
The state domain is compact, so $F$ is bounded, and it is independent of the
next action. Since the noise is centered,
\[
(\mathcal T^\star F)((z,s),a)
=
r(s,a)
+
d^\top
\left(
A^{(\sigma(z))}s+B^{(\sigma(z))}a
\right).
\]
Evaluating this function at $(s,a)=(\bar s,\bar a)$, its one-sided values at
$z=0$ differ by
\[
d^\top
\left[
\left(A^{(1)}-A^{(0)}\right)\bar s
+
\left(B^{(1)}-B^{(0)}\right)\bar a
\right]
=
\|d\|^2
>0.
\]
Thus $\mathcal T^\star F$ is discontinuous in the continuous context
coordinate $z$. In particular, it does not belong to $\mathcal C^\nu$ for any
$\nu>0$. Therefore the augmented MDP is not Mildly Smooth at any positive
order.
\end{proof}

\section{Neural Network Implementation Details}
\label{app:ntc}
\subsection{NTC Implementation}
NTC retains the graph-local Taylor structure of LTC and only changes the parameterization of the Taylor coefficients.
Recall the local Taylor feature vector
\[
T_i(s)
:=
\left[
\frac{
\bigl(
s_{\mathcal N_i^\kappa}
-
\mu_{k(s),\mathcal N_i^\kappa}
\bigr)^\beta
}{\beta!}
\right]_
{\beta\in\mathcal M_{\mathcal N_i^\kappa}^n},
\]
and the coefficient context
\[
\psi_i
\bigl(
k(s),y_i,a_{\mathcal N_i^\kappa}
\bigr)
=
\left[
1;
\mu_{k(s),\mathcal N_i^\kappa};
a_{\mathcal N_i^\kappa};
y_i
\right].
\]
LTC models each Taylor coefficient as an affine function of this context.
NTC replaces this affine map by a neural coefficient map
\[
f_{\xi_i}:
\mathbb R^{d_\psi^i}
\longrightarrow
\mathbb R^{
|\mathcal M_{\mathcal N_i^\kappa}^n|
},
\]
and defines
\begin{equation}
\widetilde Q_i^{\mathrm{NTC}}
\bigl(
y_i,s,a_{\mathcal N_i^\kappa};\xi_i
\bigr)
=
\left\langle
f_{\xi_i}\!\left(
\psi_i(k(s),y_i,a_{\mathcal N_i^\kappa})
\right),
T_i(s)
\right\rangle .
\label{eq:app-ntc}
\end{equation}
At augmented state $(z,s)$, we set $y_i=h_i(z)$.

Thus, the critic remains an order-$n$ polynomial in the local state around
the current centroid, while its coefficients may depend nonlinearly on the
local actions, operating regime, and observable context. LTC is recovered as
the special case in which $f_{\xi_i}$ is affine.

\paragraph{Architecture.}
We use one independent coefficient network for each agent. The $N$ networks
have disjoint parameters but are stored in a single batched implementation,
so that all agents can be evaluated in parallel. Each coefficient network is
a two-hidden-layer MLP with hidden width $64$, $\tanh$ activations, and
agent-wise Kaiming initialization~\citep{he2015delving}.

Importantly, the network does not receive the raw current state as input.
The state dependence enters explicitly through the Taylor vector $T_i(s)$;
the neural network only receives the coefficient context $\psi_i$. Hence,
increasing the Taylor order changes the output dimension of the network,
rather than its input dimension.

\paragraph{TD training.}
NTC is trained by on-policy semi-gradient TD using SARSA targets and a
periodically synchronized target network. Let $\bar\xi_i$ denote the target
parameters for agent $i$. For a transition
\[
(Z_t,S_t,A_t,r_{i,t},Z_{t+1},S_{t+1},A_{t+1}),
\]
the target is
\begin{equation}
\tau_{i,t}
=
r_{i,t}
+
\gamma
\widetilde Q_i^{\mathrm{NTC}}
\left(
h_i(Z_{t+1}),
S_{t+1},
A_{t+1,\mathcal N_i^\kappa};
\bar\xi_i
\right),
\label{eq:app-ntc-target}
\end{equation}
where the target-network evaluation is detached from the gradient.
Given $M_c$ critic rollouts of length $T$, the online parameters minimize
\begin{equation}
\mathcal L_i^{\mathrm{NTC}}(\xi_i)
=
\frac{1}{M_c(T-1)}
\sum_{j=1}^{M_c}
\sum_{t=0}^{T-2}
\left[
\widetilde Q_i^{\mathrm{NTC}}
\left(
h_i(Z_t^{(j)}),
S_t^{(j)},
A_{t,\mathcal N_i^\kappa}^{(j)};
\xi_i
\right)
-
\tau_{i,t}^{(j)}
\right]^2.
\label{eq:app-ntc-loss}
\end{equation}

We optimize this objective with Adam using a learning rate of
$2\times10^{-3}$. The target network is synchronized every $25$ TD updates,
\[
\bar\xi_i\leftarrow\xi_i,
\]
and remains fixed between synchronization steps. We use $50$, $200$, and $200$ TD optimization steps for Taylor orders $n=1,2,3$, respectively as higher parameter counts require more steps to fit. On Kuramoto, NTC uses learning rate $10^{-3}$, $50$ TD updates per iteration on batches of $1024$ transitions, and target synchronization every $50$ updates. Agents are processed in blocks during training;
each block operates only on the corresponding disjoint agent parameters.

\paragraph{State-value estimate for one-step TD advantages.}
 NTC parameterizes an action-value
function whose dependence on the action is generally nonlinear through
$f_{\xi_i}$. Its state-value estimate is therefore obtained by marginalizing
the fitted critic under the current policy,
\[
\widetilde V_i(h_i(z),s)
=
\mathbb E_{a\sim\pi(\cdot\mid z,s)}
\left[
\widetilde Q_i^{\mathrm{NTC}}
\bigl(h_i(z),s,a_{\mathcal N_i^\kappa}\bigr)
\right].
\]
In the implementation, this expectation is approximated using
$n_{\mathrm{mc}}=16$ independent action samples from the current policy.
The resulting $\widetilde V_i$ is then used in the localized one-step TD advantage
estimator described in Section~\ref{sec:learning}. With independent action
samples, this Monte Carlo marginalization is unbiased and can be absorbed into
the actor-rollout randomness in the analysis. For LTC, this expectation is
simpler because the critic is affine in the actions; therefore, only the
expected action is needed to compute the value function.

\subsection{Local MLP critic}
\label{app:mlp-baseline}

As a generic local function-approximation baseline, we use one independent
multilayer perceptron for each agent. The critic directly approximates the
local state-value function
\[
\widetilde V_i(s)
=
f_{\xi_i}\!\left(s_{\mathcal N_i^\kappa}\right),
\]
where the input is the same $\kappa$-hop state neighborhood used by the
Taylor Centroid critics. On the scalar ring benchmarks, this corresponds to
input dimensions $3$, $5$, and $7$ for $\kappa=1,2,3$, respectively.

Each $f_{\xi_i}$ is a two-hidden-layer MLP with $\tanh$ activations and hidden width $128$ in the policy experiments on LQR and switched LQR, and $64$ on Kuramoto and in the critic-estimation experiment. Parameters are not shared across agents: the
implementation stores $N$ separate sets of weights and evaluates them in
parallel through a batched MLP. 

The critic is trained on-policy using TD(0) with a target network. For a
transition $(s_t,r_t,s_{t+1})$, the target for agent $i$ is
\[
y_{i,t}
=
r_{i,t}
+
\gamma
\widetilde V^{\,\mathrm{targ}}_i(s_{t+1}),
\]
and the training objective is
\[
\mathcal L_{\mathrm{MLP}}
=
\frac{1}{N|\mathcal B|}
\sum_{i=1}^N
\sum_{t\in\mathcal B}
\left(
\widetilde V_i(s_t)-y_{i,t}
\right)^2.
\]
Training hyperparameters are given in Appendix~\ref{app:experimental-protocol}. Local state inputs are standardized using
rollout statistics, and value targets are rescaled using a fixed scale
initialized from the first rollout.

The critic directly outputs $\widetilde V_i(s)$ and is combined with the same localized
actor and one-step TD advantage estimator used for the Taylor Centroid methods.

\subsection{Graph neural-network critic}
\label{app:gnn-baseline}

As a generic graph-based baseline, we use a message-passing neural network to
approximate the local state-value functions. Each node is initialized from its
own state,
\[
h_i^{(0)} = s_i,
\]
and neighborhood information is propagated through mean aggregation. A
message-passing layer takes the form
\[
h_i^{(\ell+1)}
=
\tanh\!\left(
W_{\mathrm{self}}^{(\ell)} h_i^{(\ell)}
+
W_{\mathrm{nbr}}^{(\ell)}
\frac{1}{|\mathcal N_i|}
\sum_{j\in\mathcal N_i} h_j^{(\ell)}
+
b^{(\ell)}
\right).
\]
A shared linear readout produces the local value estimate
\[
\widetilde V_i(s)
=
w_{\mathrm{out}}^\top h_i+b_{\mathrm{out}}.
\]

The network has $L=\max\{2,\kappa\}$ message-passing layers, each with its own weights $(W_{\mathrm{self}}^{(\ell)},W_{\mathrm{nbr}}^{(\ell)},b^{(\ell)})$: weights are not reused across layers, and each layer performs at most one round of one-hop aggregation. The neighbor term is included in the first $\kappa$ layers only, so that $\widetilde V_i$ depends only on $s_{\mathcal N_i^\kappa}$: every layer aggregates for $\kappa\in\{2,3\}$, whereas for $\kappa=1$ the second layer acts on each node separately. All weights are shared across agents.

In the policy-optimization experiments on LQR and switched LQR, the hidden width is $128$, which gives $17{,}025$, $33{,}409$ and $66{,}305$ parameters for $\kappa=1,2,3$; since all weights are shared across agents, this number does not depend on $N$. In the critic-estimation experiment, the width is chosen so that the parameter count matches that of the MLP used there. For policy optimization on Kuramoto, the GNN replaces the random-feature map inside the RF-SAC training loop; there it has two layers of width $64$ at every $\kappa$, hence a receptive field of $\min\{\kappa,2\}$, and SAC's twin critics are two such networks.

On LQR and switched LQR, the GNN critic is trained on-policy by TD(0) with a target network. For each
agent $i$,
\[
y_{i,t}
=
r_{i,t}
+
\gamma\widetilde V_{i,\mathrm{targ}}(s_{t+1}),
\]
and the training loss is
\[
\mathcal L_{\mathrm{GNN}}
=
\frac{1}{N|\mathcal B|}
\sum_{i=1}^N
\sum_{t\in\mathcal B}
\left(
\widetilde V_i(s_t)-y_{i,t}
\right)^2.
\]
Training hyperparameters are given in Appendix~\ref{app:experimental-protocol}. State inputs are standardized using statistics from
the rollout distribution, and value targets use the same fixed rescaling as the
local MLP baseline.

For policy optimization, this critic is paired with the actor of \citet{ren2024spectral} on LQR and Kuramoto, taken from their released code. On switched LQR, it is paired with our actor.

\end{document}